\documentclass[twoside,11pt]{article}

\usepackage{blindtext}

\usepackage{amsmath}
\usepackage[amsmath]{ntheorem}

\PassOptionsToPackage{authoryear}{natbib}
\usepackage[abbrvbib]{jmlr2e}

\usepackage{lastpage}
\jmlrheading{xx}{2024}{1-\pageref{LastPage}}{8/24; Revised x/xx}{x/xx}{xx-0000}{Wan, Yu, Sutton}

\usepackage{subfiles}
\usepackage[titletoc]{appendix}

\usepackage{nameref,zref-xr}
\zxrsetup{toltxlabel}
\usepackage{amsmath,amssymb,enumitem}

\numberwithin{equation}{section}

\usepackage[utf8]{inputenc} 
\usepackage[T1]{fontenc}    
\usepackage{booktabs}       
\usepackage{amsfonts}       
\usepackage{nicefrac}       
\usepackage{microtype}      
\usepackage{xcolor}         

\usepackage{graphicx}
\usepackage{bm,bbm}
\usepackage{xcolor}         
\usepackage[algoruled,lined,linesnumbered]{algorithm2e} 
\usepackage[capitalize,nameinlink,noabbrev]{cleveref}
\usepackage[skip=0pt,belowskip=-2pt]{caption}
\usepackage{subcaption}
\usepackage{chngcntr}
\usepackage{dblfloatfix}    
\usepackage{wrapfig}

\allowdisplaybreaks

\newtheorem{mytheorem}{Theorem}[section] 
\newtheorem{myassumption}{Assumption}[section] 
\newtheorem{mylemma}{Lemma}[section]
\newtheorem{myprop}{Proposition}[section]
\newtheorem{mycorollary}{Corollary}[section]
\newtheorem{myremark}{Remark}[section]

\usepackage{enumitem}
\setlist[itemize]{itemsep=0pt,parsep=2pt,topsep=2pt}
\setlist[enumerate]{itemsep=0pt,parsep=2pt,topsep=2pt}
\def\E{\mathbb{E}}
\def\R{\mathbb{R}}
\def\I{\mathcal{I}}
\def\F{\mathcal{F}}

\def\={\mathop{\overset{\text{\rm \tiny def}}{=}}}

\def\qed{\hfill\BlackBox}
\newenvironment{proofof}[1]{\par\noindent{\bf Proof of #1\ }}{\hfill\BlackBox\\[2mm]}

\defcitealias{ScF78}{S\&F, 1978}
\defcitealias{yushkevich1982semi}{Yushkevich, 1982}

\def\old#1{}

\def\Rs{\mathcal{R}}
\def\Rst{\mathcal{R}^*} 
\def\Kst{K^*}

\def\Qo{\mathcal{Q}}

\def\Qs{\mathcal{Q}_s}
\def\S{\mathcal{S}}
\def\A{\mathcal{A}}
\def\0{\mathbf{0}}
\def\1{\mathbf{1}}

\DeclareMathAlphabet{\mathbbb}{U}{bbold}{m}{n}
\newcommand{\ind}{\mathbbb{1}}

\newcommand{\mdp}{\calM}
\newcommand{\trans}{p}

\newcommand{\sspace}{\mathcal{S}}
\newcommand{\aspace}{\mathcal{A}}
\newcommand{\rspace}{\mathcal{R}}

\newcommand{\optimalr}{r_*}

\newcommand{\ExtRVIQsolutionq}{\mathcal{Q}_\infty}

\newcommand{\otrans}{\hat p}
\newcommand{\ospace}{\mathcal{O}}
\newcommand{\lspace}{\mathcal{L}}
\newcommand{\orspace}{\mathcal{\hat R}}
\newcommand{\optionr}{\hat r}

\newcommand{\ooptimalr}{\optionr_*}

\newcommand{\ooptimalitydiffsolutionq}{\mathcal{\hat Q}_{\infty}}
\newcommand{\smdpoptimalitygeneralsolutionq}{\mathcal{Q}_{s}}
\newcommand{\smdpoptimalitygeneralsolutionv}{\mathcal{V}_{s}}

\newcommand{\smdpoptimalitysolutionq}{\mathcal{Q}}
\newcommand{\smdpoptimalitysolutionv}{\mathcal{V}}

\newcommand{\ispace}{\mathcal{I}}
\newcommand{\GRVIQsolutionq}{\mathcal{Q}_\#}
\newcommand{\GRVIQsolutionrbar}{r_\#}

\usepackage{wrapfig}

\DeclareMathOperator*{\argmin}{\arg\!\min}
\DeclareMathOperator*{\argmax}{\arg\!\max}

\def\zerovec{{\mathbf{0}}}
\def\onevec{{\mathbf{1}}}

\newcommand{\bbE}{\mathbb{E}}
\newcommand{\bbR}{\mathbb{R}}

\newcommand{\calS}{\mathcal{S}}
\newcommand{\calA}{\mathcal{A}}

\newcommand{\calF}{\mathcal{F}}

\newcommand{\calM}{\mathcal{M}}
\newcommand{\calX}{\mathcal{X}}

\newcommand{\calQ}{\mathcal{Q}}
\newcommand{\calV}{\mathcal{V}}
\newcommand{\V}{\mathcal{V}}

\newcommand{\norm}[1]{\left\lVert#1\right\rVert}
\newcommand{\abs}[1]{\left\lvert#1\right\rvert}
\newcommand{\cardS}{{\abs{\calS}}}
\newcommand{\cardA}{{\abs{\calA}}}

\newcommand\cites[1]{\citeauthor{#1}'s\ (\citeyear{#1})}
\usepackage{xargs}

\ShortHeadings{Convergence of Average-Reward Q-Learning}{Wan, Yu, Sutton}
\firstpageno{1}

\begin{document}
\title{On Convergence of Average-Reward Q-Learning \\in Weakly Communicating Markov Decision Processes}

\author{\name Yi Wan\thanks{Yi Wan and Huizhen Yu contributed equally to this work.} \email wan6@ualberta.ca\\
\addr University of Alberta, Canada and Meta AI, USA\\
        \name Huizhen Yu\footnotemark[1] \email huizhen@ualberta.ca \\
        \addr University of Alberta, Canada\\
        \name Richard S. Sutton \email rsutton@ualberta.ca\\
       \addr University of Alberta and Alberta Machine Intelligence Institute (Amii), Canada
}
\editor{My editor}

\maketitle
\begin{abstract}
This paper analyzes reinforcement learning (RL) algorithms for Markov decision processes (MDPs) under the average-reward criterion. We focus on Q-learning algorithms based on relative value iteration (RVI), which are model-free stochastic analogues of the classical RVI method for average-reward MDPs. These algorithms have low per-iteration complexity, making them well-suited for large state space problems. We extend the almost-sure convergence analysis of RVI Q-learning algorithms developed by Abounadi, Bertsekas, and Borkar \citeyearpar{abounadi2001learning} from unichain to weakly communicating MDPs. This extension is important both practically and theoretically: weakly communicating MDPs cover a much broader range of applications compared to unichain MDPs, and their optimality equations have a richer solution structure (with multiple degrees of freedom), introducing additional complexity in proving algorithmic convergence. We also characterize the sets to which RVI Q-learning algorithms converge, showing that they are compact, connected, potentially nonconvex, and comprised of solutions to the average-reward optimality equation, with exactly one less degree of freedom than the general solution set of this equation. Furthermore, we extend our analysis to two RVI-based hierarchical average-reward RL algorithms using the options framework, proving their almost-sure convergence and characterizing their sets of convergence under the assumption that the underlying semi-Markov decision process is weakly communicating.
\end{abstract}


\section{Introduction}
\label{sec: Introduction}
This paper concerns continuing reinforcement learning (RL) with the average-reward criterion. In this setting, an agent interacts continually in discrete time with an environment modeled as a finite-space Markov decision process (MDP), taking actions and receiving states and reward signals. The goal is for the agent to select actions to maximize the long-term average of the expected rewards over time, known as the \emph{reward rate}. The average-reward criterion is well-suited for systems that need to sustain performance and reliability over extended periods without operational resets. For example, average-reward RL has been applied in airline revenue management \citep{gatti2022outsmarting}, mobile health intervention \citep{liao2022batch}, recommender systems \citep{warlop2018fighting}, and network service delegation \citep{bakhshi2023multi}.

Theoretical research on average-reward RL has explored a variety of approaches, each with distinct research objectives and challenges. Before elaborating on the approach that is the focus of this paper, let us briefly mention several alternative methods. For instance, some methods tackle average-reward problems indirectly, approximating them through discounted-reward problems with sufficiently large discount factors \citep[e.g.,][]{wei2020model, hong2024provably,dong2022simple} or through undiscounted finite-horizon problems with sufficiently long horizons \citep[e.g.,][]{wei2021learning}. While these approximations can introduce numerical stability issues, they can, in theory, bypass the difficulties that direct approaches face with complex state communication structures—structures relating to accessibility among different parts of the state space, a key factor in the difficulty of the average-reward problem \citep[cf.][Chapters 8 and 9]{puterman2014markov}. Among the direct approaches, some model-based methods \citep[e.g.,][]{bartlett2012regal,ouyang2017learning} aim not only to solve average-reward problems but to do so in a sample-efficient manner, albeit at the cost of higher computational demands and memory usage compared to model-free methods. In the model-free category, actor-critic methods \citep[e.g.,][]{KT2003,abbasi2019politex} remain the most practical and widely applied, particularly in robotics, although, as policy-gradient methods, they have more restrictive MDP model conditions, such as ergodic MDPs, to ensure differentiability and other regularities required by the methods. For a more comprehensive review of average-reward algorithms, readers may refer to \citet{wan2023learning}.

In this paper, we study a family of model-free average-reward RL algorithms based on the relative value iteration (RVI) approach---also known as the successive approximation method---to solving average-reward MDPs. The core idea of this approach is exemplified by the classical RVI algorithms of \citet{white1963dynamic} and \citet{schweitzer1971iterative} (cf.\ \cref{sec: classical rvi}). Grounded in the understanding of the asymptotic behavior of undiscounted value iteration in MDPs \citep{ScF77}, these RVI algorithms can be viewed as reformulations of undiscounted value iteration, designed to successively approximate the optimal reward rate and state values (representing, in some sense, the relative ``advantages'' of starting from particular states), with the ultimate goal of solving the average-reward optimality equation and deriving an optimal policy. RVI-based model-free RL algorithms share this objective and operate analogously, but differ in their stochastic and asynchronous nature. These algorithms iteratively and incrementally estimate the optimal reward rate and state-action values (or $Q$-values) using random state transition and reward data from stochastic environments, without requiring model knowledge or simultaneous updates across all state-action pairs. Due to their stochasticity and asynchrony, it was initially unclear how convergence could be ensured. The first algorithms in this family with convergence guarantees were introduced by \citet{abounadi2001learning}, who coined the term ``RVI Q-learning.'' In this paper, we broadly use this name to refer to RVI-based Q-learning algorithms, including the Differential Q-learning algorithm and further generalized formulations introduced recently by \citet{wan2021learning}.

To place RVI Q-learning in the broader context of average-reward RL, this approach is distinct from the aforementioned methods, offering its own advantages and challenges. Each iteration of RVI Q-learning has a low computational cost and minimal memory requirements: it is similar to Q-learning for discounted problems, with the only key difference being the subtraction of a scalar estimate of the optimal reward rate from the reward at each iteration. Compared with model-based tabular methods, this makes RVI Q-learning more appealing for large state space problems with computational resource constraints. Unlike indirect methods, using the RVI approach avoids potential numerical instabilities associated with large discount factors or long horizons used to approximate average-reward problems. Additionally, unlike actor-critic methods, RVI Q-learning can be applied beyond ergodic MDPs and allows for more flexible data generation, such as data gathered from off-policy RL scenarios or based on human experts' policies. However, although outside the scope of this paper, it is worth noting that incorporating function approximation into RVI Q-learning is more challenging than in actor-critic methods, as it may compromise convergence guarantees. Improving sample efficiency and online learning performance through careful data generation remains another open challenge for RVI Q-learning.

Turning now to the main focus of this paper, we address one critical aspect of the theoretical foundation of RVI Q-learning: establishing convergence guarantees under much broader MDP model conditions than previously known. Specifically, we extend the almost sure convergence analysis of RVI Q-learning developed by \citet{abounadi2001learning} from unichain to weakly communicating MDPs. This extension is important both practically and theoretically. 

As will be elaborated in \cref{sec: wc mdps}, weakly communicating MDPs comprise all MDPs where, aside from transient states eventually not encountered under any policy, every state can be reached from every other state under \emph{some} policy. This structure not only ensures, as in unichain MDPs, that sufficient information can be gathered to discover an optimal policy in RL applications where the agent learns through a continuous stream of agent-environment interactions. But, more importantly, it also allows for scenarios common in practice where some stationary policies (possibly optimal ones) can induce Markov chains with multiple recurrent classes---distinct groups of states where the process gets ``trapped'' under the policy---an outcome not permitted under the unichain model. Thus, weakly communicating MDPs cover a much broader range of applications compared to unichain MDPs. 

Theoretically, a key distinction between weakly communicating MDPs and unichain MDPs lies in the solution structure of their average-reward optimality equations. While solutions in unichain MDPs are always unique up to an additive constant, solutions in weakly communicating MDPs can possess multiple degrees of freedom (\citealp{ScF78}; see also \cref{sec: wc mdps}), introducing additional complexity in the convergence analysis of RVI Q-learning. 

As the first main contribution of this paper, we establish, for weakly communicating MDPs, the almost sure convergence of RVI Q-learning to a subset of solutions of the average-reward optimality equation (\cref{thm: Extended RVI Q-learning}), with this subset being compact, connected, and potentially nonconvex (\cref{thm: MDP characterize Q}) and possessing exactly one fewer degree of freedom than solutions of the average-reward optimality equation (\cref{thm-dim-Qs}). These results entail the earlier findings of RVI Q-learning converging to a single point in unichain MDPs \citep{abounadi2001learning,wan2021learning} as a special case, where the corresponding solution subset has zero degrees of freedom and reduces to a singleton.

Our second set of results extends the scope of the convergence analysis from RVI Q-learning to RVI-based Q-learning algorithms for \emph{hierarchical} average-reward RL. Specifically, we study two such algorithms introduced by \citet{wan2021average}. In hierarchical RL problems, instead of directly choosing from actions, the agent selects from a set of temporally abstracted actions, or \emph{options} \citep{sutton1999between}, with the objective of maximizing the average-reward rate. This hierarchical RL problem formulation is suitable for applications involving vast action spaces and long sequences of actions for task completion. For instance, for a device-assembly robot, where each action involves applying specific forces to its joints, thousands of actions might be needed just to position a single component accurately. Without a hierarchical formulation, managing such a vast action space can be impractical and inefficient. However, by employing a hierarchical approach with options like grasping, moving, and placing objects, the problem becomes more manageable and efficient to solve. While the algorithms studied in this paper assume predefined options, it is worth noting the important and active research area of automatic construction of these options. Readers interested in option discovery can refer to works such as \citep{bacon2017option,wan2022toward,sutton2023reward} and the references therein.

The underlying decision processes of hierarchical RL problems are semi-Markov decision processes (SMDPs), which generalize MDPs by allowing state transitions to occur over varying time durations. To address hierarchical RL problems, two main classes of algorithms are typically used: \emph{inter-option} algorithms, which directly operate on the underlying SMDPs by treating each option as an action in the SMDP, and \emph{intra-option} algorithms, which exploit the structures within options for greater efficiency. The two options algorithms proposed in \citet{wan2021average} and studied in this paper belong to these respective categories.

We prove the almost sure convergence of these two options algorithms, assuming that the SMDP arising from the hierarchical formulation is weakly communicating (Theorems\ \ref{thm: inter-option Differential Q-learning}, \ref{thm: intra-option Differential Q-learning}). Previous convergence analyses \citep{wan2021average} of these two algorithms require the SMDP to be unichain; additionally, these analyses  
have gaps in stability analysis (see \cref{remark: sa result compare} for a detailed discussion). Similar to RVI Q-learning, we also characterize the sets to which these options algorithms converge (\cref{prop: c2 inter = intra equations} and Theorems\ \ref{thm: SMDP characterize Q}, \ref{thm-dim-Qs}).

Our convergence analyses of RVI Q-learning and its options extensions employ a unified framework, treating these algorithms as specific instances of an abstract stochastic RVI algorithm (\cref{sec: c1 general RVI Q}), which we analyze using the ordinary differential equation (ODE)-based proof approach from stochastic approximation (SA) theory. This analysis builds on a stability proof method for SA algorithms introduced by \citet{BoM00} and the line of argument introduced by \citet{abounadi2001learning} to analyze the solution properties of the ODEs associated with RVI Q-learning in unichain MDPs. To address the more general weakly communicating MDPs or SMDPs and the more general options algorithms, we make two important extensions to these previous analyses. First, for the inter-option algorithm for solving the underlying SMDP, the noise conditions in \citet{BoM00} are too restrictive, so we extend their result to accommodate more general noise conditions. This extension is non-trivial and requires modification of critical parts of their proof. We state our result in this paper and refer interested readers to another paper for detailed proofs \citep{yu2023note}. Secondly, unlike the case studied in \citet{abounadi2001learning}, where the ODE associated with RVI Q-learning possesses a unique equilibrium, for weakly communicating MDPs/SMDPs, the ODEs associated with our algorithms generally possess multiple equilibrium points. We extend the line of analysis of \citet{abounadi2001learning} to tackle this situation by leveraging the solution structure in the average-reward optimality equations of weakly communicating MDPs and SMDPs. 

The paper is organized as follows. \cref{sec: background} provides background information on average-reward MDPs, weakly communicating MDPs, and the classical RVI algorithm. \cref{sec: action algorithms} introduces RVI Q-learning and presents our results on its associated solution set and convergence properties. \cref{sec: options algorithms} covers hierarchical average-reward RL: we first present the preliminaries on average-reward SMDPs (\cref{sec: smdp}), the background on options and their resulting SMDPs (\cref{sec: options background}), followed by our convergence results for the two average-reward options algorithms and the properties of their corresponding solution set (Sections \ref{sec: inter-option algorithms}, \ref{sec: intra-option algorithms}). The subsequent three sections provide proofs for the properties of the solution sets associated with the algorithms (\cref{sec: solution set}), the convergence theorems (\cref{sec: convergence proofs}), and the characterization of the degrees of freedom of those solution sets (\cref{sec: solution set dim}). We conclude the paper by discussing future directions in \cref{sec: conclusions}.

\section{Background}\label{sec: background}

In this section, we start by introducing average-reward MDPs and weakly communicating MDPs. We then discuss the solution structures of average-reward optimality equations in weakly communicating MDPs, and the classical RVI approach to solving these equations. The book by \citet{puterman2014markov} and the book chapter by \citet{Kal02} on finite-space MDPs serve as primary references for the majority of the background materials discussed here. Additional references will be provided for specific results. 

\subsection{MDPs with the Average-Reward Criterion} \label{MDPs with the Average-Reward Criterion}

We consider a finite state and action MDP defined by a tuple $\mdp = (\sspace, \aspace, \rspace, \trans)$. Here $\sspace$ ($\aspace$) denotes a finite set of states (actions), and $\rspace \subset \R$ is a finite
\footnote{We consider a finite reward space $\rspace$ for notational convenience only. All results presented in this paper apply to the general case where $\rspace = \R$ and the one-stage random rewards have finite variances.}
set including all possible one-stage rewards. We use $\Delta(\calX)$ to denote the probability simplex over a finite space $\calX$. The transition function $p: \sspace \times \aspace \to \Delta(\sspace \times \rspace)$ specifies state transitions and reward generation in the MDP. Specifically, when action $a \in \aspace$ is taken at state $s \in \sspace$, the system transitions to state $s' \in \sspace$ and yields a reward $r \in \rspace$ with probability $p(s', r \mid s, a)$.

A \emph{history-dependent policy} is a collection of possibly randomized decision rules, one for each time step $n$. These rules specify which action to take at a given time step, conditioned on the history of states, actions, and rewards, $s_0, a_0, r_1, s_1, \ldots, a_{n-1}, r_{n}, s_n$, realized up to that point. If all these rules are nonrandomized, the policy is called \emph{deterministic}. When they do not vary with the time step $n$ and depend only on the current state $s_n$, the policy is called \emph{stationary} and can be represented by a function that maps each state $s \in \sspace$ to a probability distribution in $\Delta(\aspace)$. Specifically, a deterministic stationary policy can be represented by a function that maps $\sspace$ into $\aspace$.

For a given initial state $S_0 = s$, applying a policy $\pi$ in the MDP induces a random process $\{S_n, A_n, R_{n+1}\}_{n \geq 0}$ of states, actions, and rewards. Let $\E_\pi [ \cdots \mid S_0=s]$ denote the corresponding expectation operator. 
The \emph{average-reward criterion} measures the \emph{reward rate of $\pi$} for each initial state $s$ according to
\begin{equation} \label{eq: r-pi mdp}
r(\pi, s) \= \, \liminf_{n \to \infty} \frac{1}{n} \sum_{k = 1}^{n} \E_\pi [R_k \mid S_0=s], \qquad \forall \, s \in \sspace.
\end{equation}
If $\pi$ is stationary, the ``$\liminf$'' in the above definition can be replaced by ``$\lim$'' based on finite-state Markov chain theory.
A policy is called \emph{optimal} if for \emph{all} initial states $s \in \sspace$, it achieves the \emph{optimal reward rate} $\optimalr(s) \= \sup_{\pi} r(\pi, s)$, where the supremum is taken over all history-dependent policies $\pi$. 

It is well-established that there exists a deterministic optimal policy in the class $\Pi$ of stationary policies. Moreover, the stationary optimal policies $\pi_*$, the set of which we denote by $\Pi_*$, enjoy a stronger sense of optimality. This is expressed by the following inequality: for every history-dependent policy $\pi$:
\begin{equation} \label{eq: mdp strong opt}
    \lim_{n \to \infty} \frac{1}{n} \sum_{k = 1}^n \E_{\pi_*}[R_k \mid S_0=s] \geq \limsup_{n \to \infty} \frac{1}{n} \sum_{k = 1}^n \E_\pi [R_k \mid S_0=s], \qquad \forall\, s \in \sspace.
\end{equation}

Before \cref{sec: options algorithms}, our focus will primarily be on stationary policies and stationary optimal policies. For brevity, we will refer to them simply as policies and optimal policies.

\subsection{Weakly Communicating MDPs: Optimality Equations \& Solution Structures} \label{sec: wc mdps}

In general, the optimal reward rate $\optimalr(s)$ may vary with the initial state $s$. In this paper, we shall focus on a class of MDPs known as weakly communicating MDPs, wherein $\optimalr(\cdot)$ remains constant. These MDPs are characterized by the communicating structure among their states, as described below.

A set $D$ of states in an MDP forms a \textit{communicating class} if for every pair of states $s, s' \in D$, there exists a policy that can reach state $s'$ from state $s$ with positive probability. If from any state within $D$, the system cannot leave $D$ regardless of the policy employed, then  $D$ is considered \emph{closed}. A state is labeled \emph{transient} under a policy if starting from this state, almost surely it will only be revisited a finite number of times. 

\begin{definition} \rm \label{def: w.c. MDP}
An MDP is classified as \emph{weakly communicating} if it possesses a unique closed communicating class of states, with all other states being transient under all policies. When the entire state space $\sspace$ is a communicating class, the MDP is called \emph{communicating}.
\end{definition}

The concepts of communicating and weakly communicating MDPs were introduced by \citet{Bat73} and \citet{platzman1977improved}, respectively. Determining whether an MDP is communicating is straightforward. Simply consider a randomized stationary policy that assigns positive probability to every action at every state. The MDP is communicating if and only if, under this policy, the resulting Markov chain $\{S_n\}$ has a single recurrent class
\footnote{A \emph{recurrent class} corresponds to a closed communicating class, as defined above, when treating the finite-state Markov chain as an uncontrolled MDP with a single dummy policy. States in these classes are called \emph{recurrent} for the Markov chain; they will almost surely be revisited infinitely often when starting from any state within their associated class.}
and no transient states. For the MDP to be weakly communicating, in addition to having a single recurrent class, the transient states of this Markov chain $\{S_n\}$ need to remain transient in the MDP under all policies. (An efficient algorithm for classifying an MDP based on its state transition dynamics is available; see \citet[Chap.\ 8.3.2]{puterman2014markov} for details.) 

In MDP and RL applications, \emph{unichain} MDPs are frequently employed to model problems. These are a subclass of weakly communicating MDPs where, under any policy, the induced Markov chain $\{S_n\}$ has a single recurrent class, together with a (possibly empty) set of transient states. As we shall discuss shortly, this subclass is much more restrictive and less general than the broader class of weakly communicating MDPs, leading to a more limited scope of applicability.

When all states have the same optimal reward rate $\optimalr$ (which is henceforth treated as a scalar), an optimal policy can be determined from a solution of the \emph{average-reward optimality equation}, given below in two equivalent forms:
\begin{align}
v(s) & = \max_{a \in \aspace} \left\{ r_{sa} - \bar r +  \sum_{s' \in \sspace} p_{ss'}^a  v(s') \right\},\qquad \forall \, s \in \calS,  \label{eq: state-value optimality equation}  \\
q(s, a) & =  r_{sa} - \bar r +  \sum_{s' \in \sspace} p_{ss'}^a \max_{a' \in \aspace} q(s', a'),\qquad \ \, \forall \, s \in \calS, \, a \in \aspace,  \label{eq: action-value optimality equation}
\end{align}    
where $r_{sa}$ and $p_{ss'}^a$ are the expected one-stage reward and the state transition probability, respectively, given by
$r_{sa} \= \sum_{s' \in \sspace} \sum_{r \in \rspace} r \cdot \trans(s', r \!\mid\! s, a)$ and $p_{ss'}^a \= \sum_{r \in \rspace} \trans(s', r \!\mid\! s, a)$.
In the first (resp.\ second) equation, referred to as the \emph{state-value} (resp.\ \emph{action-value}) \emph{optimality equation}, we solve for $(\bar r, v) \in \R \times \R^{|\sspace|}$ (resp.\ $(\bar r,  q) \in \R \times \R^{|\sspace| \times |\aspace|}$). 
We will use both of these equations in this paper, as the RL algorithms we study aim to solve the second one, while for analysis, it is sometimes convenient to use the first one. 

It is well-established that these optimality equations have solutions. Moreover, for any solution, its $\bar r$-component always coincides with $\optimalr$, and if a policy solves the corresponding maximization problems in the right-hand side (r.h.s.)\ of either equation, the policy is optimal.

Let $\calV$ denote the set of solutions for $v$ in \eqref{eq: state-value optimality equation}, and let $\calQ$ denote the set of solutions for $q$ in \eqref{eq: action-value optimality equation}. It is notable that adding a constant to any solution of $v$ or $q$ yields another solution. Prior studies \citep{abounadi2001learning,wan2021learning} on average-reward Q-learning focused on cases where these solutions are unique up to an additive constant. Specifically, \citet{abounadi2001learning} considered unichain MDPs,
\footnote{In \citep{abounadi2001learning}, these MDPs are also required to possess a common state that is recurrent under all policies, which is unnecessarily restrictive, as discussed above.}
while \citet{wan2021learning} considered weakly communicating MDPs with this uniqueness solution property. The rationale presented in these studies can also be applied to non-weakly-communicating MDPs with a constant optimal reward rate, provided their optimality equations exhibit this uniqueness solution property. 

In weakly communicating (or communicating) MDPs, the solution structure of optimality equations is typically more complex. A fundamental work by \citet{ScF78}
\footnote{Later, we will often use the alias \citepalias{ScF78} to refer to this work for brevity.}
reveals that solutions in $\calV$ and $\calQ$ can exhibit multiple degrees of freedom, quantified by a number $n^*$. This number, along with a parametrization of the solution sets using $n^*$ parameters, can be precisely determined based on the recurrence structures of the Markov chains $\{S_n\}$ induced by optimal policies. (In fact, \citet{ScF78} characterized the solution structure for the entire family of finite-space MDPs, where the optimal reward rate may vary with the initial state. We provide an overview of their key findings in \cref{sec: degree of freedom Q}.) 
Furthermore, they categorized these $n^*$ parameters into two types, globally independent vs.\ locally independent, based on the transience/recurrence structure induced by optimal policies in the MDP. Roughly speaking, the globally independent parameters can take arbitrary values in their space. These determine the ranges within which the values of the locally independent parameters can be selected. For weakly communicating MDPs, if $n^* >1$, then all $n^*$ parameters are locally independent in the sense introduced by \citepalias{ScF78} (although this fact will not be directly utilized in our results).

We defer a detailed explanation of some of their results to \cref{sec: degree of freedom Q} for interested readers. Here, let us first demonstrate with examples when $n^*$ can equal or exceed $1$ in weakly communicating MDPs, before discussing the latter case. (Note that $n^* = 1$ indicates optimality equations have unique solutions up to an additive constant.)

\begin{example} \label{example: different MDPs} \rm
Shown in \cref{fig:different mdps} are three communicating MDPs with two states and two actions. Let $s$ and $d$ stand for actions \texttt{solid} and \texttt{dashed}, respectively.

In \cref{fig:different mdps}(a), the MDP is unichain. It has $\optimalr = 1$ and 
$$\calQ = \{q \in \R^{3} \mid q(1, d) = c-1,\, q(2, s) = c,\, q(2, d) = c - 2,\, c \in \bbR \}.$$
Solutions in $\calQ$ differ only by an additive constant.

In \cref{fig:different mdps}(b), the MDP is not unichain, since the policy that takes action \texttt{solid} at both states induces two recurrent classes, $\{\emph{1}\}$ and $\{\emph{2}\}$. In this MDP, $\optimalr = 0$ and 
    $$\calQ = \{q \in \R^{2 \times 2} \mid q(1, s) = c-1,\, q(1, d) = c,\, q(2, s) = c,\, q(2, d) = c,\,  c \in \R \}.$$ 
Like the first unichain MDP, solutions in $\calQ$ are also unique up to an additive constant.

Finally, consider the MDP in \cref{fig:different mdps}(c). It has $\optimalr = 1$ and
$$\calQ = \{q \in \R^{2 \times 2} \mid q(2, s) - 1 \leq q(1, s) \leq q(2, s) + 1;\, q(1, d) = q(2, s) - 1,\, q(2, d) = q(1, s) - 1\}.$$  
Thus, solutions in $\calQ$ do not necessarily differ by a constant vector. 

This MDP also illustrates the degrees of freedom discussed above for the solutions in $\calQ$. Here, these solutions possess two degrees of freedom that are locally, rather than globally, independent: $(q(1, s),\, q(2, s))$ can be chosen from the $2$-dimensional convex polyhedron defined by the inequality constraints $q(2, s) - 1 \leq q(1, s) \leq q(2, s) + 1$, while the values $q(1, d)$ and $q(2, d)$ are determined by $(q(1, s),\, q(2, s))$.\qed
\end{example}

\begin{figure}
    \centering
    \includegraphics[width=\textwidth]{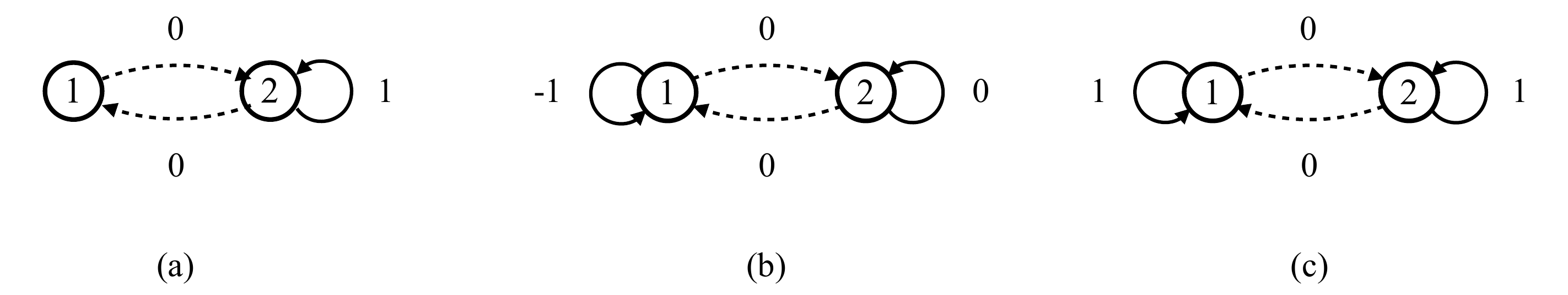}
\caption{Three examples of communicating MDPs with or without the uniqueness solution property. All these MDPs have two states $\{\emph{1, 2}\}$ and two actions $\{\texttt{solid}, \texttt{dashed}\}$ with deterministic effects. The directed solid and dashed curves between states depict deterministic state transitions corresponding to actions \texttt{solid} and \texttt{dashed}, respectively, with associated rewards indicated by numbers. Subfigure (a): a unichain MDP; (b): an MDP that is not unichain but has unique solutions in $\calQ$ (up to an additive constant); (c): an MDP without the uniqueness solution property.
}
\label{fig:different mdps}
\end{figure}

We have just provided an example where $n^* > 1$, indicating that the solutions in $\calV$ and $\calQ$ are not unique up to an additive constant.
\footnote{The sets $\calV$ and $\calQ$, being homeomorphic to each other, share the same number $n^*$ of degrees of freedom; see \cref{sec: degree of freedom Q} for details.}
More generally, based on the theory of \citet{ScF78}, we can deduce that \emph{for a weakly communicating MDP, $n^* > 1$ occurs precisely in the following situation: There exist at least two disjoint subsets of states, both forming recurrent classes under some optimal policy. However, ``traversing'' between these subsets incurs significant costs, rendering any (stationary) policy that visits both subsets infinitely often non-optimal.}

Notably, the scenario just described is quite common in real-world applications. The preceding discussion thus demonstrates that, both theoretically and practically, the class of weakly communicating MDPs is much broader and more versatile than its subfamilies with the uniqueness solution property.

\subsection{Relative Value Iteration} \label{sec: classical rvi}

\emph{Relative value iteration} (RVI), also known as \emph{successive approximations}, is a classical approach to solving average-reward optimality equations when the optimal reward rate remains constant across initial states. In this subsection, we will discuss Schweitzer's RVI algorithm \citep{schweitzer1971iterative}, which is a generalization over the first RVI algorithm proposed by \citet{white1963dynamic}. Schweitzer's algorithm was designed for solving SMDPs, a more general class of problems including MDPs (we will introduce SMDPs later in \cref{sec: smdp}). It targets the state-value optimality equation \eqref{eq: state-value optimality equation}, a focal point in the MDP/SMDP research field. 

Given our focus on stochastic RVI algorithms, which operate on state-action values, we will describe a specialized version of Schweitzer's RVI algorithm tailored to solving action-value optimality equations \eqref{eq: action-value optimality equation} in MDPs. This specialized algorithm will help elucidate the connections and differences between classical RVI and its stochastic counterpart, which will be our main focus for the rest of this paper.

This RVI algorithm operates in the space $\R^{|\sspace| \times |\aspace|}$ of state-action values. Let $\alpha \in (0,1)$ be a step-size parameter, and let $Q_0$ be the initial vector. The algorithm iteratively updates $Q_{n+1}$ for $n \geq 0$ according to the following rule: for all $(s,a) \in \sspace \times \aspace$,
\begin{equation} \label{eq: S-RVI}
Q_{n+1}(s,a) = Q_n(s,a) + \alpha \left( r_{sa} - f(Q_n) + \sum_{s' \in \sspace} p_{ss'}^a \max_{a' \in \aspace} Q_{n}(s', a')  - Q_n(s,a)\right),
\end{equation}
where $f(Q_n)$ is defined, for some fixed state-action pair $(\bar s, \bar a) \in \sspace \times \aspace$, as
$$ f(Q_n) =  r_{\bar{s}\bar{a}} + \sum_{s' \in \sspace} p_{\bar ss'}^{\bar a} \max_{a' \in \aspace} Q_{n}(s', a')  - Q_n(\bar s, \bar a).$$
It is worth noting that in this algorithm, all the iterates $Q_n$ maintain their $(\bar s, \bar a)$-component unchanged throughout the iterations by design; however, this feature is not critical. Alternative forms of functions $f$ can also be employed, leveraging fundamental results on the asymptotic behavior of undiscounted value iteration (see \citet{ScF77} for more details, though beyond the scope of this paper). 

This algorithm is proven to converge whenever the optimal reward rate is constant, particularly in a weakly communicating MDP \citep{platzman1977improved}, with $f(Q_n)$ converging to $\optimalr$ and $\{Q_n\}$ converging to a solution of the optimality equation (2.2). (See \citet[Theorem 1]{platzman1977improved} for further details, including error bounds, as well as performance bounds for the resulting policies.)

We will now delve into average-reward Q-learning in the following section, which can be viewed as the stochastic counterpart of the classical RVI algorithm.

\section{Convergence of RVI Q-Learning}\label{sec: action algorithms}

This section presents our new convergence result for a family of RVI Q-learning algorithms and our characterization of their associated solution sets in weakly communicating MDPs. We will introduce the algorithmic framework in Section~\ref{sec-3.1} and present our main results in Section~\ref{sec-3.2}, followed by a numerical demonstration in Section~\ref{sec-3.3}.

\subsection{Algorithmic Framework} \label{sec-3.1}

We consider a family of average-reward Q-learning algorithms rooted in the RVI approach. These algorithms operate without knowledge of the MDP model parameters, relying instead on random state transitions and rewards generated in the MDP to solve the action-value optimality equation \eqref{eq: action-value optimality equation}. In contrast to the classical RVI algorithm \eqref{eq: S-RVI}, these algorithms employ an \emph{asynchronous} update scheme. Here, updates are performed only for a subset of state-action pairs at each iteration, depending on the available data. Their stochastic and asynchronous nature poses challenges in ensuring desirable behavior, necessitating specific conditions that must be imposed on parameters such as step sizes, asynchronous update schedules, and the type of function $f$ employed in the algorithms. The algorithmic framework we present here was originally formulated by \citet{abounadi2001learning} and recently extended by \citet{wan2021learning}, with further details to be discussed later.

Let $\{\alpha_n\}$ be a sequence of diminishing step sizes, and let $Q_0$ be an arbitrary initial vector of state-action values in $\R^{|\sspace| \times |\aspace|}$. At time step $n \geq 0$, a nonempty subset $Y_n$ of state-action pairs is randomly selected. For each pair $(s,a) \in Y_n$, we observe a random transition and reward according to the transition function $p$ in the MDP, denoted by
$$\left(S_{n+1}^{sa},\, R_{n+1}^{sa}\right) \sim p(\cdot, \cdot \mid s, a)$$ 
(where the notation $X \sim d(\cdot)$ indicates that a random variable $X$ is distributed according to a probability distribution $d$). 
Using these transition and reward data, the algorithm updates the state-action values for those state-action pairs in $Y_n$, while keeping the other components unchanged: 
\begin{align}
    & \text{for $(s,a) \not\in Y_n$: \ \  $Q_{n+1}(s,a) = Q_n(s,a)$;} \notag \\  
& \text{for $(s,a) \in Y_n$:} \notag \\
 & \ \  Q_{n+1}(s, a) = Q_{n}(s, a) + \alpha_{\nu_n(s, a)} \left(R_{n+1}^{sa} - f(Q_n) + \max_{a' \in \aspace} Q_n(S_{n+1}^{sa}, a') - Q_n(s, a) \right).  \label{eq: Extended RVI Q-learning}
\end{align}
Here, $\nu_n(s, a)$ counts the number of updates to the $(s, a)$-component at time step $n$: $\nu_n(s, a) \= \sum_{k = 0}^n \ind\{(s, a) \in Y_k\}$, where $\ind\{\cdot\}$ denotes the indicator function; and $f : \sspace \times \aspace \to \R$ is a Lipschitz continuous function with additional properties to be given shortly. 

Regarding the selection of the set $Y_n$, in a typical RL setting, where the agent follows some policy (possibly history-dependent), known as the \emph{behavior policy}, to generate a sequence of random states, actions, and rewards $S_0, A_0, R_1, S_1, A_1, R_2, \ldots$ in the MDP, $Y_n$ can simply consist of the state-action pair $(S_n, A_n)$ encountered at time step $n$. The update \eqref{eq: Extended RVI Q-learning} then becomes
\begin{equation} \label{eq: alg in RL context}
Q_{n+1}(S_n, A_n) \! =  Q_n (S_n, A_n) + \alpha_{\nu_n(S_n, A_n)}\! \left(\! R_{n+1} - f(Q_n)\! + \max_{a' \in \aspace} Q_n(S_{n+1}, a') - Q_n (S_n, A_n)\!\right)\!.
\end{equation}

The algorithm is subject to a set of conditions. Let us enumerate them first, before a detailed commentary on each one. 

Denote $\ispace = \sspace \times \aspace$. Throughout the paper, let $\zerovec$ and $\onevec$ stand for the vector of all zeros and ones in $\R^d$, respectively, where the dimension $d$ depends on the context.

{\samepage
\begin{myassumption}[conditions on function $f$]\label{assu: f} \hfill
\begin{itemize}[leftmargin=0.7cm,labelwidth=!]
\item[\rm (i)] The function $f$ is Lipschitz continuous; i.e., there is a constant $L \geq 0$ such that $|f(x) - f(y)| \leq L \norm{x - y}$ for all $x, y \in \bbR^{|\ispace|}$.
\item[\rm (ii)] There exists a scalar $u > 0$ such that $f(x + c\onevec) = f(x) + cu$ for all $c \in \bbR$ and $x \in \R^{|\ispace|}$.
\item[\rm (iii)] For all $c \geq 0$ and $x \in \R^{|\ispace|}$, $f(cx) - f(\zerovec) = c(f(x) - f(\zerovec))$.
\end{itemize}
\end{myassumption}
}

\begin{myassumption}[conditions on step sizes $\alpha_n$]\label{assu: stepsize} \hfill
\begin{itemize}[leftmargin=0.7cm,labelwidth=!]
\item[\rm (i)] We have $\sum_{n = 0}^\infty \alpha_n = \infty$ and $\sum_{n = 0}^\infty \alpha_n^2 < \infty$. In addition, $\alpha_n > 0$ for all $n \geq 0$, and $\alpha_{n+1} \leq \alpha_n$ for all $n$ sufficiently large.
\item[\rm (ii)] For $x \in (0, 1)$, 
\begin{align*}
    \sup_n \frac{\alpha_{[xn]}}{\alpha_n} < \infty
\end{align*}
where $[\cdot]$ denote the integer part of $(\cdot)$, and as $n \to \infty$,
\begin{align*}
    \frac{\sum_{k=0}^{[yn]} \alpha_k}{\sum_{k=0}^n \alpha_k} \to 1 \qquad \text{uniformly in $y \in [x, 1]$.}
\end{align*} 
\end{itemize}
\end{myassumption}

\begin{myassumption}[conditions on asynchrony]
\label{assu: update} The following statements hold:
\begin{itemize}[leftmargin=0.7cm,labelwidth=!]
    \item[\rm (i)] There exists a deterministic $\Delta > 0$ such that
\begin{align*}
    \liminf_{n \to \infty} \frac{\nu_n( i)}{n} \geq \Delta \qquad \text{a.s., for all $i \in \ispace$.}
\end{align*}
\item[\rm (ii)] For each $x > 0$, defining $N(n, x) \= \min \left \{m > n: \sum_{k = n}^m \alpha_k \geq x \right\}$, the limit 
\begin{align*}
    \lim_{n \to \infty} \frac{\sum_{k = \nu_n( i)}^{\nu_{N(n, x)}(i)} \alpha_k}{\sum_{k = \nu_n( i')}^{\nu_{N(n, x)}(i')} \alpha_{k}} \qquad \text{exists a.s. for all $i, i' \in \ispace$.}
\end{align*}
\end{itemize}
\end{myassumption}

Let us now discuss these algorithmic assumptions one by one. 

\begin{myremark} \rm 
Assumption~\ref{assu: f} concerning the function $f$ was introduced by \citet{abounadi2001learning} with $u = 1$ in Assumption~\ref{assu: f}(ii). The extension to the more general case $u > 0$ was due to \citet{wan2021learning}. 

In the original formulation by \citet{abounadi2001learning}, $f(cx) = cf(x)$ was required, which differs from Assumption~\ref{assu: f}(iii) where $f(\zerovec)$ need not be zero. However, for analytical purposes, these two conditions are equivalent. If the iterates $\{Q_n\}$ are generated with a function $f$ satisfying Assumption~\ref{assu: f}(iii), they can be viewed as iterates generated employing the function $\hat f(x) = f(x) - f(\zerovec)$, which satisfies $\hat f(c x) = c \hat f(x)$, in an MDP where all rewards are shifted by the constant $f(\zerovec)$. Despite this equivalence, we prefer stating this condition of $f$ in the form given in Assumption~\ref{assu: f}(iii) to clarify the range of functions applicable in practice.\qed
\end{myremark}

Here are two examples of functions that satisfy \cref{assu: f}: 
\begin{align*}
f(x)  = \nu^\top x + b, \qquad   & \text{where} \  \nu \in \R^{|\ispace|} \ \text{with} \ \nu^\top \onevec > 0, \ b \in \R; \\
f(x)  = \beta  \max_{(s,a) \in \sspace \times \aspace} x(s, a) + b, \quad & \text{where}\  \beta > 0, \ b \in \R.
\end{align*}
In particular, \cref{assu: f}(ii) is satisfied with $u = \nu^\top \onevec$ and $u= \beta$, respectively.

For some choices of $f$, the algorithm can take on a different form. A particular example of this is the following algorithm, which was, indeed, the original motivation behind the extension from $u=1$ to $u > 0$.

\begin{example}[Differential Q-learning \citep{wan2021learning}] \label{ex: diff Q} \rm \hfill \\
The Differential Q-Learning algorithm maintains a scalar estimate $\bar R_n$ of the optimal reward rate and updates both $Q_n$ and $\bar R_n$ using the temporal-difference (TD) error. At time step $n$, 
the TD error for each $(s,a) \in Y_n$ is computed as
$$ \delta_n(s, a) =  R_{n+1}^{sa} - \bar R_n + \max_{a' \in \aspace} Q_n(S_{n+1}^{sa}, a') - Q_n (s, a). $$
Then $Q_{n+1}$ and $\bar R_{n+1}$ are updated as follows:
\begin{align}
    Q_{n+1}(s, a) & =  Q_n (s, a) + \alpha_{\nu_n(s, a)} \delta_n(s, a) \ind\{(s, a) \in Y_n\}, \qquad\forall \, s \in \sspace, \ a \in \aspace, \label{eq: diff Q-1} \\
    \bar R_{n+1} & = \bar R_n + \eta \sum_{(s,a) \in Y_n} \alpha_{\nu_n(s, a)} \delta_n(s, a), \label{eq: diff Q-2}
 \end{align}
where $\eta > 0$ is a parameter of the algorithm.

Given that the change from $\bar R_{n}$ to $\bar R_{n+1}$ is precisely $\eta$ times the total changes from  $Q_{n}$ to $Q_{n+1}$, we can express the Differential Q-learning algorithm equivalently in the form of algorithm \eqref{eq: Extended RVI Q-learning}, by writing $\bar R_n = f(Q_n)$ and defining the function $f$ as
\begin{equation} \label{eq: f for diff Q}
f(q) \= \eta \sum_{s \in \sspace, a \in \aspace} q(s, a) - \eta \sum_{s \in \sspace, a \in \aspace} Q_0(s, a) + \bar R_0.
\end{equation}
This function corresponds to the first example of $f$ discussed, with $\nu = \eta \onevec$ and $b$ determined by $\eta$, the initial $Q_0$, and $\bar R_0$.\qed
\end{example}

Let us now discuss the conditions regarding step sizes and asynchrony, which appear to be quite intricate.

First, notice that the step size in each component update follows the specific form $\alpha_{\nu_n(s,a)}$, where $\nu_n(s,a)$ acts as a ``local clock'' for the $(s,a)$-component. Meanwhile, a common deterministic step-size sequence $\{\alpha_k\}$ is employed for all components. 

For comparison, when tackling discounted-reward MDPs or total-reward MDPs of the stochastic shortest path type, the Q-learning algorithm enjoys much greater flexibility in selecting step sizes and asynchronous update schedules while still maintaining convergence guarantees \citep{Tsi94,YuB13}. In those problems, a separate random step-size sequence $\{\beta_{k, sa}\}$ can be used for each component, provided that $\sum_k \beta_{k,sa} = \infty$ and $\sum_k \beta^2_{k, sa} < \infty$ a.s.\ (i.e., only the first half of Assumption~\ref{assu: stepsize}(i) needs to hold). Furthermore, any update schedules ensuring each component is updated infinitely often can be employed. This stands in contrast to the collection of intricate conditions stipulated by Assumption~\ref{assu: update} on the average-reward Q-learning algorithm \eqref{eq: Extended RVI Q-learning}. 

To grasp the purposes of Assumptions~\ref{assu: stepsize} and~\ref{assu: update} and their necessity, it is important to recognize a fundamental distinction between the average-reward case and the discounted- or total-reward scenarios just mentioned: In the average-reward case, the mapping underlying the RVI approach is generally neither a contraction nor a nonexpansive mapping. Coupled with the presence of asynchrony and stochasticity, this presents significant challenges in ensuring desirable convergent algorithmic behavior.

Assumptions~\ref{assu: stepsize} and~\ref{assu: update}, with slight variations in Assumption~\ref{assu: update}(ii), were originally introduced in the broader context of asynchronous SA by \citet{Bor98,Bor00}, and later adopted in average-reward Q-learning by \citet{abounadi2001learning}. These conditions aim to establish partial asynchrony, aligning the asymptotic behavior of the asynchronous algorithm, on average, with that of a synchronous one, facilitating analysis. While a comprehensive understanding of this point requires delving into the details of SA analysis [\citep{Bor98,Bor00}; also see \citep[Chap.\ 7]{Bor09} and \cite{yu2023note}], which is beyond our scope here, we can offer some intuition about these assumptions and demonstrate their satisfaction with examples.

Assumption~\ref{assu: stepsize} requires the step-size sequence $\{\alpha_n\}$ to decrease to $0$ in an appropriate manner. As noted in \citet{Bor98}, some commonly used step-size sequences such as $1/ n$, $1 / (n \log n)$, or $\log n / n$, all satisfy this assumption.

Assumption \ref{assu: update} requires that all components undergo updates \emph{comparably often} in an \emph{evenly distributed} manner. Specifically, Assumption~\ref{assu: update}(i) requires that each component be updated infinitely often. However, it also forbids the relative frequencies of updating any two components from diverging to infinity. Assumption~\ref{assu: update}(ii) represents the most intricate aspect of the conditions governing permissible asynchronous update schedules. This condition is formulated in terms of the deterministic step-size sequence and the random update counts $\{\nu_n(s, a)\}$ for each component, with the purpose of ensuring an even distribution of updates across all components. As a reflection of this point, it is noteworthy that in the presence of both Assumptions~\ref{assu: stepsize} and \ref{assu: update}, the limits whose existence is dictated by this condition must all equal $1$ \citep{Bor98,Bor00}.

Let us illustrate with an example how Assumption \ref{assu: update} can be satisfied in a typical off-policy learning scenario.

\begin{example}\label{example: step size} \rm
Consider a step-size sequence of the form $\alpha_n = c/(n + d)$, where $c > 0$ and $d$ is a positive integer. Such a sequence satisfies Assumption~\ref{assu: stepsize}. Assume that, almost surely, for all $i \in \ispace$, $\lim_{n \to \infty} \nu_n(i)/n$ exists and is nonzero (thus fulfilling Assumption~\ref{assu: update}(i)). Note that the requirement for the existence of these limits is naturally met in scenarios where the behavior policy eventually stops changing with time and matches some stationary policy.

To verify that Assumption~\ref{assu: update}(ii) also holds in this case, we now show that for any given $x > 0$ and $i \in \ispace$, we have $\sum_{k = \nu_n(i)}^{\nu_{N(n,x)}(i)} \alpha_k \to x$ a.s., as $n \to \infty$. To simplify notation, we write $m_n = \nu_n(i)$ and $m_n^x = \nu_{N(n,x)}(i)$. In the derivation below, we will omit the term ``a.s.'' Our assumption implies 
\begin{equation} \label{ex1-eq1}
   \lim_{n \to \infty} m_n = \lim_{n \to \infty} m_n^x = \infty,  \qquad \lim_{n \to \infty} \tfrac{m_n}{n} = \lim_{n \to \infty} \tfrac{m_n^x}{N(n,x)} > 0.
\end{equation} 
Denote by $\epsilon(n)$ a generic term that depends on $n$ and tends to $0$ as $n \to \infty$; the specific expression of $\epsilon(n)$ may vary depending on the context.
Recall that $\sum_{k=1}^n 1/k = \log n - \gamma + \epsilon(n)$, where $\gamma$ is Euler's constant ($\gamma \approx 0.5772$). Using this relation, a direct calculation shows that
\begin{align}
 \textstyle{\sum_{k=m_n}^{m_n^x} \alpha_k = \sum_{k=m_n}^{m_n^x} \tfrac{c}{k+d}} & = c \log \tfrac{m_n^x + d}{m_n +d - 1} + \epsilon(n)  \notag \\
     & = c \log \tfrac{m_n^x + d}{N(n,x)} - c \log \tfrac{m_n + d - 1}{n} + c \log \tfrac{N(n,x)}{n} + \epsilon(n). \label{ex1-eq2}
\end{align}  
By \eqref{ex1-eq1}, $c \log \frac{m_n^x + d}{N(n,x)} - c \log \frac{m_n + d - 1}{n} \to 0$ as $n \to \infty$.
For the term $c \log \tfrac{N(n,x)}{n}$, since
$$ \textstyle{\sum_{k=n}^{N(n,x)} \alpha_k} = c  \log \tfrac{N(n,x) +d }{n +d -1} + \epsilon(n) = c  \log \tfrac{N(n,x) }{n} + \epsilon(n)$$
and $\lim_{n \to \infty} \sum_{k=n}^{N(n,x)} \alpha_k = x$ by the definition of $N(n,x)$, we have 
$\lim_{n \to \infty} c  \log \frac{N(n,x)}{n} = x$. Then by \eqref{ex1-eq2}, $\lim_{n \to \infty} \sum_{k=m_n}^{m_n^x} \alpha_k = x$. 

Hence, Assumption~\ref{assu: update}(ii) holds with $\lim_{n \to \infty} \frac{\sum_{k = \nu_n(i)}^{\nu_{N(n,x)}(i)} \alpha_k}{\sum_{k = \nu_n(i')}^{\nu_{N(n,x)}(i')} \alpha_k} = 1$ a.s.\ for all $i, i' \in \ispace$. Indeed, under Assumptions~\ref{assu: stepsize} and~\ref{assu: update}, it is necessary for these limits to equal $1$ [cf. the proof of \cite[Theorem 3.2]{Bor98} and \cite{Bor00}], as mentioned above.\qed
\end{example}

\subsection{Main Results} \label{sec-3.2}

Recall that $\calQ$ is the set of solutions to the action-value optimality equation \eqref{eq: action-value optimality equation}, and $\optimalr$ is the optimal reward rate.
We will show that in a weakly communicating MDP, the sequence $\{Q_n\}$ generated by algorithm \eqref{eq: Extended RVI Q-learning} converges a.s.\ to the subset of $\calQ$ constrained by $f(q) = \optimalr$:
\begin{align}\label{eq: mdp sol set}
    \ExtRVIQsolutionq \= \{q \in \calQ : f(q) = \optimalr\}.
\end{align}

First, let us characterize this solution set $\ExtRVIQsolutionq$ for the algorithm. Based on the theory of \citepalias{ScF78}, the set $\calQ$ is nonempty, closed, unbounded, connected, and possibly nonconvex. Further, as discussed in \cref{sec: wc mdps}, for a weakly communicating MDP, the solutions in $\calQ$ need not be unique up to an additive constant. With this understanding of the structure of $\calQ$, we can characterize the set $\ExtRVIQsolutionq$ as follows (the proof of which will be given in \cref{sec: solution set} in the broader context of SMDPs):

\begin{mytheorem} \label{thm: MDP characterize Q}
If the MDP is weakly communicating and \cref{assu: f} holds, then $\ExtRVIQsolutionq$ is nonempty, compact, connected, and possibly nonconvex.
\end{mytheorem}

Moreover, as we will show in Section~\ref{sec-7.2} (cf.\ \cref{thm-dim-Qs}), the solutions in $\ExtRVIQsolutionq$ have precisely one lower degree of freedom than those in $\calQ$. Thus, for a weakly communicating MDP, the set $\ExtRVIQsolutionq$ is, in general, not a singleton, in contrast to the singleton case focused in prior studies \citep{abounadi2001learning,wan2021learning}.

For a vector $q$ of state-action values, we call a deterministic policy $\pi: \sspace \to \aspace$ \emph{greedy w.r.t.\ q}, if $\pi(s) \in \argmax_{a \in \aspace} q(s, a)$ for all states $s \in \sspace$. The next theorem is our convergence result for average-reward Q-learning in weakly communicating MDPs.

\begin{mytheorem}[convergence theorem] \label{thm: Extended RVI Q-learning}
Consider algorithm \eqref{eq: Extended RVI Q-learning}. If the MDP is weakly communicating and Assumptions~\ref{assu: f}, \ref{assu: stepsize}, and \ref{assu: update} are satisfied, then almost surely, the following hold:
\begin{itemize}[leftmargin=0.7cm,labelwidth=!]
\item[\rm (i)] As $n \to \infty$, $Q_n$ converges to a sample path-dependent compact connected subset of $\ExtRVIQsolutionq$, and $f(Q_n)$ converges to the optimal reward rate $\optimalr$.
\item[\rm (ii)] For all sufficiently large $n$, the greedy policies w.r.t.\ $Q_n$ are all optimal.
\end{itemize}
\end{mytheorem} 

We will prove part (i) of this theorem in \cref{sec: convergence proofs}. The proof will use ODE-based arguments to analyze asynchronous SA algorithms. As part (ii) of this theorem is a direct consequence of part (i) and the compactness of $\ExtRVIQsolutionq$, we give here the proof of part (ii) first, assuming that part (i) has been established.

\medskip
\begin{proofof}{\cref{thm: Extended RVI Q-learning}(ii)}
Recall that any policy that is greedy w.r.t.\ a solution $\bar q$ of the optimality equation \eqref{eq: action-value optimality equation} is an optimal policy \citepalias[Theorem 3.1(e1)]{ScF78}. We define an open set $G \= \cup_{\bar q \in \ExtRVIQsolutionq} G_{\bar q}$, where $G_{\bar q}$ is a sufficiently small open neighborhood of $\bar q$ such that for all $q \in G_{\bar q}$, $\argmax_{a \in \A} q(s, a) \subset \argmax_{a \in \aspace} \bar q(s,a)$ for all $s \in \sspace$. Observe that any policy greedy w.r.t.\ some $q \in G$ is also greedy w.r.t.\ some $\bar q \in \ExtRVIQsolutionq$ and is, therefore, an optimal policy.
If a sequence $\{Q_n\}$ converges to $\ExtRVIQsolutionq$, then, since $\ExtRVIQsolutionq \subset G$ is compact (\cref{thm: MDP characterize Q}), $\{Q_n\}$ must eventually enter and never leave the open set $G$. (Otherwise, a subsequence $\{Q_{n_k}\}$ could be found in the closed set $G^c$, which has a positive distance from the compact set $\ExtRVIQsolutionq$, contradicting the convergence of $\{Q_n\}$ to $\ExtRVIQsolutionq$.) Consequently, if $Q_n \to \ExtRVIQsolutionq$, then for sufficiently large $n$, any greedy policy w.r.t.\ $Q_n$ is optimal. \cref{thm: Extended RVI Q-learning}(ii) now follows from this argument and \cref{thm: Extended RVI Q-learning}(i).
\end{proofof}

\vspace*{-0.7cm}
\begin{myremark} \rm 
\cref{thm: Extended RVI Q-learning} generalizes previous convergence results on RVI Q-learning by \citet[Sec.\ 3]{abounadi2001learning} and \citet{wan2021learning}, which are applicable only to subfamilies of weakly communicating MDPs with singleton solution sets $\ExtRVIQsolutionq$, as mentioned earlier. Moreover, concerning algorithmic stability, the proof outlined in \citet{wan2021learning} has a notable gap, while the arguments presented in \citet[Sec.\ 3.2]{abounadi2001learning} also lack some essential details. We will discuss this in more detail in \cref{remark: sa result compare}.\qed
\end{myremark}

Recall that the state space $\sspace$ of a weakly communicating MDP can be partitioned into a closed communicating class $\S^o$ of states and a (possibly empty) set $\sspace \setminus \S^o$ consisting of states that are transient under all policies. For the purpose of finding an optimal policy, it suffices to solve the optimality equation on the closed subset $\S^o$ of $\sspace$. Therefore, the requirements on the update schedules can be relaxed accordingly, instead of imposing them on all state-action pairs as in \cref{assu: update}. Such extensions are relatively straightforward; \cref{thm: Extended RVI Q-learning} itself can be applied to the communicating MDP on the state space $\S^o$ to ensure convergence guarantees under suitably relaxed conditions. 

In the rest of this subsection, let us discuss a specific instance of these extensions, which is important in the context of RL, particularly where the knowledge of $\S^o$ is not available.
Consider the off-policy learning scenario described earlier before \eqref{eq: alg in RL context}, where an agent selects actions according to some behavior policy, resulting in a single data stream $\{(S_n, A_n, R_{n+1})\}$. This data is used with the update rule \eqref{eq: alg in RL context} to compute the iterates $\{Q_n\}$ by the agent. As $\S^o$ is a closed subset and states outside $\S^o$ are transient under any policy, the agent will inevitably enter $\S^o$ and remain within this part of the state space indefinitely. At this point, we can focus on the MDP defined on $\S^o$ and apply \cref{thm: Extended RVI Q-learning} to infer the asymptotic behavior of the algorithm \eqref{eq: alg in RL context}. This leads to the following corollary, presented after some necessary notation.

Let $\ispace^{o} \= \{ (s,a) : s \in \S^o, a \in \A\}$. We express a vector $q$ of state-action values as $q = (q^o, q^t)$, where $q^o$ represents the components of $q$ corresponding to the subset $\ispace^{o}$, and $q^t$ represents the rest of the components. Namely, $q^o = (q(s,a): (s,a) \in \ispace^{o})$ and $q^t = (q(s,a): (s,a) \not\in \ispace^{o})$. Let $\calQ^o$ denote the set of solutions to the action-value optimality equation \eqref{eq: action-value optimality equation} for the communicating MDP on the state-action space $\ispace^{o}$.

\begin{mycorollary} \label{cor: extended rvi q learning}
Consider a weakly communicating MDP and the algorithm \eqref{eq: alg in RL context} in the off-policy learning setting described above. Suppose that \cref{assu: stepsize} holds and in addition: 
\begin{itemize}[leftmargin=0.7cm,labelwidth=!]
\item[\rm (i)] For each $q^t \in \R^{|\ispace \setminus \ispace^o|}$, the function $f_{q^t}(\cdot) \= f(\cdot, q^t)$ satisfies \cref{assu: f} with $\ispace^{o}$ in place of $\ispace$.
\item[\rm (ii)] \cref{assu: update} holds with $\ispace^o$ in place of $\ispace$.
\end{itemize}
Then, almost surely, as $n \to \infty$, $f(Q_n) \to \optimalr$, while the $Q^{o}_n$-component of $Q_n$ converges to a sample path-dependent compact connected subset of $\calQ^o$. Part (ii) of \cref{thm: Extended RVI Q-learning} regarding the optimality of greedy policies for sufficiently large $n$ remains valid.   

In cases where the Differential Q-learning algorithm (\cref{ex: diff Q}) is used, or when the function $f$ meets the criteria of \cref{assu: f} without dependence on $q^t$, condition (i) can be omitted as it is automatically fulfilled.
\end{mycorollary}

\begin{proof}
Let $\tilde{n}$ be the a.s.\ finite random time step at which the system enters $\S^o$. After time step $\tilde n$, the values of the $Q_{n}^t$-component of $Q_n$ remain unchanged, and the algorithm \eqref{eq: alg in RL context} effectively operates in the communicating MDP on $\S^o$ with the associated function $f_{\tilde{q}^t}$, where $\tilde{q}^t = Q_{\tilde n}^t$.  Under the assumptions of the corollary, \cref{thm: Extended RVI Q-learning} applies to this MDP on $S^o$ and the function $f_{\tilde{q}^t}$, with the corresponding solution set $\calQ_\infty$ being the subset of $\calQ^o$ constrained by $f_{\tilde{q}^t}(q^o) = \optimalr$. These observations lead to the main conclusions of the corollary, as discussed earlier. 

For the two special cases in the last assertion of the corollary, the second one is obvious. In the first case, concerning the Differential Q-learning algorithm, condition (i) can be verified directly from the expression of $f$ given in \eqref{eq: f for diff Q}.    
\end{proof}
\vspace*{-0.5cm}

\subsection{Empirical Verification of the Convergence Theorem} \label{sec-3.3}

We now present a set of experiments that empirically verify \cref{cor: extended rvi q learning} by evaluating two members within the RVI Q-learning family of algorithms \eqref{eq: alg in RL context}. The two tested members are Differential Q-learning (\cref{ex: diff Q}) and an algorithm whose $f$ function refers to the action value of a single fixed state-action pair. To streamline our presentation, in this section, we use the family name ``RVI Q-learning'' to refer to the latter family member. 

The tested domains included a communicating MDP and a weakly communicating MDP, as depicted in Figure \ref{fig: c1 two tested mdps}. The latter MDP is essentially the former with an additional state incorporated. 

\begin{figure*}[htbp]
\centering
\begin{subfigure}[t]{0.48\textwidth}
    \centering
    \includegraphics[width=\textwidth]{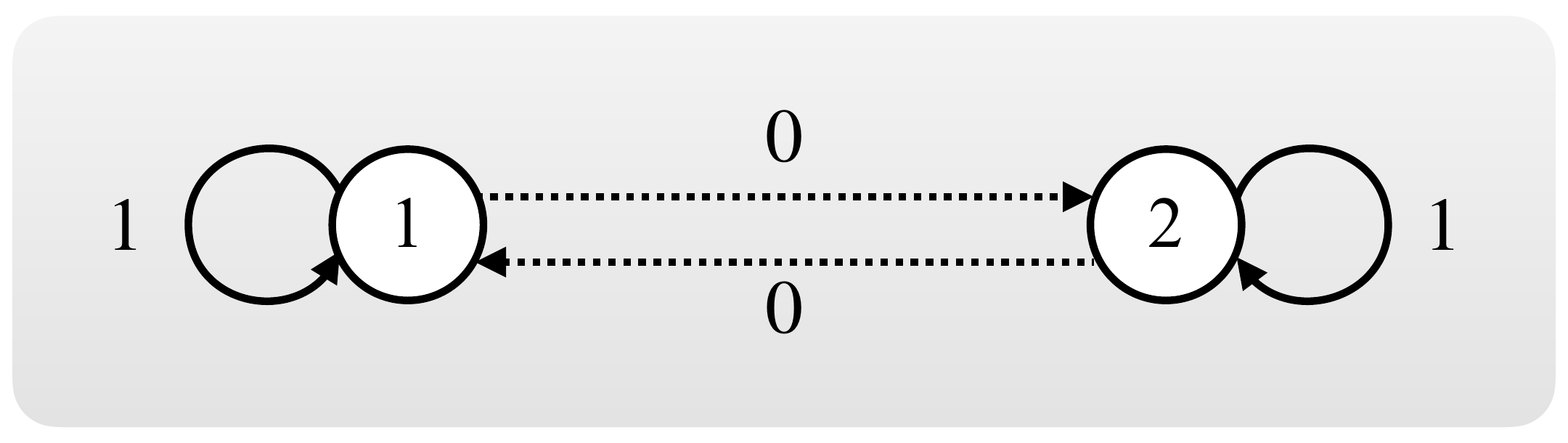}
    \subcaption{A communicating MDP. States $1$ and $2$ are in the same communicating class. For each of the two states, taking action \texttt{solid} stays at the same state and receives a reward of one, and taking action \texttt{dashed} moves to the other state and receives a reward of zero. The initial state of the MDP is state $1$.}
    \label{fig: c1 control exp mdp}
\end{subfigure}\hfill%
\begin{subfigure}[t]{0.48\textwidth}
    \centering
    \includegraphics[width=\textwidth]{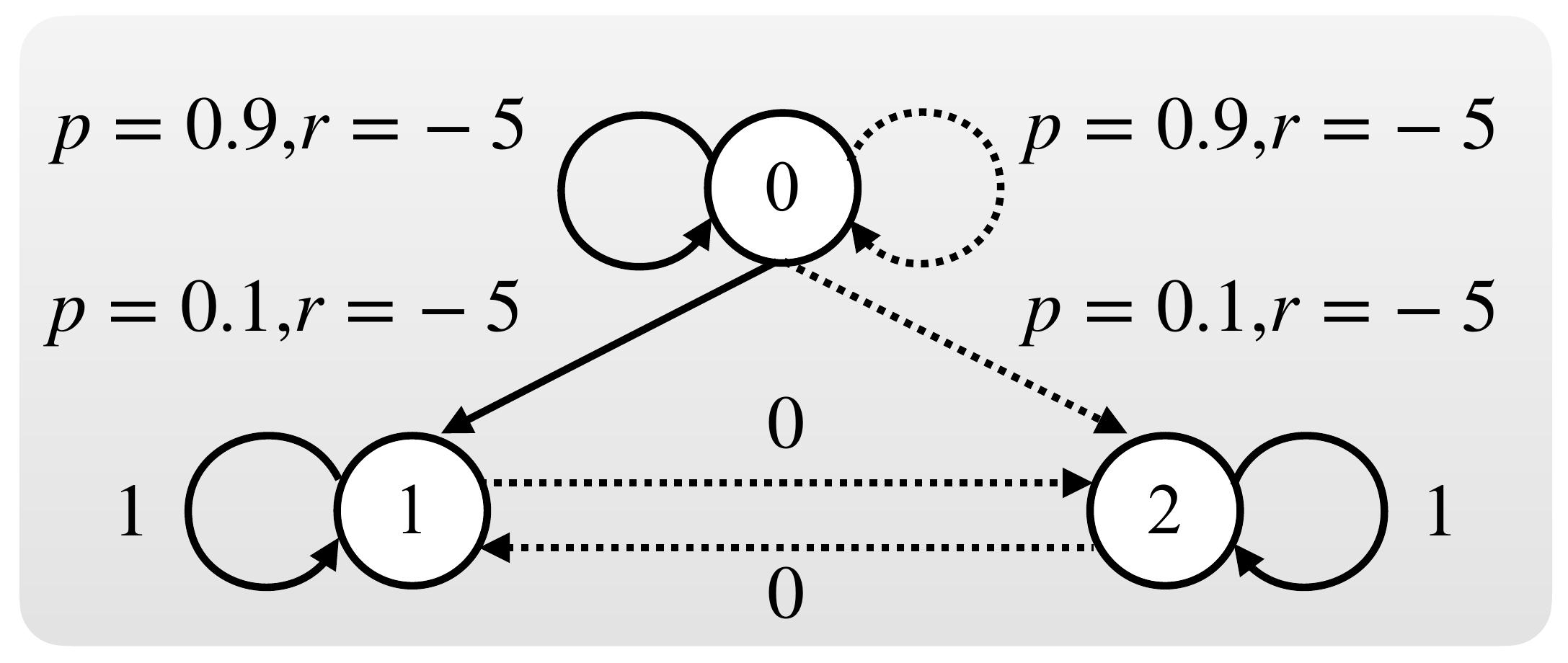}
    \subcaption{A weakly communicating MDP constructed by adding one more state (State $0$) to the MDP shown on the left panel. In state $0$, taking both \texttt{solid} and \texttt{dashed} actions stays at state $0$ with probability $0.9$. The MDP moves to state $1$ with probability $0.1$ given action \texttt{solid} and to state $2$ with probability $0.1$ given action \texttt{dashed}. The reward starting from state $0$ is always $-5$. The initial state of the MDP is state $0$.}
    \label{fig: c1 control exp wc mdp}
\end{subfigure}
    \caption{Tested MDPs for verifying the convergence of Differential Q-learning and RVI Q-learning when the solution set has more than one degree of freedom.}
    \label{fig: c1 two tested mdps}
\end{figure*}

\begin{figure*}[t!]
\centering
\begin{subfigure}[t]{0.48\textwidth}
    \centering
    \includegraphics[width=\textwidth]{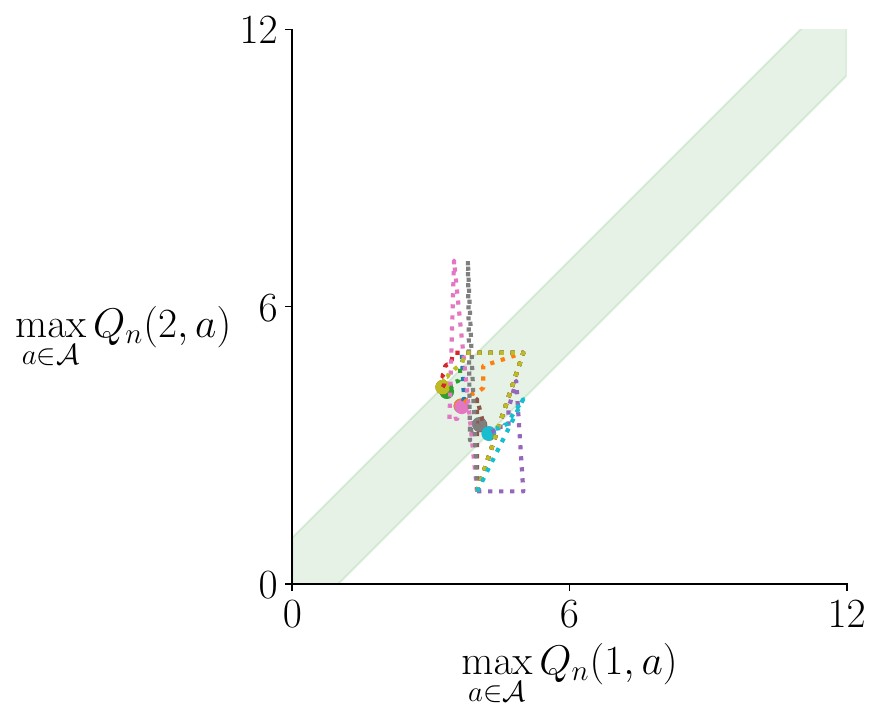}
    \subcaption{Differential Q-learning in the communicating MDP (\cref{fig: c1 control exp mdp}).}
    \label{fig: c1 diff q exp value dynamics communicating mdp}
\end{subfigure}\hfill%
\begin{subfigure}[t]{0.48\textwidth}
    \centering
    \includegraphics[width=\textwidth]{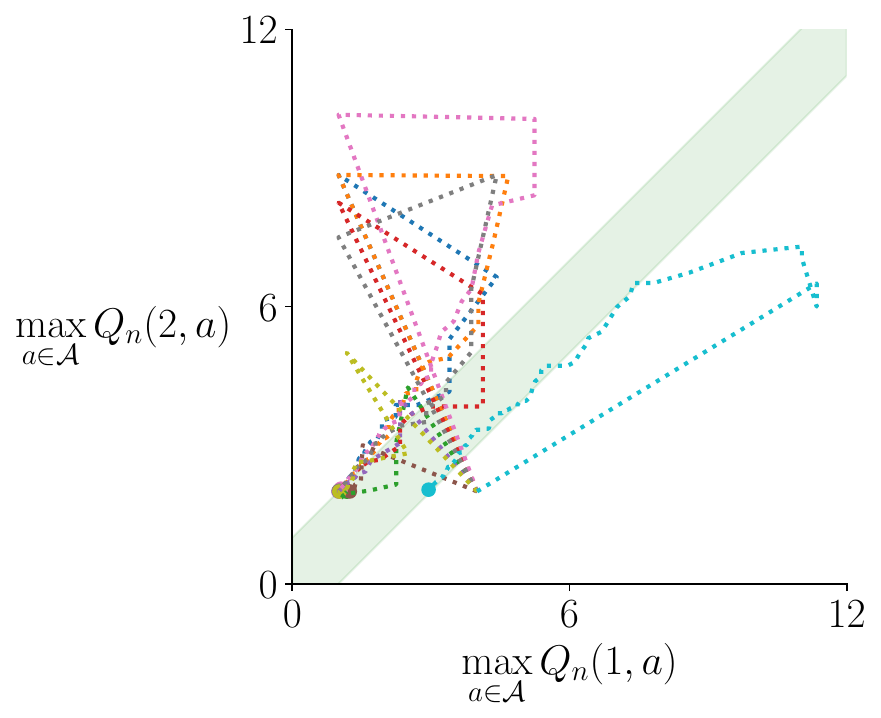}
    \subcaption{RVI Q-learning in the communicating MDP (\cref{fig: c1 control exp mdp}).}
    \label{fig: c1 rvi q exp value dynamics communicating mdp}
\end{subfigure}
\begin{subfigure}[t]{0.48\textwidth}
    \centering
    \includegraphics[width=\textwidth]{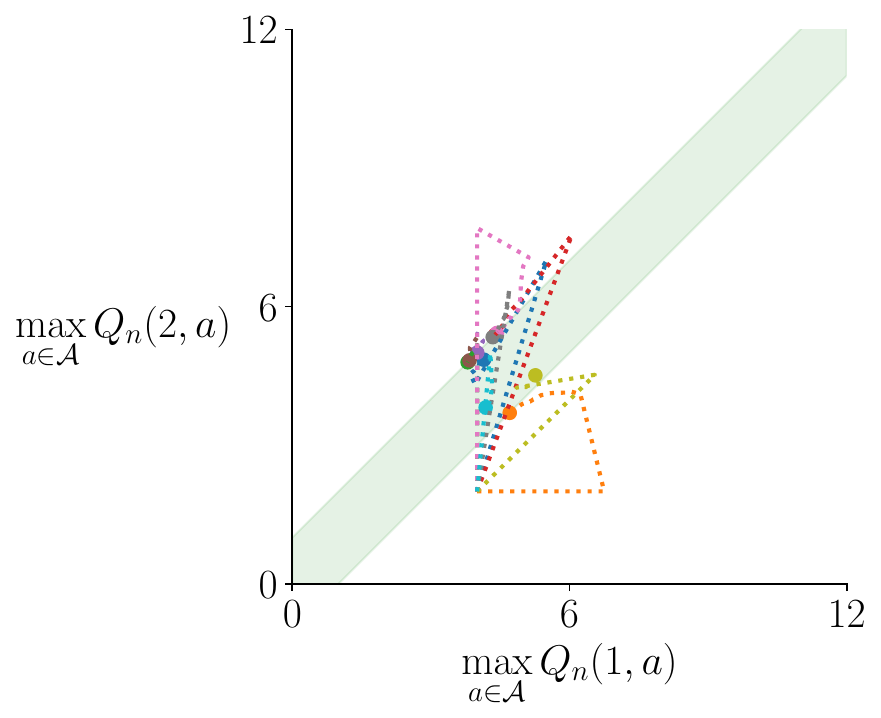}
    \subcaption{Differential Q-learning in the weakly communicating MDP (\cref{fig: c1 control exp wc mdp}). 
    }
    \label{fig: c1 diff q exp value dynamics weakly communicating mdp}
\end{subfigure}\hfill%
\begin{subfigure}[t]{0.48\textwidth}
    \centering
    \includegraphics[width=\textwidth]{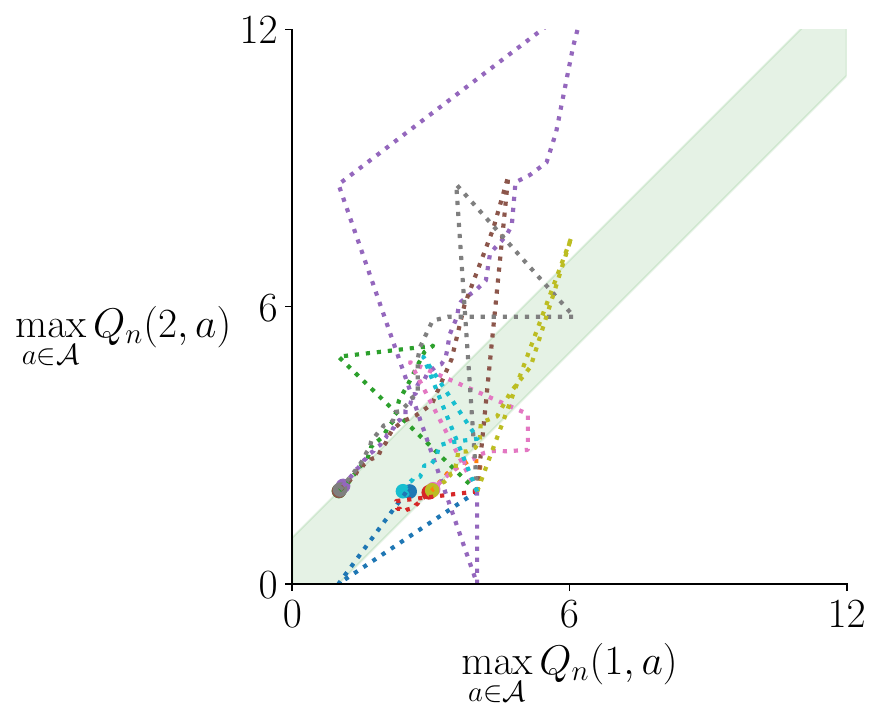}
    \subcaption{RVI Q-learning in the weakly communicating MDP (\cref{fig: c1 control exp wc mdp}).}
    \label{fig: c1 rvi q exp value dynamics weakly communicating mdp}
\end{subfigure}
    \caption{Dynamics of the estimated values produced by Differential Q-learning and RVI Q-learning in the two MDPs shown in \cref{fig: c1 two tested mdps}. The green regions denote $\calQ^o$. 
    }
    \label{fig: c1 control exp value dynamics}
\end{figure*}

For Differential Q-learning, we set $\eta = 1$, $\bar R_0 = 0$, $Q_0(1, \cdot) \equiv 4, Q_0(2, \cdot) \equiv 2$, and $Q_0(0, \cdot) \equiv 0$ in the weakly communicating MDP. Expressing Differential Q-learning's update rules \eqref{eq: diff Q-2} in the form of \eqref{eq: alg in RL context}, we have $f(q) = \sum_{s \in \sspace, a \in \aspace} q(s, a) - 12$. For RVI Q-learning, we let $f$ refer to the estimated action value of the state-action pair $(q, \texttt{dashed})$ (i.e., $f(q) = q(1, \texttt{dashed})$). $Q_0$ was chosen to be the same as in Differential Q-learning. Notably, both selections of the $f$ function adhere to condition (i) in \cref{cor: extended rvi q learning}.

Data is generated in the aforementioned off-policy learning setting. Specifically, the agent started from state $1$ in the communicating MDP and state $0$ in the other MDP. In both MDPs and for both tested algorithms, the agent then follows a behavior policy that chooses action \texttt{solid} with probability $0.8$, and action \texttt{dashed} with probability $0.2$ for all states. The step-size sequence $\alpha_n = 1/n$, ensuring that \cref{assu: stepsize} is satisfied. The choice of the behavior policy and the step-size sequence also guarantee that condition (ii) of \cref{cor: extended rvi q learning} is satisfied in both MDPs.
We performed $10$ runs for each algorithm in each MDP. Each run lasted for $20,000$ steps. For every ten steps, we recorded the higher estimated action values. 

The trajectories of the higher estimated action values of the two tested algorithms in the two MDPs are shown in the four sub-figures of \cref{fig: c1 control exp value dynamics}. In these sub-figures, each color represents the trajectory in one run. Estimated action values after $20,000$ steps are marked with a dot, matching the color of the trajectory. Notably, all colored dots fall within the green regions, which denote $\calQ^o$. This empirical result confirms that both algorithms converge to $\calQ^o$ in both MDPs, as predicted by \cref{cor: extended rvi q learning}. 
 
\section{Convergence of Options Algorithms}\label{sec: options algorithms}
This section extends our previous results for RVI Q-learning to hierarchical decision-making in MDPs involving temporally abstracted courses of actions, known as \emph{options}, rather than primitive actions. Associated with options, the underlying decision problems are SMDPs. Our focus is on two average-reward options algorithms introduced by \citet{wan2021average}: inter-option Q-learning and intra-option Q-learning. While the inter-option algorithm is more general and applicable to SMDPs, the intra-option algorithm exploits options' internal structures for computational efficiency.
 
\cites{wan2021average} convergence analyses (previously noted to contain gaps) required a unichain condition on the associated SMDPs. In this section, we characterize the solution properties and fully establish the convergence for these algorithms, under the much weaker assumption that the SMDPs are weakly communicating. 
 
We begin by introducing basic definitions and outlining optimality results for average-reward SMDPs (\cref{sec: smdp}). We then provide a formal description of decision-making with options, their connection to SMDPs, and the basis for the two option algorithms (\cref{sec: options background}), before presenting these algorithms alongside our convergence results in Sections\ \ref{sec: inter-option algorithms} and \ref{sec: intra-option algorithms}.

\subsection{Average-Reward Weakly Communicating SMDPs} \label{sec: smdp}

SMDPs generalize MDPs by providing greater flexibility in modeling temporal dynamics. Unlike MDPs, where state transitions occur at fixed intervals, SMDPs allow for transitions with random durations, known as \emph{holding times}. In the context of the options algorithms we will introduce later in this section, holding times in associated SMDPs correspond to the duration an option takes to terminate once initiated from a state in the MDP. To focus our discussion, we will consider finite state and action SMDPs where holding times are constrained to be greater than a fixed positive number. Furthermore, for notational simplicity, we will assume that both rewards and holding times are discrete, taking only countable values. Although we restrict our attention to these settings, many results presented here extend to more general SMDPs.

Specifically, we consider an SMDP defined by the tuple $(\sspace, \aspace, \rspace, \lspace, \trans)$. Here $\sspace$ ($\aspace$) is a finite set of states (actions), and $\rspace \subset \R$ ($\lspace \subset \R_+$) is a countable set of possible rewards (holding times). The transition function $\trans: \sspace \times \aspace \to \Delta(\sspace \times \rspace \times \lspace)$ governs the state evolution and reward generation in the SMDP. If the system is currently in state $s \in \sspace$ and action $a \in \aspace$ is applied, then with probability $\trans(s', r, l \mid s, a)$, the system transitions to state $s'$ at time $l \in \lspace$ and incurs reward $r \in \rspace$. For the remainder of this paper, we implicitly assume the following regularity condition on the SMDP model.

\begin{myassumption}\label{assu: smdp}
The SMDP $(\sspace, \aspace, \rspace, \lspace, \trans)$ is such that:
\begin{itemize}[leftmargin=0.7cm,labelwidth=!]
\item[\rm (i)] For some $\epsilon > 0$, $l \geq \epsilon$ for all possible holding times $l \in \lspace$. 
\item[\rm (ii)] For each state-action pair $(s,a) \in \sspace \times \aspace$, the expected holding time and expected reward incurred with the transition from $(s,a)$ are both finite.
\end{itemize}
\end{myassumption}

In an SMDP, actions are applied initially at time $0$ and subsequently at discrete moments upon state transitions. Policies, whether history-dependent or stationary, randomized or deterministic, are defined similarly to MDPs (cf.\ \cref{MDPs with the Average-Reward Criterion}). However, in SMDPs, $n$ represents the number of transitions, and the history up to the $n$th transition before the next action selection includes states, actions, rewards, and holding times realized up to that point: $s_0, a_0, r_1, l_1, s_1, \ldots, a_{n-1}, r_n, l_n, s_n$.

Similar to MDPs, the average reward rate of a policy $\pi$ is defined for each initial state $s \in \sspace$ as: 
\begin{equation} \label{eq: r-pi smdp}
 r(\pi, s) \= \, \textstyle{ \liminf_{t \to \infty}  t^{-1} \, \E_\pi \left [\sum_{n = 1}^{N_t} R_n \mid S_0 = s \right]\!.}
\end{equation} 
Here the expectation is taken w.r.t.\ the probability distribution of the random process $\{(S_n, A_n, R_{n+1}, L_{n+1})\}_{n \geq 0}$ induced by the policy $\pi$ and initial state $S_0=s$. The summation $\sum_{n = 1}^{N_t} R_n$ represents the total rewards received by time $t$, where $N_t$ counts the number of transitions by that time, defined as $N_t = \max \{ n \mid T_n  \leq t\}$ with $T_n \= \sum_{i=1}^n L_i$ and $T_0=0$. If the policy $\pi$ is stationary, then in the above definition of $r(\pi, s)$,  the $\liminf$ can be replaced by $\lim$ according to renewal theory [cf.\ \citep{Ros70}]. 

The optimal reward rate $\optimalr(\cdot)$ and optimal policies are defined similarly to MDPs, with the existence of a deterministic and stationary optimal policy well-established [cf.\ \citepalias{ScF78,yushkevich1982semi}]. 
Furthermore, stationary optimal policies $\pi^*$ enjoy a stronger sense of optimality, as indicated by the inequality: for any history-dependent policy $\pi$,
$$ \textstyle{ \lim_{t \to \infty} t^{-1}\, \bbE_{\pi^*} \left [\sum_{n = 1}^{N_t} R_n \mid S_0 = s \right] \geq \limsup_{t \to \infty} t^{-1}\, \bbE_\pi \left [\sum_{n = 1}^{N_t} R_n \mid S_0 = s \right]\!.}
$$

The classification of an SMDP as weakly communicating, communicating, or unichain is exactly as in the case of an MDP, as the definitions depend only on the communicating structure among the states (cf.\ \cref{sec: wc mdps}). Similar to the MDP case, in a weakly communicating SMDP, the optimal reward rate $\optimalr$ remains constant. In this case, the average-reward optimality equation can be expressed in two equivalent forms, either as the \emph{state-value optimality equation} or as the (state and) \emph{action-value optimality equation} [cf.\ \citepalias{ScF78,yushkevich1982semi}]:
\begin{align}
v(s) & = \max_{a \in \aspace} \left\{ r_{sa} - \bar r \cdot l_{sa} +  \sum_{s' \in \sspace} p_{ss'}^a  v(s') \right\},\qquad \forall \, s \in \calS,  \label{eq: SMDP state-value optimality equation}  \\
q(s, a) & =  r_{sa} - \bar r \cdot l_{sa} +  \sum_{s' \in \sspace} p_{ss'}^a \max_{a' \in \aspace} q(s', a'),\qquad \ \, \forall \, s \in \calS, \, a \in \aspace.  \label{eq: SMDP action-value optimality equation}
\end{align} 
In these optimality equations, we solve for $(\bar r, v)$ or $(\bar r, q)$. For each state-action pair $(s, a)$, $r_{sa}$ and $l_{sa}$ are the expected reward and expected holding time, respectively, while $\trans_{ss'}^a$ is the probability of transitioning from state $s$ to state $s'$ when taking action $a$. That is,
\begin{align}\label{eq: rsa}
    r_{sa} \= \sum_{s' \in \sspace} \sum_{r \in \rspace} \sum_{l \in \lspace} r\cdot \trans(s', r, l \mid s, a), \ \quad l_{sa} \= \sum_{s' \in \sspace} \sum_{r \in \rspace} \sum_{l \in \lspace}  l \cdot  \trans(s', r, l \mid s, a),
\end{align}
and
\begin{align}\label{eq: pssa}
    p_{ss'}^a \= \sum_{r \in \rspace} \sum_{l \in \lspace} \trans(s', r, l \mid s, a).
\end{align}

These optimality equations admit solutions, with solution structures similar to those described in \cref{sec: wc mdps} for MDPs. Specifically, any solution $(\bar r, v)$ or $(\bar r, q)$ will have its $\bar r$-component equal the optimal reward rate $r_*$. Moreover, any stationary policy that solves the corresponding maximization problems on the r.h.s.\ of either equation is optimal.

In weakly communicating SMDPs, the solutions of $v$ for \eqref{eq: SMDP state-value optimality equation} and the solutions of $q$ for \eqref{eq: SMDP action-value optimality equation}, denoted as $\mathcal{V}$ and $\mathcal{Q}$ respectively, may not be unique up to an additive constant. Instead, they can have multiple degrees of freedom, as characterized by \citepalias{ScF78} (cf.\ Sections~\ref{sec: wc mdps} and \ref{sec: degree of freedom Q}).

As noted in \cref{sec: classical rvi}, Schweitzer's RVI algorithm was originally proposed for solving these average-reward optimality equations in SMDPs \citep{schweitzer1971iterative,platzman1977improved}. Later, we will mention some details of this algorithm (cf.\ \cref{footnote-srvi}), as it served as the inspiration for one of the options algorithms we will discuss in this section.
 
\subsection{Average-Reward Learning with Options: Problem Formulations} \label{sec: options background}

We now return to the topic of average-reward MDPs, but with a different focus: finding the best policies among a class of \emph{hierarchical} policies defined by \emph{options}. In this context, an option represents a predefined low-level mechanism for controlling the system, while a hierarchical policy dictates how to switch between these low-level mechanisms. Formally, an option $o$ comprises an associated (possibly history-dependent) policy $\pi_o$ along with initiation and termination rules. The initiation rule specifies the states at which option $o$ can be activated. Once activated at some (random) time step $k$, actions are chosen according to the policy $\pi_o$, treating $k$ as the starting time step, until the option is deactivated based on its associated (possibly probabilistic) termination rule. During this period, decisions regarding actions and termination rely on ``local'' histories, comprising realized outcomes since the option's activation. The hierarchical policies of interest are history-dependent policies in the MDP framework, which specify the initial activation of options and how to switch to other options once an option terminates.

In this paper, we focus on the setting where the collection $\ospace$ of options is finite, and each option $o$ is associated with a \emph{stationary policy} $\pi_o$ and a \emph{memoryless, ``stationary'' termination rule} that depends solely on the current state of the system. Specifically, when option $o \in \mathcal{O}$ is active, the probability of taking action $a$ at state $s$ is given by $\pi(a \mid s, o) \= \pi_o( a \mid s)$. After option $o$ has been activated for one time step, before each subsequent action is selected, it is decided whether option $o$ should be terminated. The termination probability, denoted by $\beta(s,o)$ when $s$ is the current state, governs this decision. Upon termination, another option will be immediately activated depending on the hierarchical policy employed. To simplify notation, we assume that at each state, all options from $\ospace$ can be initiated. Thus, $\pi( a \mid s, o)$ and $\beta(s,o)$, where $(s, o) \in \sspace \times \ospace$ and $a \in \aspace$, are the given parameters associated with the set $\ospace$ of options.

In addition, we make the assumption throughout this section that the options satisfy the following condition. This assumption ensures that for each option, both the cumulative rewards and the duration of its active phase have finite expectations and variances.

\begin{myassumption}\label{assu: option assumption}
For each option in $\ospace$, once activated, there exists a nonzero probability that the option terminates in $\cardS$ time steps, irrespective of the state from which it is initiated.
\end{myassumption}

The problem at hand is to determine an optimal policy within the class $\Pi_{\ospace}$ of hierarchical policies associated with $\ospace$. A hierarchical policy $\mu$ is considered \emph{optimal} if it achieves the maximum average reward rate $r(\mu, s)$ (as defined in \eqref{eq: r-pi mdp}) among this class, for \emph{all} initial states $s \in \sspace$. This problem can be formulated in two ways, which will be explained shortly. The first formulation, known as the \emph{inter-option} formulation, does not rely on the internal structures of the options and reformulates the problem as finding a stationary optimal policy in an average-reward SMDP. The second formulation, called the \emph{intra-option} formulation, leverages the options' structures, particularly their memoryless, stationarity properties.

\subsubsection{Inter-Option Formulation} \label{sec-4.2.1}
Given an MDP $(\sspace, \aspace, \rspace, p)$ and a finite set $\ospace$ of options satisfying \cref{assu: option assumption}, we define an associated SMDP $(\sspace, \ospace, \orspace, \lspace, \otrans)$ on the state-action space $\sspace \times \ospace$:
\begin{itemize}[leftmargin=0.5cm,labelwidth=!]
\item The set $\orspace$ of possible rewards in the SMDP consists of all possible cumulative rewards during the active phase of each option in the MDP, while the set $\lspace$ of possible holding times includes all possible lengths of these phases.
\item The transition function $\otrans: \sspace \times \ospace \to \Delta(\sspace \times \orspace \times \lspace)$ of the SMDP is defined as follows: For each state-option pair $(s,o)$, $\otrans(s', r, l \mid s, o)$ is assigned the probability, in the MDP, that option $o$, if initiated from state $s$, terminates exactly $l$ time steps later, ending at state $s'$ and resulting in cumulative reward $r$.
\end{itemize}

Under \cref{assu: option assumption}, this SMDP satisfies the regularity condition required in \cref{assu: smdp}. Any policy $\mu$ for this SMDP corresponds to a hierarchical policy in $\Pi_\ospace$ for the MDP (also denoted by $\mu$). Moreover, under \cref{assu: option assumption}, it is not hard to show that the average reward rate of $\mu$ in the SMDP, as defined by \eqref{eq: r-pi smdp}, coincides with its average reward rate in the MDP, as defined by \eqref{eq: r-pi mdp}.

Let us denote all such policies $\mu$ in the MDP by $\hat{\Pi}_\ospace$. Note that $\hat{\Pi}_\ospace$ is a proper subset of $\Pi_\ospace$. A hierarchical policy in $\hat{\Pi}_\ospace$ decides which option to activate next at each decision point, based solely on past active options and their resulting durations and cumulative rewards. It disregards additional information that a general hierarchical policy in $\Pi_\ospace$ might consider, such as past states, actions, or rewards encountered within each active phase of those options. However, due to the Markovian property of the average-reward problem under consideration, it is sufficient to focus on $\hat{\Pi}_\ospace$. This is because for any policy in $\Pi_\ospace$ and any given initial state, there exists a policy in $\hat{\Pi}_\ospace$ that achieves no less average reward rate. (This conclusion follows from standard arguments; see e.g., \citet[proof of Theorem 5.5.1]{puterman2014markov}.)

With the preceding discussion, we arrive at the following conclusion.

\begin{myprop}[SMDP--MDP connection] \label{prop: smdp-mdp}
Under \cref{assu: option assumption}, any optimal policy $\mu$ for the SMDP $(\sspace, \ospace, \orspace, \lspace, \otrans)$ is also an optimal hierarchical policy for the MDP $(\sspace, \aspace, \rspace, p)$ with options $\ospace$, and the average reward rates of $\mu$ are identical in both problems. Moreover, compared with other hierarchical policies in the MDP, $\mu$ is strongly optimal in the sense defined by the inequality \eqref{eq: mdp strong opt}.
\end{myprop}

Based on this proposition and the SMDP theory reviewed in the previous subsection, we can find an optimal hierarchical policy $\mu$ by identifying a stationary optimal policy for the associated SMDP $(\sspace, \ospace, \orspace, \lspace, \otrans)$. This can be achieved under the condition that the SMDP is weakly communicating, through solutions of its action-value optimality equation \eqref{eq: SMDP action-value optimality equation}. For clarity, we express this optimality equation in the present option context as:  
\begin{equation}
   q(s, o)  = \hat r_{so} - \bar r \cdot \hat l_{so} + \sum_{s' \in \sspace} \otrans_{ss'}^{o} \max_{o' \in \ospace} q (s', o'), \qquad \forall \, s \in \sspace, \ o \in \ospace, \label{eq: inter-option option-value optimality equation}
\end{equation}
where we have used the symbols $\hat r_{so}$, $\hat l_{so}$, and $\otrans_{ss'}^{o}$ to denote the expected reward, the expected holding time, and the state transition probability, respectively, in the SMDP $(\sspace, \ospace, \orspace, \lspace, \otrans)$. We shall refer to this equation as the \emph{option-value optimality equation}. The inter-option Q-learning algorithm, which we will discuss later, is based on the RVI approach for solving this equation. 

\subsubsection{Intra-Option Formulation}
Recall that the options under consideration possess memoryless, stationarity properties, as represented by the parameters $\{\pi(a \mid s, o)\}_{a \in \aspace}$ and $\beta(s,o)$, $s \in \sspace, o \in \ospace$, governing their action selection and termination. By leveraging this internal structure of the options, we obtain an alternative formulation of the optimality equation \eqref{eq: inter-option option-value optimality equation}:
\begin{align}
 &   q(s, o) = \sum_{a \in \aspace} \pi(a \mid s, o) \left (r_{sa} - \bar r + \sum_{s' \in \sspace} p_{ss'}^a  U[q](s', o) \right), \qquad \forall\, s \in \sspace, \, o \in \ospace, \label{eq: intra-option option-value optimality equation} \\
& \text{where} \ \ 
    U[q](s', o)  \= \big( 1 - \beta(s', o) \big) q(s', o) + \beta(s', o) \max_{o' \in \ospace} q(s', o').\label{eq: intra-option option-value optimality equation2}
\end{align}
(Recall that $r_{sa}$ and $p_{ss'}^a$ represent, respectively, the expected one-stage reward and the state transition probability in the MDP, as previously defined in \cref{sec: wc mdps}.) We will delve into the intra-option algorithm, designed to solve this equation, later in this section.

\begin{myremark} \rm
To offer some insights, we mention that the preceding equation can also be derived by considering another formulation of the average-reward problem at hand as an associated MDP $(\widetilde \sspace, \widetilde \aspace, \rspace, \tilde p)$ on the (finite) state space $\widetilde \sspace = \sspace \times \ospace$. Here, each state at a given time represents the pair of the state and the active option at that time in the original MDP. The (finite) action space $\widetilde \aspace$ consists of all mappings $\tilde \mu$ from $\sspace$ into $\ospace$. The transition function $\tilde p$ is determined by the parameters of the options and the original MDP, describing the generation of the one-stage reward and the transition to the next pair of state and active option in the original MDP, if options are activated according to a mapping $\tilde{\mu} \in \widetilde \aspace$. Equation \eqref{eq: intra-option option-value optimality equation} then emerges as the state-value optimality equation \eqref{eq: state-value optimality equation} for this associated MDP.\qed 
\end{myremark}

We now establish the equivalence between the intra-option and inter-option formulations of the optimality equation on state-option values:

\begin{myprop} \label{prop: c2 inter = intra equations}
If $\ospace$ satisfies \cref{assu: option assumption}, then $(\bar r, q)$ solves the option-value optimality equation \eqref{eq: inter-option option-value optimality equation} if and only if it solves equation \eqref{eq: intra-option option-value optimality equation}. 
\end{myprop}

\begin{proof}
Consider the following scenario in the MDP: starting from the current state and active option $(S_0, O_0)$, actions are selected according to $O_0$ until some time step $\tau \geq 1$ later when $O_0$ is deactivated, resulting in a trajectory of states, actions, and rewards $(S_0, A_0, R_1, S_1, A_1, R_2, \ldots, S_\tau)$. Let $\mathbb{E}_{so}$, $(s,o) \in \sspace \times \ospace$, denote the expectation operator with respect to the probability distribution of this process given that $(S_0, O_0) = (s,o)$. Note that, due to the memoryless property of the options, this distribution remains the same regardless of whether option $o$ has just been activated at state $S_0$ or was activated prior to the visit to state $S_0$.

In view of the definitions of the option parameters, equation \eqref{eq: intra-option option-value optimality equation} can be rewritten as:
\begin{equation} \label{eq-option-prf1b}
  q(s, o) = \E_{so} \left[ R_1 - \bar r + \ind\{ \tau = 1\} \max_{o' \in \ospace}  q(S_1, o') +   \ind\{ \tau > 1\} q(S_1, o)  \right],    \quad \forall \, s \in \sspace, \, o \in \ospace. 
\end{equation}
On the other hand, by the definition of the SMDP $(\sspace, \ospace, \orspace, \lspace, \otrans)$, the optimality equation \eqref{eq: inter-option option-value optimality equation} can be expressed equivalently as:
\begin{equation} \label{eq-option-prf1a}
  q(s, o) = \E_{so} \left[ \sum_{k=0}^{\tau - 1} ( R_{k+1}  - \bar r) + \max_{o' \in \ospace} q(S_{\tau}, o' ) \right], \qquad \forall \, s \in \sspace, \, o \in \ospace.
\end{equation}

To establish the proposition, let us first assume that $(\bar r, q)$ solves \eqref{eq-option-prf1a}.
Let us decompose the term inside the expectation in \eqref{eq-option-prf1a} into three parts, separating the case $\tau = 1$ from the case $\tau > 1$:
\begin{equation} \label{eq-option-prf2}
 (R_1 - \bar r) + \ind\{\tau = 1\} \max_{o' \in \ospace} q(S_1, o' )  +   \ind\{\tau > 1\} \left( \sum_{k=1}^{\tau - 1} ( R_{k+1}  - \bar r) + \max_{o' \in \ospace} q(S_{\tau}, o' ) \right).
 \end{equation}
Comparing this expression with the r.h.s.\ of \eqref{eq-option-prf1b}, we see that to prove that $(\bar r, q)$ also solves \eqref{eq-option-prf1b} amounts to showing that for all $(s, o) \in \sspace \times \ospace$, the following equality holds:
\begin{equation} \label{eq-option-prf3}  
  \E_{so} \left[  \ind\{ \tau > 1\} q(S_1, o)  \right] =  \E_{so} \left[ \ind\{\tau > 1\} \left( \sum_{k=1}^{\tau - 1} ( R_{k+1}  - \bar r) + \max_{o' \in \ospace} q(S_{\tau}, o' ) \right) \right].
\end{equation}
Now, for each $(s,o) \in \sspace \times \ospace$, we have
\begin{equation} \label{eq-option-prf4}
\E_{so} \left[ \ind\{\tau > 1\} \left(  \sum_{k=1}^{\tau - 1} ( R_{k+1}  - \bar r) + \max_{o' \in \ospace} q(S_{\tau}, o' ) \right)  \,\Big|\, S_1, O_1, \tau > 1 \right] = \ind\{\tau > 1\} q(S_1, o). 
\end{equation}
Here, the expectation on the left-hand side is taken conditioned on the event $\{\tau > 1\}$ and $(S_1, O_1)$, where $O_1$ represents the active option at time step $1$, equaling $o$ when $\tau > 1$. The equality stems from the memoryless property of the options and the assumption that $(\bar r, q)$ satisfies \eqref{eq-option-prf1a}. The desired result \eqref{eq-option-prf3} then follows straightforwardly from \eqref{eq-option-prf4}, confirming $(\bar r, q)$ as a solution to \eqref{eq-option-prf1b}.

Next, let us assume $(\bar r, q)$ solves \eqref{eq-option-prf1b}. For each $(s,o) \in \sspace \times \ospace$, we can expand the expression for $q(s,o)$ from the r.h.s.\ of \eqref{eq-option-prf1b} by leveraging the memoryless property of the options and iteratively applying \eqref{eq-option-prf1b} to express $q(S_1, o)$, $q(S_2, o)$, and so on. This process leads to the following identity relations: for all $n \geq 1$, with $\tau \wedge n \= \min \{ \tau, n\}$, 
\begin{equation}  \label{eq-option-prf5}
   q(s, o) = \E_{so} \left[ \sum_{k=0}^{\tau \wedge n - 1} ( R_{k+1}  - \bar r) + \ind\{ \tau \leq n\} \max_{o' \in \ospace} q(S_{\tau}, o' )  + \ind\{ \tau > n \} q(S_n, o) \right].
\end{equation}
Denote the term inside the expectation by $Y_n$. Under \cref{assu: option assumption}, as $n \to \infty$, $Y_n$ converges a.s.\ to $\sum_{k=0}^{\tau - 1} ( R_{k+1}  - \bar r) + \max_{o' \in \ospace} q(S_{\tau}, o' )$. Additionally, for all $n$, $|Y_n|$ can be bounded by the integrable random variable $\sum_{k=0}^{\tau - 1} ( | R_{k+1}| + |\bar r|) + 2 \| q \|_\infty$ (with its integrability following from \cref{assu: option assumption}). Hence, by the dominated convergence theorem \citep[Theorem 4.3.5]{Dud02}, $\lim_{n \to \infty} \E_{so} [Y_n]$ exists and equals the r.h.s.\ of \eqref{eq-option-prf1a}. Combined with identity \eqref{eq-option-prf5}, this proves that $(\bar r, q)$ satisfies \eqref{eq-option-prf1a}. 
\end{proof}
\vspace*{-0.7cm}

\subsection{Inter-Option Algorithm}\label{sec: inter-option algorithms}

In this subsection, we focus on the inter-option Q-learning algorithm, which aims to find an optimal hierarchical policy for a given MDP with options $\ospace$, by solving the option-value optimality equation \eqref{eq: inter-option option-value optimality equation} of the associated SMDP. 

We shall assume that the associated SMDP is weakly communicating. Based on the previous discussions in Sections~\ref{sec: smdp} and~\ref{sec-4.2.1}, this assumption implies that in optimizing over the hierarchical policies for the MDP, regardless of the initial state, the optimal reward rate $\hat r_*$ remains constant. Moreover, $\hat r_*$ is also the optimal reward rate in the associated SMDP, coinciding with the $\bar r$-component of every solution of the optimality equation \eqref{eq: inter-option option-value optimality equation}, where these solutions exist but are not necessarily unique up to an additive constant.

Note that for the associated SMDP to be weakly communicating, it is neither necessary nor sufficient for the MDP to be weakly communicating. A sufficient condition is that the MDP is weakly communicating and for every state, each action has a non-zero probability of being chosen by some option, but this condition could be unnecessarily restrictive in practice. On the other hand, if the associated SMDP is communicating, then the MDP must be communicating. 

To solve \eqref{eq: inter-option option-value optimality equation}, consider its equivalent scaled form \eqref{eq: scaled inter-option option-value optimality equation}, obtained by dividing the equation by the expected option duration $\hat{l}_{so}$ for each state-option pair:
\begin{equation} 
 \frac{1}{\hat{l}_{so}} \Big(\hat r_{so} - \hat l_{so} \cdot \bar r  + \sum_{s' \in \sspace} \otrans_{ss'}^{o} \max_{o' \in \ospace} q (s', o') - q(s, o)\Big) = 0, \qquad \forall \, s \in \sspace, \, o \in \ospace, \notag
\end{equation}
or equivalently,
\begin{equation} \label{eq: scaled inter-option option-value optimality equation}
  \frac{\hat r_{so}}{\hat{l}_{so}} - \bar r  + \frac{1}{\hat{l}_{so}} \sum_{s' \in \sspace} \otrans_{ss'}^{o} \max_{o' \in \ospace} q (s', o') + \Big(1 - \frac{1}{\hat{l}_{so}}\Big) \, q(s,o)  - q(s, o) = 0, \quad \forall \, s \in \sspace, \, o \in \ospace.
\end{equation}
As $\hat{l}_{so} \geq 1$, this equation can be related to the average-reward optimality equation for an MDP, effectively transforming the SMDP into an equivalent MDP. \citet{schweitzer1971iterative} first used this idea to derive a convergent RVI algorithm
\footnote{Schweitzer's RVI algorithm for solving SMDPs' action-value optimality equations is similar to but differs from \eqref{eq: S-RVI}: for all $(s,a) \in \sspace \times \aspace$,
$$ Q_{n+1}(s,a) = Q_n(s,a) + \alpha \left( \frac{r_{sa} - l_{sa}  \cdot f(Q_n) + \sum_{s' \in \sspace} p_{ss'}^a \max_{a' \in \aspace} Q_{n}(s', a')  - Q_n(s,a)}{l_{sa}} \right),$$  
where 
$f(Q_n) =  l_{\bar s \bar a}^{-1}\left(r_{\bar{s}\bar{a}} + \sum_{s' \in \sspace} p_{\bar ss'}^{\bar a} \max_{a' \in \aspace} Q_{n}(s', a')  - Q_n(\bar s, \bar a)\right)$ for some fixed state-action pair $(\bar s, \bar a)$,
and the step size $\alpha$ can be chosen within $(0, \min_{s \in \sspace, a \in \aspace} l_{sa})$. This algorithm converges provided that the average reward rate remains constant, particularly in weakly communicating SMDPs \citep{platzman1977improved}.\label{footnote-srvi}}
that solves similarly scaled state-value optimality equations for SMDPs. The inter-option Q-learning algorithm, introduced by \citet{wan2021average}, was inspired by Schweitzer's RVI algorithm and can be viewed as its asynchronous stochastic counterpart. Here is how the inter-option algorithm operates.

The algorithm maintains estimates of both state-option values and expected option durations, updating them iteratively using ``option-level'' transition data from the MDP. At each iteration $n$, these estimates are represented by $|\sspace \times \ospace|$-dimensional vectors $Q_n$ and $L_n > \zerovec$, respectively. The initial values $Q_0$ and $L_0 > \zerovec$ can be arbitrarily chosen. Similar to RVI Q-learning, the components $Q_n(s,o)$ and $L_n(s,o)$ are updated for chosen state-option pairs $(s,o)$ from a randomly selected nonempty subset $Y_n \subset \sspace \times \ospace$, while the remaining components remain unchanged.

Updates are based on transition data generated by executing selected options in the MDP. For each $(s,o) \in Y_n$, the algorithm executes option $o$ from state $s$ in the MDP until termination at some state $\hat S_{\tau}$ after $\tau \geq 1$ time steps. Let $S_{n+1}^{so}= \hat S_{\tau}$, $L_{n+1}^{so} = \tau$, and $R_{n+1}^{so}$ be the cumulative reward incurred during this period. Then $(S_{n+1}^{so}, R_{n+1}^{so}, L_{n+1}^{so})$ follows the transition distribution $\otrans(\cdot \mid s, o)$ of the associated SMDP by definition. Using these generated data for $(s,o) \in Y_n$, the algorithm updates the components of $Q_n$ and $L_n$ according to the following rules:
\begin{align}
& \text{for $(s, o) \not\in Y_n$:}  \ \ \   Q_{n+1}(s, o) \= Q_{n}(s, o), \ \ \  L_{n+1}(s, o) \= L_{n}(s, o); \notag \\
& \text{for $(s, o) \in Y_n$:}  \nonumber \\
 & \    Q_{n+1}(s, o)  \= Q_{n}(s, o) +  \alpha_{\nu_n(s, o)} \frac{R_{n+1}^{so} - L_n(s, o) f(Q_n) + \max_{o' \in \ospace} Q_n (S_{n+1}^{so}, o') - Q_n(s, o)}{L_n(s, o)},  \label{eq: c2 Inter-option algorithm}  \\
 & \   L_{n+1}(s, o) \= L_{n}(s, o) + \beta_{\nu_n(s, o)} (L_{n+1}^{so} - L_{n}(s, o)).   
    \label{eq: c2 Inter-option Differential TD-learning L}
\end{align}
Here, $\nu_n(s,o)$ denotes the cumulative count of how many times the state-option pair $(s,o)$ has been chosen up to iteration $n$, with $\nu_n(s,o) = \sum_{k=0}^{n} \ind \{ (s,o) \in Y_k \}$. The step-size sequence $\{\alpha_n\}$, the function $f: \sspace \times \ospace \to \R$, and the asynchronous update schedules must satisfy the same assumptions as in RVI Q-learning. The update rule \eqref{eq: c2 Inter-option Differential TD-learning L} applies stochastic gradient descent to estimate the expected option duration $\hat l_{so}$, using a separate standard step-size sequence $\beta_n \in [0,1], n \geq 0$. We summarize these algorithmic conditions below.

{\samepage
\begin{myassumption}[algorithmic requirements for the inter-option algorithm]\label{assu: inter-option extended rvi q learning}\hfill
 \begin{itemize}[leftmargin=0.7cm,labelwidth=!]
\item[\rm (i)] The function $f$ satisfies \cref{assu: f}, the step-size sequence $\{\alpha_n\}$ satisfies \cref{assu: stepsize}, and the asynchronous update schedules are such that $\{\alpha_n\}$ and $\{\nu_n\}$ jointly satisfy \cref{assu: update}, with the space $\ispace = \sspace \times \ospace$ in these assumptions.
\item[\rm (ii)] The step-size sequence $\{\beta_n\}$ is such that $\beta_n \in [0,1]$ for $n \geq 0$, $\sum_{n = 0}^\infty \beta_n = \infty$, and $\sum_{n = 0}^\infty \beta_n^2 < \infty$.
\end{itemize}    
\end{myassumption}
}

As can be seen, the main distinction between the update rule \eqref{eq: c2 Inter-option algorithm} of the inter-option algorithm and RVI Q-learning \eqref{eq: Extended RVI Q-learning} lies in the scaling of the updates with estimated option durations. This scaling approach will be crucial to ensure the algorithm's convergence in our analysis, as it was for Schweitzer’s classical RVI algorithm. In addition, computationally, scaling helps stabilize the updates across state-option pairs by mitigating variation due to differing option durations.

Similar to RVI Q-learning, the general update rule \eqref{eq: c2 Inter-option algorithm} may assume different forms with specific choices of the function $f$. As an example, here is the inter-option extension of the Differential Q-learning algorithm discussed previously in \cref{ex: diff Q}:

\begin{example} [Inter-Option Differential Q-learning \citep{wan2021average}] \label{ex: inter-option diff Q} \rm \hfill \\
In addition to $Q_n$ and $L_n$, this algorithm also maintains a reward rate estimate $\bar R_n$, similar to Differential Q-learning. At iteration $n$, for each $(s,o) \in Y_n$, it computes the TD error:
\begin{align*} 
    \delta_n(s, o) & \= R_{n+1}^{so} - L_n(s, o) \bar R_n + \max_{o' \in \ospace} Q_n (S_{n+1}^{so}, o') - Q_n(s, o).
\end{align*}
The TD error terms are then scaled by the estimated option durations when updating $Q_n$ and $\bar R_n$:
\begin{align*}
    Q_{n+1}(s, o) & \= Q_{n}(s, o) + \alpha_{\nu_n(s, o)} (\delta_{n}(s, o) / L_n(s, o)) \ind\{(s, o ) \in Y_n\}, \quad \forall\, s \in \sspace, \, o \in \ospace, \\
    \bar R_{n+1} & \= \bar R_{n} + \eta \sum_{(s, o ) \in Y_n} \alpha_{\nu_n(s, o)} \delta_{n}(s, o)/L_n(s, o),
\end{align*}
where $\eta > 0$ is an algorithmic parameter, while the update rule for $L_n$ remains the same as \eqref{eq: c2 Inter-option Differential TD-learning L}. Following the same reasoning for Differential Q-learning in \cref{ex: diff Q}, this inter-option algorithm can be seen as an instance of the general  inter-option algorithm, with the function $f$ defined as 
$f(q) = \eta \sum_{s \in \sspace, o \in \ospace} q(s,o) - \eta \sum_{s \in \sspace, o \in \ospace} Q_0(s,o) + \bar R_0$. 

The convergence of this algorithm was analyzed by \citet{wan2021average} under a unichain condition on the associated SMDP for ensuring that the optimality equation \eqref{eq: inter-option option-value optimality equation} has a unique solution of $q$ (up to an additive constant).
However, their proof is inadequate; see \cref{remark: sa result compare}(a) for more details.\qed
\end{example}

As our main results regarding the inter-option algorithm, we characterize its solution set and provide its convergence properties in the two ensuring theorems. These results mirror Theorems~\ref{thm: MDP characterize Q} and~\ref{thm: Extended RVI Q-learning} for RVI Q-learning. 

Let $\hat \calQ$ denote the set of solutions $q$ to the option-value optimality equation \eqref{eq: inter-option option-value optimality equation}. Consider the subset of $\hat \calQ$ constrained by $f(q) = \ooptimalr$:
\begin{equation}\label{eq: smdp sol set}
\ooptimalitydiffsolutionq \= \{q \in \hat \calQ :  f(q) = \ooptimalr\},
\end{equation}
which is the desired solution set for the inter-option algorithm.

\begin{mytheorem}\label{thm: SMDP characterize Q}
Given an MDP and a set of options satisfying \cref{assu: option assumption}, if the associated SMDP is weakly communicating and $f$ satisfies \cref{assu: f}, then the set $\ooptimalitydiffsolutionq$ is nonempty, compact, connected, and possibly nonconvex.
\end{mytheorem}

The preceding theorem characterizes $\ooptimalitydiffsolutionq$; its proof will be given in \cref{sec: solution set}. Furthermore, in \cref{sec-7.2}, we will apply the theory of \citepalias{ScF78} to show that $\ooptimalitydiffsolutionq$ has precisely one less degree of freedom than the set $\hat \calQ$.

The next theorem establishes the convergence of the inter-option algorithm. For a given vector $q$ of state-option values, let us call a hierarchical policy $\mu$ \emph{greedy w.r.t.\ $q$}, if $\mu$ corresponds to a deterministic stationary policy $\mu: \sspace \to \ospace$ in the associated SMDP and for each state $s \in \sspace$, $\mu(s) \in \argmax_{o \in \ospace} q(s, o)$.

\begin{mytheorem}[convergence theorem] \label{thm: inter-option Differential Q-learning}
For a given MDP with a set of options satisfying \cref{assu: option assumption}, 
consider its associated SMDP, and let $\{Q_n\}$ be generated by the algorithm (\ref{eq: c2 Inter-option algorithm}-\ref{eq: c2 Inter-option Differential TD-learning L}) under Assumption~\ref{assu: inter-option extended rvi q learning}. 
If the associated SMDP is weakly communicating, then the following hold almost surely:
\begin{itemize}[leftmargin=0.7cm,labelwidth=!]
\item[\rm (i)] As $n \to \infty$,  $Q_n$ converges to a sample path-dependent compact connected subset of $\ooptimalitydiffsolutionq$, and $f(Q_n)$ converges to the optimal reward rate $\ooptimalr$. 
\item[\rm (ii)] For all sufficiently large $n$, the greedy hierarchical policies w.r.t. $Q_n$ are all optimal.
\end{itemize}
\end{mytheorem}

We will prove part (i) of this theorem in \cref{sec: convergence proofs}, employing ODE-based methods. Part (ii) follows from part (i) and the compactness of the set $\ooptimalitydiffsolutionq$ (\cref{thm: SMDP characterize Q}), using the same arguments as in the proof for \cref{thm: Extended RVI Q-learning}(ii). In particular, with those same proof arguments, we establish the optimality of greedy policies for the associated SMDP when $n$ is sufficiently large. The optimality of these policies as hierarchical policies in the MDP then follows from \cref{prop: smdp-mdp}.

\subsection{Intra-Option Algorithm} \label{sec: intra-option algorithms}
The intra-option Q-learning algorithm aims to solve the hierarchical decision problem with options by finding a solution to the alternative optimality equation \eqref{eq: intra-option option-value optimality equation} for option values. Unlike the inter-option case, this algorithm benefits from knowing the option parameters $\pi(a \mid s, o)$ and $\beta(s,o)$ and leverages options' internal memoryless and stationarity properties. These enable the algorithm to utilize single-step transition data to update option values, so that there is no need to execute options until completion or estimate their durations during each iteration. This characteristic significantly enhances the intra-option algorithm's data efficiency compared to its inter-option counterpart.

In particular, the intra-option algorithm iteratively updates option-value estimates by using ``action-level'' single-step transition data. 
To generate these data, the algorithm applies some (stationary) policies $b_0, b_1, \ldots$ in the MDP, where the choice of each policy may depend on the algorithmic history.  
Specifically, with some given small $\epsilon \in (0,1)$ as the algorithmic parameter, at iteration $n \geq 0$:
\begin{itemize}[leftmargin=0.5cm,labelwidth=!]
\item[1.] The algorithm selects a nonempty subset $X_n$ of states and a policy $b_n$. The choices are made such that for all $s \in X_n$, $\min \{ b_n(a \mid s): \, b_n(a \mid s) > 0, \, a \in \aspace \} \geq \epsilon$ and the subset $\ospace_n(s)$ of options is nonempty, where
$$ \ospace_n(s) \= \{ o \in \ospace: \, \pi( \cdot \mid s, o) \ \text{is absolutely continuous w.r.t.} \ b_n(\cdot \mid s) \}.$$
\item[2.] For each $s \in X_n$, the algorithm applies the policy $b_n$ to sample an action $A_n^s \sim b_n(\cdot \mid s)$ and observes the resulting state $S_{n+1}^s$ and reward $R^s_{n+1}$ from the MDP (i.e., $(S_{n+1}^s, R^s_{n+1}) \sim \trans(\cdot, \cdot \mid s, A_n^s)$). 
\end{itemize}
Let $Y_n \= \{(s,o) : s \in X_n, o \in \ospace_n(s)\}$. 
Using the generated data, the algorithm then updates the option-value estimates $Q_n$ according to the following rules:
\begin{align}
& \text{for } (s, o) \not \in Y_n: \ \ \  Q_{n+1}(s, o) \= Q_{n}(s, o); 
\nonumber\\
& \text{for } (s, o) \in Y_n: \ \ \ \text{with} \  \rho_n(s, o)  \= \pi(A_{n}^{s} \mid s, o) / b_n(A_{n}^{s} \mid s), 
\nonumber \\
  &  \ \  Q_{n+1}(s, o) \= Q_{n}(s, o) + \alpha_{\nu_n(s, o)} \rho_n(s, o) 
    \left (R_{n+1}^{s} - f (Q_n)
    + U[Q_n](S_{n+1}^{s}, o) - Q_n(s, o)  \right),  
    \label{eq: transformed intra-option update}
\end{align}
where $U[Q_n]$ is as defined in \eqref{eq: intra-option option-value optimality equation2}:
\begin{align*}
    U[Q_n](S_{n+1}^{s}, o) = (1 - \beta(S_{n+1}^{s}, o)) Q_n(S_{n+1}^{s}, o) + \beta(S_{n+1}^{s}, o) \max_{o' \in \ospace} Q_n(S_{n+1}^{s}, o').
\end{align*}
The initial values $Q_0$ can be arbitrarily chosen.

In the above, $\rho_n(s,o)$ is an \textit{importance sampling ratio} term that compensates for the difference between the behavior policy $b_n$ and the option $o$'s policy $\pi(\cdot \mid \cdot, o)$. The choices of $\{b_n\}$ ensure that these ratios are all bounded by $1/\epsilon$; this boundedness property will be useful in our subsequent convergence analysis.
The cumulative counts $\nu_n(s,o) \= \sum_{k=0}^{n} \ind \{ (s,o) \in Y_k \}$. The function $f$, the step sizes $\alpha_n$, and the asynchronous update schedules are required to satisfy the same assumptions as in the inter-option Q-learning algorithm.

\begin{myremark} \rm 
An intra-option extension of the Differential Q-learning algorithm (\cref{ex: diff Q}) can be derived similarly to the inter-option case presented in \cref{ex: inter-option diff Q}, with the function $f$ defined as in the latter example. The previous convergence analysis of this algorithm by \citet{wan2021average} faces the same issue as noted for the inter-option Differential Q-learning algorithm; see \cref{remark: sa result compare}(a) for details.\qed
\end{myremark}

Due to the equivalence between the optimality equations \eqref{eq: inter-option option-value optimality equation} and \eqref{eq: intra-option option-value optimality equation} (\cref{prop: c2 inter = intra equations}), the intra-option algorithm shares the same solution set $\ooptimalitydiffsolutionq$ as the inter-option algorithm. The next theorem shows that the algorithm also enjoys the same convergence guarantees.

\begin{mytheorem}[convergence theorem] \label{thm: intra-option Differential Q-learning}
Given an MDP and a set of options satisfying \cref{assu: option assumption}, consider $\{Q_n\}$ generated by the intra-option algorithm \eqref{eq: transformed intra-option update}. If the corresponding SMDP is weakly communicating and \cref{assu: inter-option extended rvi q learning}(i) holds, then the conclusions of \cref{thm: inter-option Differential Q-learning}(i, ii) hold almost surely.
\end{mytheorem} 

The proof of part (i) is provided in \cref{sec: convergence proofs}. Part (ii) then follows from part (i) by the same proof for \cref{thm: inter-option Differential Q-learning}(ii).

\section{Properties of Solution Sets $\ExtRVIQsolutionq$ and $\ooptimalitydiffsolutionq$ (Proofs of Theorems\ \ref{thm: MDP characterize Q}, \ref{thm: SMDP characterize Q})}\label{sec: solution set}

Given a weakly communicating SMDP, recall that $\calQ$ denotes the set of solutions of $q$ to the optimality equation \eqref{eq: SMDP action-value optimality equation}. Since an MDP is a special case of SMDP, the solution sets $\ExtRVIQsolutionq$ and $\ooptimalitydiffsolutionq$ addressed in Theorems\ \ref{thm: MDP characterize Q} and \ref{thm: SMDP characterize Q} for RVI Q-learning and its options extensions are special cases of the following solution set for a weakly communicating SMDP:
\begin{align}\label{eq: Qs}
\smdpoptimalitygeneralsolutionq \= \{q \in \R^{\cardS \times \cardA}:\, q \in \smdpoptimalitysolutionq, \ f(q) = \optimalr \},
\end{align}
where $f: \sspace \times \aspace \to \bbR$ satisfies \cref{assu: f}, and $\optimalr$ is the optimal reward rate of the SMDP. Let us prove the following result for $\smdpoptimalitygeneralsolutionq$, which entails Theorems\ \ref{thm: MDP characterize Q} and \ref{thm: SMDP characterize Q}.

\begin{mytheorem}\label{thm: smdp properties}
    In a weakly communicating SMDP, with f satisfying Assumption\ \ref{assu: f}, the set $\smdpoptimalitygeneralsolutionq$ is (i) nonempty, compact, and connected, and (ii) possibly nonconvex.
\end{mytheorem}

Based on Schweitzer and Federgruen's results \citepalias{ScF78}, we know that the solution set $\calQ$ is nonempty, closed and unbounded, always connected, but possibly nonconvex. The preceding theorem shows that adding the constraint $f(q) = \optimalr$ selects a connected and compact subset of solutions from $\calQ$. (Later, in \cref{sec: solution set dim}, we will further utilize the results of \citepalias{ScF78} to show that this constraint reduces the number of degrees of freedom in the solutions by exactly $1$; cf.\ \cref{thm-dim-Qs}.) The compactness of $\smdpoptimalitygeneralsolutionq$ has a crucial role in ensuring the stability of the algorithms, as will be seen in our subsequent convergence proofs.

We now proceed to prove \cref{thm: smdp properties}. Its part (ii) will be demonstrated directly with an example of a nonconvex set $\smdpoptimalitygeneralsolutionq$ (see \cref{ex: nonconvex Q}). Our immediate focus will be on proving its part (i). 
For notational simplicity and a cleaner presentation, we will work with the state-value optimality equation \eqref{eq: SMDP state-value optimality equation} instead:
\begin{equation} \label{eq: SMDP state-value optimality equation2}
  v(s) = \max_{a \in \aspace} \left \{ r_{sa} - \bar r \cdot l_{sa} +  \sum_{s' \in \sspace} p_{ss'}^a v(s') \right\}, \qquad \forall \, s \in \sspace.
\end{equation}
It is well-known that the action-value optimality equation \eqref{eq: SMDP action-value optimality equation} for any weakly communicating SMDP can be viewed as the state-value optimality equation \eqref{eq: SMDP state-value optimality equation2} for an equivalent, weakly communicating SMDP defined on an enlarged (finite) state-action space, with the original state-action pairs treated as states. (For a precise definition, see the discussion on $\text{SMDP}_q$ near the end of \cref{sec-7.2}.) Thus, to prove \cref{thm: smdp properties}(i), it is sufficient (actually equivalent) to establish its conclusions for the following subset of solutions (in $v$) to \eqref{eq: SMDP state-value optimality equation2}:
\begin{align}\label{eq: Vs}
\smdpoptimalitygeneralsolutionv \= \{v \in \R^{\cardS}:\, v \in \smdpoptimalitysolutionv, \ f(v) = \optimalr \}.
\end{align}
Here, $\smdpoptimalitysolutionv$ denotes the set of all solutions of $v$ to \eqref{eq: SMDP state-value optimality equation2}, and $f : \sspace \to \R$ satisfies \cref{assu: f} with the space $\ispace$ being $\sspace$ instead.  

The following lemma is closely related to the compactness of $\smdpoptimalitygeneralsolutionv$ and the algorithmic stability mentioned earlier. It shows an important property of weakly communicating SMDPs: while the solutions in $\calV$ may not be unique up to an additive constant, they must be so if all rewards are zero. The solutions in this special case delineate the directions in which the solutions in the original $\calV$ can ``escape to $\infty$,'' making it relevant to our original problem. We will use this lemma for the compactness part of \cref{thm: smdp properties} and later, also for the stability part required in the convergence analysis in \cref{sec: convergence proofs} (cf.\ Remark~\ref{remark: stability ascpect}).

Although this lemma can be inferred from the general results from \citepalias{ScF78} on general multichain SMDPs (cf.\ \cref{rmk-app-A1}(b) in \cref{sec: solution set dim}), we provide here a concise and direct alternative proof, by leveraging the weakly-communicating structure.

\begin{mylemma} \label{lemma: 0 reward MDP has 1 d solution}
In a weakly communicating SMDP with zero rewards, $\smdpoptimalitysolutionv = \{c \onevec \mid c \in \R \}$.
\end{mylemma}

\begin{proof}
Recall that in a weakly communicating SMDP, there is a unique, closed communicating class of states, denoted by $S^o$, and the remaining states in $\sspace \setminus S^o$ are transient under all policies.
With zero rewards, the optimal reward rate is $0$ and the optimality equation \eqref{eq: SMDP state-value optimality equation} thus reduces to
\begin{align}\label{eq: zero reward v opt eqn}
     v(s) = \max_{a \in \aspace} \left\{ \sum_{s' \in \sspace} p_{ss'}^a v(s') \right\}, \quad s \in \sspace.
\end{align}
Any constant function $v$ satisfies \eqref{eq: zero reward v opt eqn}. 

Conversely, let $v$ be a solution of \eqref{eq: zero reward v opt eqn}. Consider these two nonempty subsets of states:
$$ S_\text{min} \= \argmin_{s \in \sspace} v(s), \qquad \ \  S_\text{max} \=  \argmax_{s \in \sspace} v(s).$$
By \eqref{eq: zero reward v opt eqn}, there is a zero probability of transitioning from a state $s \in S_\text{min}$ to a state $s' \not \in S_\text{min}$, regardless of the action chosen. Therefore, $S_\text{min}$ is a closed class of states by definition (cf.\ \cref{sec: wc mdps}), and this implies $S^o \subset S_\text{min}$ since the SMDP is weakly communicating. 

On the other hand, by \eqref{eq: zero reward v opt eqn}, there exists a nonempty subset $S'_\text{max}$ of $S_\text{max}$ such that $S'_\text{max}$ is a recurrent class under some deterministic policy. Since the SMDP is weakly communicating, this implies $S'_\text{max} \subset S^o$. Thus, $ S'_\text{max} \subset S_\text{min}$ and consequently, $\min_{s \in \sspace} v(s) = \max_{s \in \sspace} v(s)$; i.e., $v$ is a constant function.   
\end{proof}
\vspace*{-0.5cm}

We now prove \cref{thm: smdp properties}(i).

\smallskip
\begin{proofof}{\cref{thm: smdp properties}(i)}
As discussed earlier, it suffices to establish the conclusions of \cref{thm: smdp properties}(i) for the set $\smdpoptimalitygeneralsolutionv$ instead.
First, let us prove that $\smdpoptimalitygeneralsolutionv$ is nonempty, closed and connected. This proof uses the definition of this set, the properties of the set $\calV$ given in \citepalias{ScF78}, and the conditions on the function $f$ given in \cref{assu: f}(i, ii).

\noindent (i) Closedness: The set $\smdpoptimalitygeneralsolutionv$ is clearly closed, as all the functions in its defining equations are real-valued and continuous on $\bbR^{|\sspace|}$. 

\noindent (ii) Nonemptiness: By \citepalias[Theorem~3.1(b)]{ScF78}, $\calV \not= \varnothing$. Let $v_* \in \calV$. Then for all $c \in \bbR$, we have $v_* + c \onevec \in \calV$  [cf.\ \eqref{eq: SMDP state-value optimality equation2}], particularly for $c_* =  (r_* - f(v_*))/u$ where $u > 0$ is the constant from \cref{assu: f}(ii). By \cref{assu: f}(ii),
we have $f(v_* + c_* \onevec) = f(v_*) + c_* u = r_*$, implying $v_* + c_* \onevec \in \smdpoptimalitygeneralsolutionv$. Therefore, $\smdpoptimalitygeneralsolutionv \not= \varnothing$. 

\noindent (iii) Connectedness: By \citepalias[Theorem~4.2(b)]{ScF78}, the set $\smdpoptimalitysolutionv$ is connected. To extend this connectedness to $\smdpoptimalitygeneralsolutionv$, consider the continuous function $z : \smdpoptimalitysolutionv \to \smdpoptimalitygeneralsolutionv$ defined as $z(v) \= v + \tfrac{\optimalr -  f(v)}{u} \onevec$. Here the continuity of $z$ follows from that of $f$ (\cref{assu: f}(i)) and that $z(v) \in \smdpoptimalitygeneralsolutionv$ follows from \cref{assu: f}(ii), similarly to the nonemptiness proof above. This function $z$ maps the connected set $\smdpoptimalitysolutionv$ onto $\smdpoptimalitygeneralsolutionv$, since $z(v) = v$ for any $v \in \smdpoptimalitygeneralsolutionv \subset \smdpoptimalitysolutionv$. As the image of a connected set under a continuous function is connected, it follows that $\smdpoptimalitygeneralsolutionv$ is connected.

To prove the compactness of $\smdpoptimalitygeneralsolutionv$, we need to show that this closed set is also bounded. We employ proof by contradiction. Suppose $\smdpoptimalitygeneralsolutionv$ is unbounded. Then there exists a sequence $\{x_n\}$ in $\smdpoptimalitygeneralsolutionv$ such that, as $n \to \infty$,
\begin{equation} \label{eq-Thm5.1-prf1}
 \norm{x_n} \to \infty, \qquad y_n \= x_n / \norm{x_n} \to y_\infty \ \ \text{for some $y_\infty \in \R^{\cardS}$ with $\norm{y_\infty} = 1$.} 
\end{equation}
(Since the unit ball in $\R^{\cardS}$ is compact, we can always find such an unbounded sequence $\{x_n\}$ from any unbounded sequence in $\smdpoptimalitygeneralsolutionv$ by choosing a proper subsequence.)

Since $x_n \in \smdpoptimalitygeneralsolutionv$, we have
\begin{align*}
    &x_n(s) = \max_{a \in \aspace} \left\{r_{sa} - \optimalr \cdot l_{sa} + \sum_{s' \in \sspace} p_{ss'}^a x_n(s') \right\}, \quad \forall\, s \in \sspace, \\
    & f(x_n) = \optimalr.
\end{align*}
Hence, $y_n = x_n/\norm{x_n}$ satisfies:
\begin{align*}
    & y_n(s) = \max_{a \in \aspace} \left \{ \frac{r_{sa} - \optimalr \cdot l_{sa}}{\norm{x_n}} + \sum_{s' \in \sspace} p_{ss'}^a y_n(s') \right \},  \quad \forall\, s \in \sspace, \\
    & f(y_n) = f(\zerovec) + \frac{\optimalr - f(\zerovec)}{\norm{x_n}},
\end{align*}
where we applied \cref{assu: f}(iii) to $f(c x_n)$ with $c = 1/\norm{x_n}$ to derive the second equation. Taking $n \to \infty$ in the above two equations and using \eqref{eq-Thm5.1-prf1} and the continuity of $f$ (\cref{assu: f}(i)), we obtain the relations satisfied by the point $y_\infty$:
\begin{align}
    & y_\infty(s) = \max_{a \in \aspace} \left\{\sum_{s' \in \sspace} p_{ss'}^a y_\infty(s')\right\}, \quad \forall\, s \in \sspace, \label{eq: y opt eqn}\\
    & f(y_\infty) = f(\zerovec). \label{eq: y det eqn}
\end{align}
Now \eqref{eq: y opt eqn} is the same as \eqref{eq: zero reward v opt eqn}. The solutions of this equation are constant functions, as shown in the proof of \cref{lemma: 0 reward MDP has 1 d solution}. Thus $y_\infty = c \onevec$ for some $c \in \R$. Then, by \eqref{eq: y det eqn} and \cref{assu: f}(ii), we have $f(y_\infty) = f(\zerovec) + c u = f(\zerovec)$, implying $c = 0$ and hence $y_\infty = \zerovec$. However, this is impossible since $\norm{y_\infty} = 1$. This contradiction shows that the set $\smdpoptimalitygeneralsolutionv$ must be bounded.
\end{proofof}
\vspace*{-0.5cm}

We close this section by demonstrating with an example that the solution set $\smdpoptimalitygeneralsolutionq$ can be nonconvex, thereby establishing \cref{thm: smdp properties}(ii). This example involves an MDP, a special case of SMDP.

\begin{example}[A nonconvex $\smdpoptimalitygeneralsolutionq$] \rm \label{ex: nonconvex Q}
Consider a weakly communicating MDP with three states and two actions, as illustrated in \cref{fig: c2 communicating_example} (left subfigure). The optimal reward rate is $0$.

\begin{figure*}[h]
\centering
    \begin{subfigure}{\textwidth}
    \includegraphics[width=\textwidth]{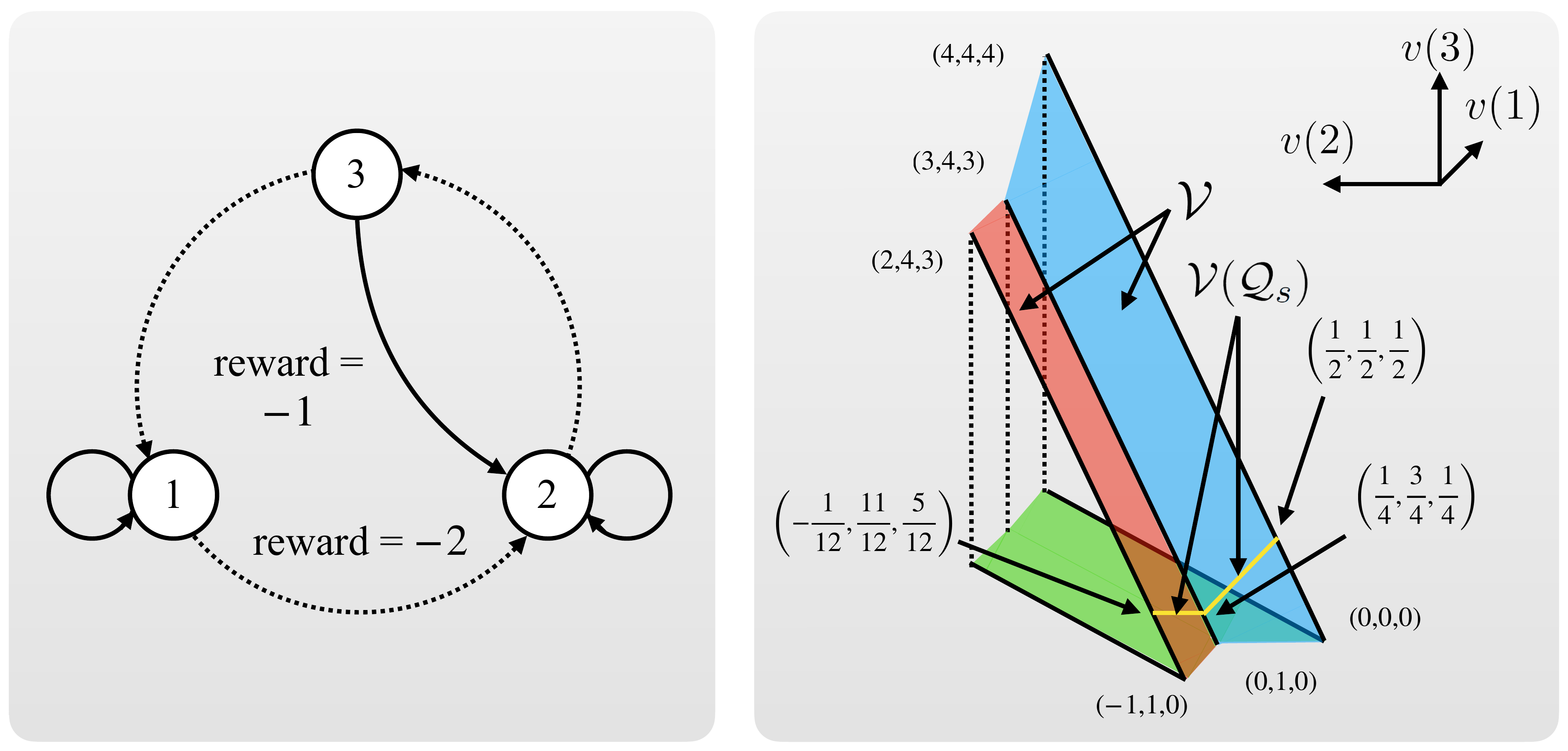}
    \end{subfigure}
    \caption{An illustrative MDP example. \emph{Left}: The example MDP has three states $\{\emph{1, 2, 3}\}$ and two actions $\{\texttt{solid},  \texttt{dashed}\}$ with deterministic effects. The directed solid and dashed curves between states depict deterministic state transitions corresponding to actions \texttt{solid} and \texttt{dashed}, respectively. Taking action \texttt{solid} (resp. \texttt{dashed}) at state \emph{3} (resp. state \emph{1}) results in a reward of $-1$ (resp. $-2$), while all other rewards are $0$. \emph{Right}: Visualization of the solution set $\smdpoptimalitysolutionv$ and its  subset $\smdpoptimalitysolutionv(\smdpoptimalitygeneralsolutionq)$, comprising the state value functions corresponding to the solutions in $\smdpoptimalitygeneralsolutionq$. The red and blue regions together represent $\smdpoptimalitysolutionv$, while the two yellow line segments correspond to $\smdpoptimalitysolutionv(\smdpoptimalitygeneralsolutionq)$. Both sets are nonconvex.}
    \label{fig: c2 communicating_example}
\end{figure*}

Let $f(q) = \sum_{i = 1}^3 \sum_{a \in \aspace} q(i,a)$. Such a choice of $f$ satisfies \cref{assu: f} on $f$. Let $s$ and $d$ stand for actions \texttt{solid} and \texttt{dashed}, respectively. 
Consider two points $q_1, q_2 \in \smdpoptimalitygeneralsolutionq$ and the midpoint $\bar q \= 0.5 q_1 + 0.5 q_2$, with their components given by:
$$
\begin{array}{c|cc}
 q_1: & s & d \\
\hline
\emph{1} & 1/2 & -3/2 \\
\emph{2} & 1/2 & 1/2 \\
\emph{3} & -1/2 & 1/2
\end{array}
\qquad \quad
\begin{array}{c|cc}
 q_2: & s & d \\
\hline
\emph{1} & -2/3 & -2/3 \\
\emph{2} & 4/3 & 1/3 \\
\emph{3} & 1/3 & -2/3
\end{array}
\qquad \quad
\begin{array}{c|cc}
 \bar{q}: & s & d \\
\hline
\emph{1} & -1/12 & -13/12 \\
\emph{2} & 11/12 & 5/12 \\
\emph{3} & -1/12 & -1/12 
\end{array}
$$
That $q_1, q_2 \in \smdpoptimalitygeneralsolutionq$ can be directly verified, as they satisfy both $f(q) = \optimalr = 0$ and the action-value optimality equation \eqref{eq: action-value optimality equation} for this MDP. However, the midpoint $\bar q$ violates the latter equation, as $\frac{5}{12} = \bar q(2, d) \neq \max\{\bar q(3, s),\, \bar q(3, d)\} = -\frac{1}{12}$. Therefore, $\bar q \not \in  \smdpoptimalitygeneralsolutionq$, and the set $\smdpoptimalitygeneralsolutionq$ is not convex.

While it is hard to visualize the set $\smdpoptimalitygeneralsolutionq$ in $\R^6$, let us derive and plot its corresponding set of state values $v(\cdot) = \max_{a \in \aspace} q(\cdot,a)$ in $\R^3$ to provide a more intuitive picture. 
First, in this MDP, the state-value optimality equation \eqref{eq: SMDP state-value optimality equation2} becomes
\begin{align*}
    v(1) = \max\{v(1), -2 + v(2)\},\quad v(2) = \max\{v(2), v(3)\},\quad v(3) = \max\{v(1), -1 + v(2)\}.
\end{align*}
Its solution set $\smdpoptimalitysolutionv$ is plotted in \cref{fig: c2 communicating_example} (right subfigure) as the two connected strips in red and blue. Consider the subset $\smdpoptimalitysolutionv(\smdpoptimalitygeneralsolutionq)$ of state value functions corresponding to the state-action value functions in $\smdpoptimalitygeneralsolutionq$; that is, $$\smdpoptimalitysolutionv(\smdpoptimalitygeneralsolutionq) \= \{ v \in \R^3 :\, \exists\, q \in \smdpoptimalitygeneralsolutionq\ \text{with}\ v(i) =  \max_{a \in \aspace} q(i, a) \  \text{for} \ i = 1,2,3.\}$$
Since $f(q) = 0$ for $q \in \smdpoptimalitygeneralsolutionq$, we can express the set $\smdpoptimalitysolutionv(\smdpoptimalitygeneralsolutionq)$ using the relationship between $q$ and $v(\cdot) = \max_{a \in \aspace} q(\cdot, a)$ provided by the action-value optimality equation \eqref{eq: action-value optimality equation} for this MDP. It is given by
$$\smdpoptimalitysolutionv(\smdpoptimalitygeneralsolutionq) = \{v \in \smdpoptimalitysolutionv :\, 2 v(1) + 3 v(2) + v(3) = 3\},$$
and depicted in \cref{fig: c2 communicating_example} (right subfigure) as the two connected yellow line segments within the set $\smdpoptimalitysolutionv$. Observe that both $\smdpoptimalitysolutionv$ and $\smdpoptimalitysolutionv(\smdpoptimalitygeneralsolutionq)$ are nonconvex.
\qed
\end{example}

\section{Convergence Proofs (Theorems\ \ref{thm: Extended RVI Q-learning}, \ref{thm: inter-option Differential Q-learning}, \ref{thm: intra-option Differential Q-learning})} \label{sec: convergence proofs}

In this section, we prove the convergence theorems for the three studied average-reward Q-learning algorithms: RVI Q-learning, and its inter- and intra-option extensions (Theorems~\ref{thm: Extended RVI Q-learning}, \ref{thm: inter-option Differential Q-learning}, \ref{thm: intra-option Differential Q-learning}). We approach this task in a unified manner by focusing on establishing the convergence of an abstract, general stochastic RVI algorithm. This framework encompasses the three specific algorithms as special cases and may also have broader applications beyond MDPs/SMDPs. The convergence analysis will be presented in \cref{sec: c1 general RVI Q}, leading to \cref{thm: General RVI Q}, which will then be specialized to specific contexts to derive Theorems~\ref{thm: Extended RVI Q-learning}, \ref{thm: inter-option Differential Q-learning}, and \ref{thm: intra-option Differential Q-learning} in \cref{sec: proofs for specific algs}.

Our proof strategy is similar to that of \citet{abounadi2001learning} for RVI Q-learning and can be outlined as follows: The RVI algorithms we consider are asynchronous SA (stochastic approximation) algorithms, and we employ ODE-based proof methods to analyze their behavior. Specifically, we use a stability criterion and proof method developed by \citet{BoM00} to analyze the algorithms' stability, i.e., the boundedness of their updates. Once stability is established, applying SA theory allows us to relate the asymptotic behavior of the algorithms to that of their associated ODEs' solutions as time approaches infinity. Finally, by analyzing the solution properties of these associated ODEs, we derive concrete characterizations of the algorithms' convergence properties. 

Our analysis builds upon prior work \citep{BoM00} for stability analysis and \citep{abounadi2001learning} for analyzing the ODEs associated with RVI Q-learning. However, we extend these prior analyses in two important ways to address the learning algorithms in weakly communicating MDPs/SMDPs.  

Our first extension pertains to stability analysis. We extend Borkar and Meyn's result \citeyearpar{BoM00} to accommodate more general noise conditions for asynchronous SA algorithms (cf.\ \cref{assu: noise}), which are needed, particularly for addressing the inter-option algorithm for solving the underlying SMDPs. This extension requires a deep dive into Borkar and Meyn's stability proof for synchronous SA algorithms, modifying critical parts of the proof by constructing auxiliary processes. We will state our result in Section~\ref{sec: overview} (Theorem~\ref{thm: async sa}), referring interested readers to our separate paper \citep{yu2023note} for detailed proofs. We will subsequently apply this result to analyze the RVI algorithms' behavior in \cref{sec: c1 general RVI Q}.

Our second extension involves characterizing the solution properties of the associated ODEs. As we showed earlier, in the case of weakly communicating MDPs/SMDPs, the equations associated with the RVI algorithms generally have non-unique solutions, resulting in their corresponding ODEs having multiple equilibrium points. This differs from the case considered previously in \citep{abounadi2001learning}, where the ODE involved always possesses a unique equilibrium. In \cref{sec: c1 general RVI Q}, we will focus on carrying out this second extension.

We will now introduce the materials to be employed in our subsequent analysis, including several definitions and concepts related to ODEs, as well as our recent extension of Borkar and Meyn's result mentioned earlier.

\subsection{Preliminaries and an Extended SA Result for Analysis} \label{sec: overview}

For a Lipschitz continuous function $h: \R^d \to \R^d$, consider the ODE $\dot{x}(t) = h(x(t))$. 
This ODE is well-posed: for each initial condition $x_0 \in \R^d$, it has a unique solution $x(t)$ defined on $\R$ and satisfying $x(0) = x_0$. A point $x \in \R^d$ is an \emph{equilibrium} of the ODE if $h(x) = 0$. A set $A \subset \R^d$ is \emph{invariant} for the ODE if, whenever $x(0) \in A$, the solution $x(t) \in A$ for all $t \in \R$. Equivalently, $A$ is invariant if and only if for all $t \in \R$, $A = \cup_{x(0) \in A} \{x(t)\}$.

We will also need the notions of Lyapunov stability and global asymptotic stability of a set or a point.
Let $A$ be a compact subset of $\R^d$ and $A^\delta$, where $\delta > 0$, its closed $\delta$-neighborhood. The set $A$ is called \emph{stable} for the ODE in the sense of Lyapunov if, given any $\epsilon > 0$, there exists $\delta > 0$ such that for all initial conditions $x(0) \in A^{\delta}$, $x(t) \in A^\epsilon$ for all $t \geq 0$. The set $A$ is \emph{globally asymptotically stable} if it is stable and for all initial conditions $x(0) \in \R^d$, $x(t)$ approaches the set $A$ as $t \to \infty$. (See \citet[Chap.\ 4.2.2]{KuY03} for a reference.) A point $x \in \R^d$ is called stable or globally asymptotically stable for the ODE, if the set $\{x\}$ has the respective stability property.

In the Borkar-Meyn framework \citeyearpar{BoM00}, we consider a Lipschitz continuous function $h$ with additional properties that ensure no solution $x(t)$ of the ODE would ``drift'' to infinity as $t \to \infty$. These properties are specified in terms of the scaling limit of the function $h$ (i.e., the function $h_\infty$ defined below) as follows.

\begin{myassumption}[conditions on the function $h$] \label{assu: h} \hfill
\begin{enumerate}[leftmargin=0.8cm,labelwidth=!]
\item[\rm (i)] Lipschitz continuity: for some $0 \leq L < \infty$, $\| h(x) - h(y) \| \leq L \| x - y\|$ for all $x, y \in \R^{d}$.
\item[\rm (ii)]  For $c \geq 1$ and functions $h_c(x) \= h(cx)/c$, we have $h_c(x) \to h_\infty(x)$ as $c \to \infty$, uniformly on compact subsets of $\R^{d}$, where $h_\infty$ is a continuous function on $\R^{d}$.
\footnote{It is worth noting that in this assumption, part (ii) is the same as the pointwise convergence of $h_c(x)$ to some real-valued function $h_\infty(x)$ as $c \to \infty$. This is because if the pointwise limit $h_\infty$ exists, it must be Lipschitz continuous with the same Lipschitz constant as $h$, and the convergence must be uniform on compact subsets of $\R^d$.}
\item[\rm (iii)] Furthermore, the ODE\
\begin{align}\label{eq: hinfty ode}
    \dot{x}(t) = h_\infty (x(t)) 
\end{align}
has the origin as its unique globally asymptotically stable equilibrium. 
\end{enumerate}
\end{myassumption}

Let $\I = \{1, 2, \ldots, d\}$, and write $h_i$ for the $i$th component of $h$. In our work \citep{yu2023note}, we have studied a class of asynchronous SA algorithms described by the update rule:
 \begin{equation} \label{eq: async sa}
    Q_{n+1}(i)  = Q_n(i)  + \alpha_{\nu_n(i)} \left( h_i (Q_n) + M_{n+1}(i) + \epsilon_{n+1}(i) \right) \ind \{ i \in Y_n\}, \ i \in \I,
\end{equation}
where $Q_0$ is a given initial vector. Similar to the RVI Q-learning algorithms, $Y_n$ is a nonempty random subset of $\I$, $\nu_n(i) = \sum_{k=0}^n \ind \{ i \in Y_k\}$, and $\{\alpha_n\}$ and $\{\nu_n\}$ satisfy Assumptions~\ref{assu: stepsize} and \ref{assu: update}. The terms $M_{n+1}$ and $\epsilon_{n+1}$ represent two types of noises present in the evaluation of $h(Q_n)$: $M_{n+1}$ accounts for noise with zero conditional mean, while $\epsilon_{n+1}$ may have a nonzero conditional mean. These noise terms are subject to the following conditions:

Let $\{\F_n\}$ be an increasing family of $\sigma$-fields such that $\F_n \supset \sigma (Q_m, Y_m, M_m, \epsilon_m; m \leq n)$.

\begin{myassumption}[conditions on the noise terms] \label{assu: noise} \hfill
\begin{itemize}[leftmargin=0.7cm,labelwidth=!]
\item[\rm (i)] For all $n \geq 0$, $\E [ \| M_{n+1} \| ] < \infty$, $\E [ M_{n+1} \mid \F_n ] = 0$ a.s.,
\footnote{This means that $\{M_n\}$ is a martingale-difference sequence.}
and moreover, there exists a deterministic constant $K \geq 0$ such that
$ \E \left[ \| M_{n+1} \|^2 \mid \F_n \right] \leq K \left(1 +\| Q_n \|^2 \right)$ a.s.
\item[\rm (ii)] For all $n \geq 0$, $\norm{\epsilon_{n+1}} \leq \delta_{n+1} (1 + \norm{Q_n})$, where $\delta_{n+1}$ is $\calF_{n+1}$-measurable and as $n \to \infty$, $\delta_n \to 0$ a.s.
\end{itemize}
\end{myassumption} 

In the context of a specific algorithm, $\F_n$ typically represents the history of the algorithm up to time step $n$. The term $M_{n+1}$ represents a ``centered'' component, while $\epsilon_{n+1}$ represents a ``biased'' component, deviating from the desired value $h(Q_n)$. Assumption~\ref{assu: noise}(ii) requires that the biased noise component becomes vanishingly small relative to $1 + \|Q_n\|$ as time progresses, although it needs not vanish absolutely should $\{Q_n\}$ become unbounded. This noise term, $\epsilon_{n+1}$, arises in our inter-option algorithm for solving an SMDP, as the function $h$ in this case depends on expected holding times in the SMDP, parameters that can only be estimated with increasing accuracy over time. 

Under these conditions, we have shown, by extending Borkar and Meyn's stability proof, that the iterates $\{Q_n\}$ from algorithm \eqref{eq: async sa} is almost surely bounded. This stability result, combined with SA theory \citep{Bor98, Bor00, Bor09}, yields the following theorem, which we will apply in our subsequent convergence analysis of the RVI Q-learning and options algorithms.

\begin{mytheorem}[\text{\citet[Theorems 1 and 2]{yu2023note}}] \label{thm: async sa}
Under Assumptions~\ref{assu: stepsize}, \ref{assu: update}, \ref{assu: h}, and \ref{assu: noise}, almost surely, the sequence $\{Q_n\}$ generated by \eqref{eq: async sa} is bounded and converges to a (possibly sample path-dependent) compact, connected, internally chain transitive,
\footnote{See \citet[Section 2.1]{Bor09} for definition; we will not use this property in this work.}
invariant set of the ODE\ $\dot{x}(t) = h(x(t))$.
\end{mytheorem}
\vspace*{-0.22cm}

Before proceeding, we make some additional comments regarding the stability aspect of prior analyses of RVI Q-learning algorithms and related works: \vspace*{-0.15cm}
\begin{myremark}  \rm \label{remark: sa result compare} 
\noindent (a) Previous convergence proofs for Differential/RVI Q-learning \citep{wan2021learning} and the two options algorithms \citep{wan2021average} have a notable gap: They applied a convergence result from the book \citep[Chap.\ 7.4]{Bor09} for asynchronous SA algorithms without first establishing its required condition on the stability of the algorithms. Therefore, these previous analyses are considered inadequate. \\*[0.02cm]
\noindent (b) In the convergence analysis of RVI Q-learning by \citet{abounadi2001learning}, the authors relied on a stability assertion for asynchronous SA algorithms from \citet[Theorem 2.5]{BoM00}. This theorem is set in a general distributed computing framework that allows for communication delays (which are not considered in our algorithmic framework). However, \citet{BoM00} did not provide an explicit proof of this stability result. Additionally, their conditions on the noise terms are stronger than ours: The martingale-difference noise terms $M_n$ are required to adhere to a specific form, whereas the noise terms $\epsilon_n$ are absent. For a more detailed discussion, see \citep[Remark 1(b) and the Appendix]{yu2023note}. \\*[0.02cm]
\noindent (c) Within the Borkar-Meyn framework, \citet[Theorem 1]{bhatnagar2011borkar} provided a stability proof for asynchronous SA with bounded communication delays, where he required the noise component $M_{n+1}$ to be bounded by $\| M_{n+1} \| \leq K (1 + \| x_n\|)$ for all $n \geq 0$, for some deterministic constant $K$. This condition is much more restrictive than the standard condition on martingale-difference noises described in \cref{assu: noise}(i).\qed
\end{myremark}
\vspace*{-0.35cm}

\subsection{An Abstract Stochastic RVI Algorithm and Its Convergence} \label{sec: c1 general RVI Q}
 
In this section, we introduce an abstract stochastic RVI algorithm and establish its convergence. By abstracting away context and implementation details, this algorithm unifies the three specific algorithms of interest, allowing us to focus on essential arguments in their convergence analysis.

The objective of this algorithm is to solve an equation that involves a max-norm nonexpansive mapping. Specifically, it aims to find a solution of $(\bar r, q)$ to the following equation:
\begin{align} \label{eq: General RVI Q Bellman equation}
    r(i) - \bar r + g(q)(i) - q(i) = 0, \quad \forall\, i \in \ispace \, \=  \,\{1, \ldots, d\}.
\end{align}
Here, $\bar r \in \bbR$ and $q \in \bbR^{d}$ are unknown variables to be solved for, while $r \in \bbR^d$ is a given vector. The mapping $g: \bbR^{d} \to \bbR^{d}$ possesses nonexpansiveness and other properties similar to those encountered in previously studied cases, as detailed below. In addition, we assume that the solutions of this equation exhibit a structure reminiscent of the specific optimality equations discussed earlier.

\begin{myassumption}[conditions on $g$]\label{assu: g} \hfill
\begin{itemize}[leftmargin=0.8cm,labelwidth=!]
\item[\rm (i)] The mapping $g$ is nonexpansive w.r.t.\ the max-norm: $\norm{g(x) - g(y)}_\infty \leq \norm{x - y}_\infty$ for all $x, y \in \R^d$.
\item[\rm (ii)] For all $c \in \bbR$ and $x \in \bbR^{d}$, $g(x + c\onevec) = g(x) + c\onevec$.
\item[\rm (iii)] For all $c \geq 0$ and $x \in \bbR^{d}$, $g(cx) = cg(x)$.
\end{itemize}
\end{myassumption}

\begin{myassumption}[conditions on the solution set of \eqref{eq: General RVI Q Bellman equation}]\label{assu: solution set} \hfill
\begin{itemize}[leftmargin=0.7cm,labelwidth=!]
    \item[\rm (i)] Equation \eqref{eq: General RVI Q Bellman equation} admits at least one solution of $(\bar r, q)$. All these solutions share a common value of $\bar r$, denoted by $\GRVIQsolutionrbar$.
\item[\rm (ii)] If $r(\cdot) \equiv 0$ instead, then $(\bar r, q) = (0, c \onevec)$, $c \in \R$, are the only solutions to \eqref{eq: General RVI Q Bellman equation}.
\end{itemize}
\end{myassumption}

\begin{myremark}\label{remark: specific equations} \rm
(a) Equation \eqref{eq: General RVI Q Bellman equation} encompasses the specific optimality equations of interest, including the action-value optimality equation \eqref{eq: action-value optimality equation}, the optimality equation \eqref{eq: intra-option option-value optimality equation} for the intra-option algorithm, and the scaled equivalent form \eqref{eq: scaled inter-option option-value optimality equation} of the option-value optimality equation \eqref{eq: inter-option option-value optimality equation} for the inter-option algorithm. This relationship can be seen by comparing these specific equations with \eqref{eq: General RVI Q Bellman equation} term by term. The details will be given in \cref{sec: proofs for specific algs}, where we apply the results of this subsection to specific algorithms.\\*[0.02cm]
\noindent (b) In the context of MDPs and SMDPs, \cref{assu: solution set} is satisfied if the MDP/SMDP is weakly communicating (cf.\ \cref{lemma: 0 reward MDP has 1 d solution}). More generally, this assumption holds true in an MDP or SMDP where the optimal average reward rate remains constant, and the policy that applies every action with positive probability induces a single recurrent class of states (along with a possibly empty set of transient states). In particular, it can be deduced from the theory of \citepalias{ScF78} (cf.\ \cref{sec: degree of freedom Q}) that \cref{assu: solution set}(ii) must hold in this case. Thus, the convergence result we present below applies to this broader class of MDPs/SMDPs, not only to those weakly communicating ones.\qed
\end{myremark}
 
With $f$ satisfying \cref{assu: f}, define a subset of solutions of $q$ to \eqref{eq: General RVI Q Bellman equation} by
\begin{align}\label{eq: qsharp}
    \GRVIQsolutionq \= \{q \in \R^d : \text{$(\GRVIQsolutionrbar, q)$ solves \eqref{eq: General RVI Q Bellman equation};} \ f(q) = \GRVIQsolutionrbar\}.
\end{align}
Define a function $h: \R^d \to \R^d$ by
\begin{equation} \label{eq: h for RVI}
    h(q) \= r - f(q) \onevec + g(q) - q, \quad q \in \R^d.
\end{equation}
The following lemma examines implications of the preceding assumptions, some of which will be directly used in subsequent analysis, while others serve to define the scope of problems addressable by our abstract framework.

\begin{mylemma} \label{lem: basic properties of abstract RVI}
Assumptions \ref{assu: f}, \ref{assu: g} and \ref{assu: solution set} together imply the following:
\begin{itemize}[leftmargin=0.8cm,labelwidth=!]
\item[\rm (i)] The set $\GRVIQsolutionq$ is nonempty, connected, and compact. It is the solution set of $h(q) = \zerovec$.
\item[\rm (ii)] The function $h$ satisfies \cref{assu: h}(i, ii) with $h_\infty(q) =  f(\zerovec)\onevec  - f(q)\onevec + g(q) - q$. \item[\rm (iii)] The origin is the unique solution to $h_\infty(q) = \zerovec$. 
\end{itemize}
\end{mylemma}

\begin{proof}
First, observe that Assumptions \ref{assu: g} and \ref{assu: solution set} lead to implications similar to the solution properties present in the specific problems we considered earlier:
\begin{itemize}[leftmargin=0.7cm,labelwidth=!]
\item[\rm (a)] If $(\GRVIQsolutionrbar, q)$ solves \eqref{eq: General RVI Q Bellman equation}, then so does $(\GRVIQsolutionrbar, q+ c\onevec)$ for all $c \in \R$.
\item[\rm (b)] If $r(\cdot) \equiv b$ for some $b \in \R$ instead, then $(b, c \onevec)$, $c \in \R$, are the only solutions to \eqref{eq: General RVI Q Bellman equation}.
\end{itemize}
Indeed, (a) follows from Assumption \ref{assu: g}(ii) since $g(q+ c \onevec) - (q+ c \onevec) = g(q) - q$ under this assumption, while (b) is a consequence of this assumption combined with Assumption \ref{assu: solution set}(ii).

We now verify the three statements of the lemma:

\smallskip
\noindent (i) First, implication (a) and Assumption \ref{assu: solution set}(i), together with Assumption \ref{assu: f}(ii) on $f$, ensure the nonemptiness of $\GRVIQsolutionq$. The reasoning is the same as that used in proving the nonemptiness part of Theorem~\ref{thm: smdp properties}(i): Let $(\GRVIQsolutionrbar, q)$ be a solution to \eqref{eq: General RVI Q Bellman equation}. Define $q_* = q + (\GRVIQsolutionrbar - f(q)) \onevec/u$, where  $u > 0$ is the constant from \cref{assu: f}(ii). Then by implication (a), $(\GRVIQsolutionrbar, q_*)$ solves \eqref{eq: General RVI Q Bellman equation}, and by Assumption \ref{assu: f}(ii), $f(q_*) = f(q) + u (\GRVIQsolutionrbar - f(q)) \onevec/u = \GRVIQsolutionrbar$. Consequently, $q_* \in \GRVIQsolutionq \not= \varnothing$.

Given the continuity of $f$ and $g$ (Assumptions \ref{assu: f}(i) and \ref{assu: g}(i)), it is evident that $\GRVIQsolutionq$ is closed by definition. Its compactness can be deduced similarly to the compactness proof for Theorem~\ref{thm: smdp properties}(i), with the optimality equation in that proof replaced by \eqref{eq: General RVI Q Bellman equation}. Specifically, assuming $\GRVIQsolutionq$ is unbounded, we would have an unbounded sequence $\{x_n\}$ in $\GRVIQsolutionq$ such that, as $n \to \infty$, $y_n \= x_n / \| x_n\|$ converges to some point $y_\infty \in \R^d$ with $\|y_\infty \| = 1$ that satisfies the relations: 
\begin{equation} \label{eq-compactQ-prf1}
   y_\infty = g(y_\infty), \qquad f(y_\infty) = f(\zerovec).
\end{equation}   
However, under Assumptions \ref{assu: solution set}(ii) and \ref{assu: f}(ii), the only solution to \eqref{eq-compactQ-prf1} is $\zerovec$, contradicting $y_\infty \not= \zerovec$. Thus, $\GRVIQsolutionq$ must be compact. 

For the connectedness of the set $\GRVIQsolutionq$, consider the set $\mathcal{Q}'$ of solutions to the equation $q = r - r_\# \onevec + g(q)$, which is nonempty by Assumption \ref{assu: solution set}(i). Since $g$ is nonexpansive w.r.t.\ $\|\cdot\|_\infty$ (Assumption \ref{assu: g}(i)) and by \citet[Theorem 4.1]{borkar1997analog}, the set of fixed points of a $\|\cdot\|_\infty$-nonexpansive mapping is connected, $\mathcal{Q}'$ is connected. Using Assumption \ref{assu: f}(ii) on $f$, it then follows from the same proof for the connectedness part of Theorem~\ref{thm: smdp properties}(i) (substituting $\mathcal{Q}'$ for $\smdpoptimalitysolutionv$) that $\GRVIQsolutionq$ is connected.

Finally, since by Assumption \ref{assu: solution set}(i), $\GRVIQsolutionq$ is the solution set of $h(q) = \zerovec$, statement (i) is proven.

\smallskip
\noindent (ii) The Lipschitz continuity of $h$ follows from that of $f$ and $g$ (Assumptions \ref{assu: f}(i) and \ref{assu: g}(i)). 
Since $f(cq) = f(\zerovec) + c (f(q) - f(\zerovec))$ and $g(cq) = c g(q)$ for $c \geq 0$ by Assumptions \ref{assu: f}(iii) and \ref{assu: g}(iii), we have that as $c \to \infty$, 
$$h(cq)/c = \big( r - f(cq) \onevec  + g(c q) - c q \big)/c \  \to \ f(\zerovec)\onevec - f(q)\onevec + g(q) - q, $$
and the convergence is uniform on the entire space of $q$.
This proves that $h$ satisfies \cref{assu: h}(i, ii) with the function $h_\infty$ as stated in the lemma.

\smallskip
\noindent (iii) Statement (iii) follows from Assumption~\ref{assu: f}(ii) on $f$ and implication (b) mentioned earlier, applied with $b = f(\zerovec)$.
\end{proof}
\vspace*{-0.5cm}

The abstract stochastic RVI algorithm we now introduce aims to solve \eqref{eq: General RVI Q Bellman equation} by solving $h(q) = \zerovec$ [cf.\ \eqref{eq: h for RVI}]: Starting from some initial $Q_0 \in \R^d$, compute iteratively $Q_{n+1}$ at time step $n$ by updating the individual components for a randomly selected nonempty subset $Y_n \subset \ispace$ according to
\begin{align}
    Q_{n+1}(i) & \= Q_n(i) + \alpha_{\nu_n( i)} \big( r(i) - f(Q_n) + g(Q_n)(i) - Q_n(i) ) \ind \{i \in Y_n\} \nonumber\\
    &  \quad + \alpha_{\nu_n( i)} ( M_{n+1}(i) + \epsilon_{n+1}(i) \big) \ind \{i \in Y_n\}, \label{eq: c1 General RVI Q async update}
\end{align}
where $\nu_n(i) = \sum_{k=0}^n \ind \{ i \in Y_k\}$. The function $f$, $\{\alpha_n\}$, and $\{\nu_n\}$ satisfy Assumptions \ref{assu: f}, \ref{assu: stepsize}, and \ref{assu: update}, as in the previously studied cases, while $M_{n+1}$ and $\epsilon_{n+1}$ are noise terms that satisfy \cref{assu: noise} w.r.t.\ an increasing family of $\sigma$-fields $\F_n$ containing $\sigma(Q_m, Y_m, M_m, \epsilon_m; m \leq n)$ for $n \geq 0$.

\begin{mytheorem}\label{thm: General RVI Q}
Under Assumptions~\ref{assu: f}--\ref{assu: update} and \ref{assu: noise}--\ref{assu: solution set}, almost surely, the sequence $\{Q_n\}$ generated by algorithm \eqref{eq: c1 General RVI Q async update} is bounded and converges to a compact connected subset of $\GRVIQsolutionq$, with $f(Q_n) \to \GRVIQsolutionrbar$ consequently.
\end{mytheorem}

\begin{myremark} \rm
The preceding theorem characterizes the algorithm's convergence behavior in terms of its individual iterates. Further characterization in terms of segments of consecutive iterates can be made by combining our analysis below with \citep[Corollary 2]{yu2023note}. This additional characterization reveals that although $\{Q_n\}$ may not converge to a single point, the algorithm will spend increasingly more ``ODE-time''
\footnote{Here ``ODE-time'' is a sense of time introduced in the ODE-based analysis. The amount of ``ODE-time'' elapsed during an iteration is the sum of the step sizes $\alpha_{\nu_n(i)}$ involved in all the component updates at that iteration (see \citet{yu2023note} for details).}
in arbitrarily small neighborhoods around its iterates' limit points, with the duration spent around each limit point tending to infinity, thereby creating the appearance of convergence to a single point.\qed
\end{myremark}
\vspace*{-0.15cm}

In the rest of this subsection, we prove \cref{thm: General RVI Q}. We intend to invoke Theorem~\ref{thm: async sa} with the function $h$ defined by \eqref{eq: h for RVI}. As can be seen, \eqref{eq: c1 General RVI Q async update} has the same form as \eqref{eq: async sa} with this choice of $h$, and in Lemma~\ref{lem: basic properties of abstract RVI}(ii, iii), we have already verified that $h$ partially satisfies Assumption~\ref{assu: h}, a requirement of Theorem~\ref{thm: async sa}. Therefore, based on Theorem~\ref{thm: async sa}, we can obtain Theorem~\ref{thm: General RVI Q} if we can show that:
\begin{itemize}[leftmargin=0.5cm,labelwidth=!]
\item[1.] The origin is globally asymptotically stable for the ODE $\dot{x}(t) = h_\infty(x(t))$. (This will fulfil Assumption~\ref{assu: h} on $h$, making Theorem~\ref{thm: async sa} applicable.)
\item[2.] Every compact invariant set of the ODE $\dot{x}(t) = h(x(t))$ is contained in its equilibrium set $\GRVIQsolutionq$. (This, together with Theorem~\ref{thm: async sa}, will yield the convergence of $\{Q_n\}$ to $\GRVIQsolutionq$.)
\end{itemize}

\begin{myremark} \rm \label{remark: stability ascpect}
Establishing statement 1 alone will, as per the boundedness part of Theorem~\ref{thm: async sa}, ensure the almost sure boundedness of the iterates $\{Q_n\}$. In our abstract framework, an assumption crucial for statement 1, and hence the stability of the algorithm, is Assumption~\ref{assu: solution set}(ii) concerning the solutions to equation \eqref{eq: General RVI Q Bellman equation} in the special case of $r(\cdot) \equiv 0$. For the specific RVI Q-learning algorithms, this assumption corresponds to the solution property described in Lemma~\ref{lemma: 0 reward MDP has 1 d solution} for a weakly communicating MDP/SMDP with zero rewards.\qed
\end{myremark}
\vspace*{-0.3cm}

We now proceed to prove the preceding two statements by investigating the solution properties of the ODEs involved through a series of lemmas. A key step will be to show that the set $\GRVIQsolutionq$ is globally asymptotically stable for the ODE $\dot{x}(t) = h(x(t))$ (\cref{lem-cvg-5}). Our approach closely follows the line of reasoning presented in \citep[Sec.~3.1]{abounadi2001learning} for RVI Q-learning, as also utilized in prior works \citep{wan2021average,wan2021learning}. However, we extend this approach to encompass the more general scenario where $\GRVIQsolutionq$ is not necessarily a singleton.

It is worth noting that, as indicated in Lemma~\ref{lem: basic properties of abstract RVI}(iii), the ODE $\dot{x}(t) = h_\infty(x(t))$ has a unique equilibrium point at the origin. Consequently, we can already deduce the global asymptotic stability of the origin for this ODE (the first statement above) based on the aforementioned prior analyses. However, this conclusion will also emerge as a special case of our broader analysis.

As in \citep[Sec.\ 3.1]{abounadi2001learning}, to study the solution property of the ODE 
\begin{equation}    
\dot x(t) = h(x(t)), \quad \text{where} \ h(q) = r - f(q) \onevec + g(q) - q, \ \ q \in \R^d, \label{eq: c1 original ode}
\end{equation}
we shall first relate its solution to the solution of another ODE defined as
\begin{equation}
 \dot{y}(t) = h'(y(t)), \quad \text{where} \ h'(q) \= r - \GRVIQsolutionrbar \onevec + g(q) - q, \ \ q \in \R^d. \label{eq: c1 aux ode}
\end{equation}
Alternatively, the latter ODE can be written as
\begin{equation}
    \dot y(t) = T_1 (y(t)) - y(t), \quad \text{where} \ T_1 (q) \= r - \GRVIQsolutionrbar \onevec + g(q). \notag
\end{equation}
Under the max-norm, the mapping $T_1$ is nonexpansive due to the nonexpansiveness of $g$ (\cref{assu: g}(i)), and
its set of fixed points is nonempty since this is the same as the set of solutions to $h'(q) = \zerovec$, which is nonempty by \cref{assu: solution set}(i). For mappings $T_1$ that satisfy these conditions, \citet[Theorem 3.1 and Lemma 3.2]{borkar1997analog} have characterized the solution properties of the ODE $\dot y(t) = T_1 (y(t)) - y(t)$. The following lemma restates their general results for the case considered here. 

\begin{mylemma}[cf.\ \citet{borkar1997analog}]
\label{lemma: c1 aux ode convergence}
Let $y(t)$ be a solution of the ODE \eqref{eq: c1 aux ode}. Then for any equilibrium point $\bar y$ of \eqref{eq: c1 aux ode}, the distance $\norm{y(t) - \bar y}_\infty$ is nonincreasing, and as $t \to \infty$, $y(t) \to y_\infty$, an equilibrium point of \eqref{eq: c1 aux ode} that may depend on $y(0)$.
\end{mylemma}

Unlike ODE~\eqref{eq: c1 aux ode}, ODE\eqref{eq: c1 original ode} cannot be expressed as $\dot{x}(t) = T_2(x(t)) - x(t)$ for some nonexpansive mapping $T_2$ because the mapping $ q \mapsto r - f(q) \onevec + g(q)$ lacks the nonexpansiveness property in general. However, the functions $h$ and $h'$ defining the two ODEs differ only by a constant vector (i.e., a vector with identical entries). As assumed, for the function $g$, a constant shift in its argument yields the same shift in its output: $g(x + c \onevec) = g(x) + c \onevec$, $c \in \R$, by \cref{assu: g}(ii). From these observations, it can be deduced that with the same initial condition, the solutions $x(t)$ and $y(t)$ of the two ODEs must differ by a constant vector at any given time. This deduction, along with an expression of the difference $x(t) - y(t)$ in terms of $y(t)$, is presented in the next lemma. 

\citet{abounadi2001learning} first derived this result. (They considered $u = 1$ in their framework, and \citet{wan2021learning} extended the derivation to the more general case $u > 0$.) Their proof relied also on the nonexpansiveness of the mapping $T_1$ w.r.t.\ the span seminorm.
\footnote{Recall that the span seminorm on $\R^d$ is defined as $\| x \|_{\text{sp}} = \max_i x(i) - \min_i x(i)$.
If the mapping $g$ is monotonic (i.e., $x \geq y$ implies $g(x) \geq g(y)$), then under \cref{assu: g}(i, ii), $g$ must be nonexpansive w.r.t.\ the span seminorm (which can be verified directly). For any of the specific RVI algorithms for MDPs/SMDPs considered, the corresponding mapping $g$ is indeed monotonic. Thus, for these specific algorithms, $g$ and hence $T_1$ are nonexpansive w.r.t.\ the span seminorm as well.}
Here, we provide an alternative proof that does not require this assumption. Instead, we directly utilize the existence and uniqueness of solutions to the autonomous and nonautonomous ODEs involved, along with the aforementioned observations.

\begin{mylemma}\label{lemma: c1 connection between original and aux ode}
If $x(t)$ and $y(t)$ are solutions of the ODEs \eqref{eq: c1 original ode} and \eqref{eq: c1 aux ode}, respectively, with the same initial condition $x(0) = y(0)$, then $x(t)= y(t) + z(t) \onevec$, where $z(t)$ is the unique solution of the ODE $\dot z(t)= - u z(t) + (\GRVIQsolutionrbar - f(y(t)))$ with $z(0) = 0$, and $u > 0$ is the constant from \cref{assu: f}(iii).
\end{mylemma}

\begin{proof}
For a given initial condition $x(0) = y(0) = y_0$ and the corresponding solution $y(t)$ of \eqref{eq: c1 aux ode}, let us consider a function $\phi(t) = y(t) + z(t) \onevec$, where $z(t)$ is some real-valued differentiable function with $z(0) = 0$. If $\phi$ satisfies the ODE \eqref{eq: c1 original ode}, then it must coincide with $x(\cdot)$ since \eqref{eq: c1 original ode} has a unique solution for each initial condition.

For $\phi$ to satisfy \eqref{eq: c1 original ode}, it is equivalent to have the term 
$$ \dot{\phi}(t) = \dot{y}(t) + \dot{z}(t) \onevec  = r - r_\# \onevec + g(y(t)) - y(t) + \dot{z}(t) \onevec $$
coincide with the term $h(\phi(t))$, which can be expressed as
$$ h(\phi(t))  = r - f(\phi(t)) \onevec + g(\phi(t)) -  \phi(t) = r - f(y(t) + z(t) \onevec) \onevec +  g(y(t)) -  y(t), $$
since $g(\phi(t))  = g(y(t)) + z(t) \onevec$ by \cref{assu: g}(ii). 
Comparing the two terms, this is equivalent to having $z(t)$ satisfy the following ODE:
\begin{equation} \label{eq-z-prf1}
    \dot{z}(t) =  \GRVIQsolutionrbar - f(y(t) + z(t) \onevec) = - u z(t) + (\GRVIQsolutionrbar - f(y(t)),
\end{equation} 
where the second equality follows from \cref{assu: f}(iii) on $f$. Applying the variation of parameters (or constants) formula [see, e.g., \citep[p.\ 99]{HiS74}], the solution to this ODE is given by
\begin{equation} \label{eq: expression for z}
    z(t) = \int_0^t \exp(u (\tau - t)) \left (\GRVIQsolutionrbar - f(y(\tau)) \right) d\tau .
\end{equation}
This, together with the preceding proof, establishes that $x(t) = y(t) + z(t) \onevec$ with $z(t)$ satisfying the ODE \eqref{eq-z-prf1}.
\end{proof}
\vspace*{-0.3cm}

Recall the stability notions for ODEs introduced at the beginning of \cref{sec: overview}. The next lemma establishes the global asymptotic stability of the equilibrium set $\GRVIQsolutionq$ for the ODE~\eqref{eq: c1 original ode}. It extends prior results \cite[Theorem 3.4]{abounadi2001learning} and \cite[Lemma B.4]{wan2021learning}, which consider the case of a unique equilibrium point. While the proof arguments are similar, we give the details here for clarity and completeness.

For $\epsilon > 0$, denote by $\GRVIQsolutionq^\epsilon$ the closed $\epsilon$-neighborhood of $\GRVIQsolutionq$ w.r.t.\ $\| \cdot\|_\infty$.

\begin{mylemma} \label{lem-cvg-5}
The set $\GRVIQsolutionq$ is globally asymptotically stable for the ODE\ \eqref{eq: c1 original ode}. Furthermore, as $t \to \infty$,  every solution $x(t)$ of \eqref{eq: c1 original ode} converges to an element in $\GRVIQsolutionq$ depending on $x(0)$.
\end{mylemma}

\begin{proof}
We first prove the Lyapunov stability of $\GRVIQsolutionq$. Let $x(t)$ be a solution to \eqref{eq: c1 original ode}, and consider the solution $y(t)$ to \eqref{eq: c1 aux ode} with the same initial condition $y(0) = x(0)$. 
By \cref{lemma: c1 connection between original and aux ode}, we have $x(t) = y(t) + z(t) \onevec$, with $z(t)$ given by \eqref{eq: expression for z}.

For any $q_* \in \GRVIQsolutionq$, let us derive a bound on $\norm{q_* - x(t)}_\infty$ for $t \geq 0$, in terms of the initial distance $\norm{q_* - x(0)}_\infty$. Since $\norm{q_* - y(t)}_\infty \leq \norm{q_* - y(0)}_\infty$ by \cref{lemma: c1 aux ode convergence}, using the expression \eqref{eq: expression for z} for $z(t)$, we have
\begin{align}
    \norm{q_* - x(t)}_\infty  
    & = \norm{q_* - (y(t) + u z(t)  \onevec)}_\infty \nonumber\\
    & \leq \norm{q_* - y(t)}_\infty + u \abs{z(t)} \nonumber \\
    & \leq \norm{q_* - y(0)}_\infty + u \int_0^t \exp(u (\tau - t)) \abs{\GRVIQsolutionrbar - f(y(\tau))} d\tau \nonumber \\ 
    & = \norm{q_* - x(0)}_\infty + u \int_0^t \exp(u (\tau - t)) \abs{f(q_*) - f(y(\tau))} d\tau \label{eq: c1 globally asymptotically stable equilibrium lemma: eq 1},
\end{align}
where the last equality holds since $q_* \in \GRVIQsolutionq$ implies $f(q_*) = \GRVIQsolutionrbar$ [cf.\ \eqref{eq: qsharp}].
By the Lipschitz continuity of $f$ (\cref{assu: f}(i)), we have
\begin{align*}
    \abs{ f(q_*) - f( y(\tau))} & \leq L \norm{q_* - y(\tau)}_\infty \leq L \norm{q_* - y(0)}_\infty 
    = L \norm{q_* - x(0)}_\infty,
\end{align*}
where the second inequality holds by \Cref{lemma: c1 aux ode convergence}. Therefore,
\begin{align*}
    \int_0^t \exp(u (\tau - t)) \abs{ f(q_*) - f(y(\tau)) } d\tau & \leq \int_0^t \exp(u (\tau - t)) L \norm{q_* - x(0) }_\infty d\tau \\
    & = L \norm{q_* - x(0) }_\infty \int_0^t \exp(u (\tau - t)) d\tau \\
    & = \frac{L(1 - \exp(-u t))}{u }\norm{q_* - x(0) }_\infty.
\end{align*}
Substituting the above relation in \eqref{eq: c1 globally asymptotically stable equilibrium lemma: eq 1}, we obtain
\begin{equation} \label{eq-stability-prf1}
    \norm{q_* - x(t)}_\infty \leq (1 + L) \norm{q_* - x(0)}_\infty, \qquad \forall\, q_* \in \GRVIQsolutionq, \ t \geq 0.
\end{equation}

The Lyapunov stability of $\GRVIQsolutionq$ is now inferred from \eqref{eq-stability-prf1}: 
Given $\epsilon > 0$, let $\delta = \epsilon/(1+L)$. If $x(0) \in \GRVIQsolutionq^\delta$, then, since there is some $q_* \in \GRVIQsolutionq$ with $\norm{q_* - x(0)}_\infty \leq \delta$ and the distance $\norm{x(t) - q_*}_\infty \leq \epsilon$ for all $t \geq 0$ by \eqref{eq-stability-prf1}, it follows that $x(t) \in \GRVIQsolutionq^\epsilon$ for all $t \geq 0$.

We now prove that every solution of ODE~\eqref{eq: c1 original ode} converges to an element in $\GRVIQsolutionq$. This will not only confirm the second statement of the lemma but also, alongside the just-established Lyapunov stability of $\GRVIQsolutionq$, establish its global asymptotic stability.

To this end, let us consider \eqref{eq: expression for z}:
$ z(t) = \int_0^t \exp(u\tau - ut) \big(\GRVIQsolutionrbar - f(y(\tau) \big) d\tau.$
Observe that for each $t \geq 0$, the expression $\exp(u \tau - u t) d \tau$ defines a finite measure on the interval $[0, t]$ with a total mass of $\frac{1 - e^{-ut}}{u}$. As $t \to \infty$, the total mass of this measure tends to $\frac{1}{u}$, while the measure of any given bounded interval $[0, T]$ tends to $0$. Recall also that as $\tau \to \infty$, we have $f(y(\tau)) \to f(y_\infty)$ by the convergence of $y(\tau) \to y_\infty$ (\cref{lemma: c1 aux ode convergence}) and the continuity of $f$ (\cref{assu: f}(i)). From these two facts, it follows that as $t \to \infty$, $z(t) \to \frac{\GRVIQsolutionrbar - f(y_\infty)}{u}$ and hence, by \cref{lemma: c1 connection between original and aux ode},
$$ x(t) = y(t) + z(t) \onevec \ \ \to \ \ x_\infty \= y_\infty + (\GRVIQsolutionrbar - f(y_\infty)) \onevec / u.$$
By \cite[Chap.\ II, Theorem 2.8]{BhS02}, this convergence of $x(t) \to x_\infty$ implies that $x_\infty$ is an equilibrium point of ODE~\eqref{eq: c1 original ode}, and therefore $x_\infty \in \GRVIQsolutionq$ by \cref{lem: basic properties of abstract RVI}(i). Alternatively, we can verify directly $x_\infty \in \GRVIQsolutionq$, similarly to the nonemptiness proof for \cref{lem: basic properties of abstract RVI}(i).
\end{proof}
\vspace*{-0.35cm}

Finally, from the preceding lemma, we deduce the following statements needed to conclude the proof of \cref{thm: General RVI Q}.

\begin{mylemma} \label{lem: compact inv set}
Any compact invariant set of the ODE \eqref{eq: c1 original ode} is contained in $\GRVIQsolutionq$.
\end{mylemma}
\begin{proof}
For $x \in \R^d$, let $\phi(t; x)$ denote the solution of (\ref{eq: c1 original ode}) with $x(0) = x$. To prove the lemma, we employ proof by contradiction. Suppose $A$ is a compact invariant set of \eqref{eq: c1 original ode} but $A \not\subset \GRVIQsolutionq$. 
Then $d_{A,\GRVIQsolutionq} \= \sup_{x \in A} \inf_{y \in \GRVIQsolutionq} \| x - y \|_\infty > 0$ (since $\GRVIQsolutionq$ is closed by Lemma~\ref{lem: basic properties of abstract RVI}(i)). 

Let $0 < \epsilon < d_{A, \GRVIQsolutionq}$. By the Lyapunov stability of $\GRVIQsolutionq$ (Lemma~\ref{lem-cvg-5}), there exists $\delta > 0$ such that 
\begin{equation} \label{eq-ql-prf3}
    \phi(t; x) \in \GRVIQsolutionq^\epsilon, \ \ \forall \, t \geq 0, \ \ \text{if} \  x \in \GRVIQsolutionq^\delta.
\end{equation}
Also, by Lemma~\ref{lem-cvg-5}, for any $x \in \R^d$, $\phi(t;x)$ converges to $\GRVIQsolutionq$ as $t \to \infty$, and therefore, there exists a time $t_x$ such that $\phi(t_x; x) \in \GRVIQsolutionq^{\delta/2}$. Since $h$ is Lipschitz continuous (Lemma~\ref{lem: basic properties of abstract RVI}(ii)), $\phi(t; x)$ is continuous in $x$. Hence, there is an open neighborhood $D_x$ of $x$ such that 
\begin{equation} \label{eq-ql-prf4}
     \phi(t_x; y) \in  \GRVIQsolutionq^{\delta}, \qquad \forall \, y \in D_x.
 \end{equation}
As the collection $D_x, x \in A,$ forms an open cover of the compact set $A$, there exist a finite number of points $x^1, x^2, \ldots, x^l \in A$ with $A \subset \cup_{i=1}^l D_{x_i}$. 
Now let $\bar t = \max_{1 \leq i  \leq l} t_{x_i}$. Then by (\ref{eq-ql-prf4}) and \eqref{eq-ql-prf3}, we have
\begin{equation} \label{eq-ql-prf5}
  \phi(t; x) \in \GRVIQsolutionq^\epsilon, \qquad \ \forall \, x \in A, \ t \geq \bar t.
\end{equation}  
On the other hand, $\{\phi(\bar t; x) \mid x \in A\} = A$ since $A$ is invariant for the ODE\ (\ref{eq: c1 original ode}). Consequently,  (\ref{eq-ql-prf5}) implies that $A \subset \GRVIQsolutionq^\epsilon$, contradicting $d_{A, \GRVIQsolutionq} > \epsilon$. The proof is now complete.
\end{proof}
\vspace*{-0.3cm}

The following corollary follows from \cref{lem-cvg-5}.

\begin{mycorollary} \label{cor: global asym stability of origin}
The origin is the unique globally asymptotically stable equilibrium of the ODE $\dot{x}(t) = h_\infty(x(t))$.
\end{mycorollary}

\begin{proof}
We can reduce the case under concern to a special case treated in the preceding analysis as follows. By Lemma~\ref{lem: basic properties of abstract RVI}(ii), the function $h_\infty$ is given by $h_\infty(q) = f(\zerovec)\onevec - f(q) \onevec + g(q) - q$. 
If we replace $r$ with $f(0) \onevec$ in the preceding analysis, the function $h$ becomes identical to $h_\infty$, and the equilibrium set $Q_\#$ of the ODE \eqref{eq: c1 original ode}, $\dot{x}(t) = h(x(t))$, reduces to the singleton set $\{\zerovec\}$ (Lemma~\ref{lem: basic properties of abstract RVI}(iii)). Furthermore, the function $h'$ used in deriving Lemma~\ref{lem-cvg-5} becomes 
$$h'(q) = f(\zerovec) \onevec - \GRVIQsolutionrbar \onevec + g(q) - q = g(q) - q,$$
since, under Assumption~\ref{assu: solution set}(ii), the value $\GRVIQsolutionrbar$, as the unique solution of $\bar r$ to \eqref{eq: General RVI Q Bellman equation} when $r(\cdot) \equiv f(\zerovec)$, is precisely $f(\zerovec)$ (cf.\ implication (b) discussed in the proof of Lemma~\ref{lem: basic properties of abstract RVI}). Correspondingly, the ODE \eqref{eq: c1 aux ode}, $\dot{y}(t) = h'(y(t))$, has the nonempty set $\{c \onevec : c \in \R\}$ as its equilibrium set by Assumption~\ref{assu: solution set}(ii). 

This shows that the preceding analysis applies here. Consequently, by Lemma~\ref{lem-cvg-5}, the origin is the unique globally asymptotically stable equilibrium for the ODE $\dot{x}(t) = h_\infty(x(t))$.   
\end{proof}

\vspace*{-0.3cm}

\begin{proofof}{\cref{thm: General RVI Q}}
As discussed immediately after its statement, this theorem follows from the combination of Theorem~\ref{thm: async sa} with Lemma~\ref{lem: basic properties of abstract RVI}(ii, iii), \cref{cor: global asym stability of origin}, and Lemma~\ref{lem: compact inv set}.  
\end{proofof}

\vspace*{-0.5cm}

\subsection{Convergence of Specific RVI Q-Learning Algorithms} \label{sec: proofs for specific algs}
This section shows RVI Q-learning and its inter- and intra-option extensions are special cases of the abstract RVI algorithm \eqref{eq: c1 General RVI Q async update}. Their convergence results (Theorems\ \ref{thm: Extended RVI Q-learning}, \ref{thm: inter-option Differential Q-learning}, \ref{thm: intra-option Differential Q-learning}) then immediately follow from \cref{thm: General RVI Q}.

To simplify notation, in the following proofs, let $\| \cdot\|$ stand for $\|\cdot\|_\infty$.

\subsubsection{RVI Q-learning (\cref{thm: Extended RVI Q-learning}(i))}\label{sec: proof of Differential and RVI Q-learning}

Recall that RVI Q-learning \eqref{eq: Extended RVI Q-learning} aims to solve the action-value optimality equation \eqref{eq: action-value optimality equation}, which corresponds to the ``abstract optimality equation'' \eqref{eq: General RVI Q Bellman equation} with $\ispace = \sspace \times \aspace$ and $r$ and $g$ defined as
$$ r(i) = r_{sa}, \qquad  g(q)(i) = \sum_{s' \in \sspace} p_{ss'}^a \max_{a'} q(s', a'), \qquad i = (s,a) \in \ispace, \ q \in \R^{|\ispace|}$$
(where $r_{sa}$ and $p_{ss'}^a$ are the one-stage expected reward and state transition probability defined immediately after \eqref{eq: action-value optimality equation}).
The mapping $g$ here clearly satisfies \cref{assu: g}. In a weakly communicating MDP, the solution set of \eqref{eq: action-value optimality equation} satisfies \cref{assu: solution set} by \cref{lemma: 0 reward MDP has 1 d solution} and the basic optimality properties of MDPs (cf.\ \cref{sec: wc mdps}).  

We now rewrite RVI Q-learning \eqref{eq: Extended RVI Q-learning} in the form of the abstract update rule \eqref{eq: c1 General RVI Q async update} by defining the noise terms as $\epsilon_{n+1} = \zerovec$ and
$$ M_{n+1}(i) = R_{n+1}^{sa} - r_{sa} + \max_{a' \in \aspace} Q_n(S_{n+1}^{sa}, a') - g(Q_n)(s, a), \quad \text{if} \ i = (s,a) \in Y_n,$$
and $M_{n+1}(i) = 0$ otherwise. Let us verify that the noise terms $\{M_{n+1}\}$ satisfy \cref{assu: noise}(i) with $\F_n = \sigma(Q_m, Y_m, M_m; m \leq n)$. Then \cref{thm: Extended RVI Q-learning}(i) will follow immediately from \cref{thm: General RVI Q}. 

We verify below that $\E[\|M_n\|] < \infty$ for all $n \geq 1$; the remaining conditions in \cref{assu: noise}(i) can be verified straightforwardly. Since the random one-stage rewards $R_{n+1}^{sa}$ have finite variances under our model assumption (cf.\ \cref{MDPs with the Average-Reward Criterion}), we have 
\begin{equation}
        \E[\norm{M_{n+1}}] \leq K + 2 \E[\norm{Q_n}] \label{eq: M_n ineq}
\end{equation}
for some suitable constant $K$. 
That $\E[\norm{Q_n}] < \infty$ for all $n \geq 1$ can be easily verified using the iterative update rule of $Q_n$ \eqref{eq: Extended RVI Q-learning}, the finiteness of the one-stage rewards, the Lipschitz continuity of $f$ (\cref{assu: f}(i)), and the finiteness of $\sup_{n \geq 0} \alpha_n$ (\cref{assu: stepsize}(i)).

\cref{thm: Extended RVI Q-learning}(i) now follows from Theorem 6.2, as discussed earlier.

\subsubsection{Inter-Option Algorithm (\cref{thm: inter-option Differential Q-learning}(i))} \label{sec: proof of Inter-Option Algorithm}

The scaled equivalent form \eqref{eq: scaled inter-option option-value optimality equation} of the option-value optimality equation \eqref{eq: inter-option option-value optimality equation} is a special case of the ``abstract optimality equation'' \eqref{eq: General RVI Q Bellman equation} with the following correspondences: $\ispace = \sspace \times \aspace$ and for each $i = (s,o) \in \ispace$ and $q \in \R^{|\ispace|}$, 
$$ r(i) = \frac{\hat r_{so}}{\hat l_{so}},\qquad
    g(q)(i)  = \frac{1}{\hat l_{so}} \sum_{s' \in \sspace} \otrans_{ss'}^o \max_{o' \in \ospace} q(s', o') + \Big(1 - \frac{1}{\hat l_{so}}\Big) \cdot q(s, o) 
    $$
(where $\hat r_{so}$, $\hat l_{so}$, and $\otrans_{ss'}^o$ are the expected one-stage cumulative rewards, expected option durations, and transition probabilities defined immediately after \eqref{eq: inter-option option-value optimality equation}). Since $\hat l_{so} \geq 1$ for all $(s,o) \in \sspace \times \aspace$, the above mapping $g$ satisfies \cref{assu: g}. 
Since the associated SMDP is assumed to be weakly communicating, the solution set of \eqref{eq: scaled inter-option option-value optimality equation} (equivalently, \eqref{eq: inter-option option-value optimality equation}) satisfies \cref{assu: solution set} by \cref{lemma: 0 reward MDP has 1 d solution} and the basic optimality properties of SMDPs (cf.\ \cref{sec: smdp}).

With $r$ and $g$ thus defined, we rewrite the inter-option algorithm \eqref{eq: c2 Inter-option algorithm}-\eqref{eq: c2 Inter-option Differential TD-learning L} in the form of the abstract update rule \eqref{eq: c1 General RVI Q async update} by defining the noise terms as follows: For each $i = (s,o) \in Y_n$, 
\begin{align*} 
    M_{n+1}(i) &= \frac{R_{n+1}^{so} - \hat r_{so}}{L_n(s, o)} + \frac{\max_{o' \in \ospace} Q_n(S_{n+1}^{so}, o') - \sum_{s' \in \sspace} \otrans_{ss'}^o \max_{o' \in \ospace} Q_n(s', o')}{\hat l_{so}},\\
    \epsilon_{n+1}(i) &= \frac{\hat r_{so} + \max_{o' \in \ospace} Q_n(S_{n+1}^{so}, o') - Q_n(s, o) }{L_n(s, o)}  - \frac{\hat r_{so} + \max_{o' \in \ospace} Q_n(S_{n+1}^{so}, o') - Q_n(s, o)}{\hat l_{so}},
\end{align*}
while $M_{n+1}(i) = \epsilon_{n+1}(i) = 0$ if $i \not\in Y_n$.
We now verify that these noise terms $\{M_{n+1}\}$ and $\{\epsilon_{n+1}\}$ satisfy \cref{assu: noise} with $\F_n = \sigma(Q_m, Y_m, L_m, M_m, \epsilon_m; m \leq n)$. \cref{thm: inter-option Differential Q-learning}(i) will then follow immediately from \cref{thm: General RVI Q}. 

To verify that $\{M_{n+1}\}$ satisfies \cref{assu: noise}(i), we first observe from the update rule \eqref{eq: c2 Inter-option Differential TD-learning L} for $L_n$ that for all $n \geq 0$, $L_n(s,o)$ is bounded below by the deterministic positive constant $\min\{1, L_0(s, o)\}$. This is because each option takes at least one time step to terminate (i.e., $L_{n+1}^{so} \geq 1$ always), while the initial $L_0(s,o) > 0$, and the step sizes $\beta_n \in [0,1]$ by \cref{assu: option assumption}(ii). 

From this lower bound for $L_n$ it follows that $\E[\norm{\epsilon_{n+1}}] \leq K_1 + K_2 \E[\norm{Q_n}]$ for some constants $K_1, K_2 > 0$. Then, similarly to the previous proof in \cref{sec: proof of Differential and RVI Q-learning}, we apply induction to prove $\E[\norm{M_{n+1}}]< \infty$ for all $n \geq 0$, using the Lipschitz continuity of $f$ (\cref{assu: f}(i)), the boundedness of $\E[|R_{n+1}^{so}|]$ for all $n \geq 0$, $s \in \sspace, o \in \ospace$ (implied by \cref{assu: option assumption}), along with the finiteness of $\sup_{n \geq 0} \alpha_n$ (\cref{assu: stepsize}(i)), and the lower bound for $\{L_n\}$. The remaining conditions in \cref{assu: noise}(i) can be verified straightforwardly, using the lower bound for $\{L_n\}$ together with the boundedness of $\E[(R_{n+1}^{so})^2]$ for all $n \geq 0, s \in \sspace, o \in \ospace$ (implied by \cref{assu: option assumption}).

To verify that $\{\epsilon_{n+1}\}$ satisfies \cref{assu: noise}(ii), we first note that in updating $L_n$, the random option durations $L_{n+1}^{so}$ have finite variances (as implied by \cref{assu: option assumption}). Furthermore, for every state-option pair $(s,o)$, the corresponding component is updated infinitely often (as implied by \cref{assu: update}(i)), while the step sizes $\beta_n$ satisfy standard conditions (\cref{assu: inter-option extended rvi q learning}(ii)). 
Therefore, standard stochastic approximation results [e.g., \citep{blum1954approximation}] imply that as $n \to \infty$,
$$  L_n(s, o) \to \hat l_{so} \ \ a.s., \quad \forall \, s \in \sspace, \, o \in \ospace.$$
Now letting $\delta_{n+1} \= \max_{s \in \sspace, o \in \ospace} \left\{ \max\{|\hat r_{so}|, 2\} \cdot \left|\frac{1}{L_n(s, o)} -  \frac{1}{\hat l_{so}}\right|\right\}$, we have $\norm{\epsilon_{n+1}} \leq \delta_{n+1} (1 + \norm{Q_n})$ for all $n \geq 0$ and $\delta_{n+1} \to 0$ a.s.\ as $n \to \infty$. This verifies \cref{assu: noise}(ii). \cref{thm: inter-option Differential Q-learning}(i) then follows from \cref{thm: General RVI Q}, as discussed earlier.

\subsubsection{Intra-Option Algorithm (\cref{thm: intra-option Differential Q-learning}(i))} \label{sec: proof of Intra-Option Algorithm}

For the intra-option algorithm \eqref{eq: transformed intra-option update}, its associated optimality equation \eqref{eq: intra-option option-value optimality equation} corresponds to the ``abstract optimality equation'' \eqref{eq: General RVI Q Bellman equation} with $\ispace = \sspace \times \aspace$, and $r$ and $g$ defined as follows. For each $i = (s,o) \in \ispace$ and $q \in \R^{|\ispace|}$, 
$$ r(i) = r_{so}^{(1)} \= \sum_{a \in \aspace} \pi(a \mid s, o) r_{sa},\qquad
    g(q)(i) = \sum_{a \in \aspace} \pi(a \mid s, o) \sum_{s' \in \sspace} p_{ss'}^a  U[q](s',o). $$
Recall from \eqref{eq: intra-option option-value optimality equation2} that $U[q](s',o) = \beta(s', o) \max_{o' \in \ospace} q(s', o')+ (1 - \beta(s', o))  q(s', o)$, where $\beta(s', o)$ denotes the termination probability for the option $o$ at state $s'$. 
This mapping $g$ clearly satisfies \cref{assu: g}. As the associated SMDP is weakly communicating by assumption, the solution set of \eqref{eq: intra-option option-value optimality equation}, being the same as that of \eqref{eq: inter-option option-value optimality equation} (\cref{prop: c2 inter = intra equations}), satisfies \cref{assu: solution set}, as already verified in the previous inter-option case.

With $r$ and $g$ defined as above, we can express the intra-option algorithm \eqref{eq: transformed intra-option update} in the form of the abstract RVI algorithm \eqref{eq: c1 General RVI Q async update} by setting the noise term $\epsilon_{n+1}$ to zero and defining the noise term $M_{n+1}$ as follows. For each $i = (s,o) \in Y_n$,
\begin{align*}
M_{n+1}(i) & =   \rho_n(s, o) \cdot \Big (R_{n+1}^{so} -f(Q_n) + U[Q_n](S^{so}_{n+1}, o) -Q_n(s, o)\Big) \\
    & \quad \ - \big(r^{(1)}_{so} - f(Q_n) + g(Q_n)(s, o) - Q_n(s, o)\big); 
\end{align*}    
and $M_{n+1}(i) = 0$ otherwise. As in the previous proofs, if we show that $\{M_{n+1}\}$ satisfies \cref{assu: noise}(i) with $\F_n = \sigma(Q_m, Y_m, b_m, M_m, \epsilon_m; m \leq n)$, then we can directly derive \cref{thm: intra-option Differential Q-learning}(i) from \cref{thm: General RVI Q}. 

Recall that $\rho_n(s,o)$, $(s,o) \in Y_n$, are the importance sampling ratios defined w.r.t.\ the behavior policy $b_n$ as $\rho_n(s,o) = \pi(A_{n}^{s} \mid s, o) / b_n(A_{n}^{s} \mid s)$, where $A_{n}^{s} \sim b_n( \cdot \mid s)$.  These terms are bounded by a deterministic constant for all $n \geq 0$, by the definition of the intra-option algorithm \eqref{eq: transformed intra-option update}. Consequently, the verification of \cref{assu: noise}(i) in this case is very similar to that for RVI Q-learning in \cref{sec: proof of Differential and RVI Q-learning}, therefore omitted.

This concludes the proof.

\section{Supplementary Materials and Additional Analysis: Degrees of Freedom in RVI Algorithms' Solutions} \label{sec: solution set dim}
In Section \ref{sec: solution set}, we discussed various properties, including compactness and connectedness, of solution sets for RVI Q-learning/options algorithms (the sets $\ExtRVIQsolutionq$, $\ooptimalitydiffsolutionq$, and $\smdpoptimalitygeneralsolutionq$ in Theorems~\ref{thm: MDP characterize Q}, \ref{thm: SMDP characterize Q}, \ref{thm: smdp properties}).  In this section, we further investigate the degrees of freedom in these solutions. Our derivation is built upon the remarkable work of \citet{ScF78}, who studied the solution structure of average-reward optimality equations for MDPs or SMDPs. We begin by reviewing their key results, which shed light on how recurrence structures of stationary optimal policies determine the number $n^*$ of degrees of freedom in these equations. Subsequently, we show that their results imply that for a weakly communicating MDP/SMDP, the solution sets of RVI Q-learning algorithms can be parameterized by $n^*-1$ parameters within an $(n^*-1)$-dimensional convex polyhedron (cf. Theorem~\ref{thm-dim-Qs} and \eqref{eq-param-Qs}).

\subsection{Review: Degrees of Freedom in Average-Reward Optimality Equations} \label{sec: degree of freedom Q}
Since MDPs are special cases of SMDPs, we shall focus on the latter. 
Recall the action- and state-value optimality equations for a weakly communicating SMDP [cf.\ \eqref{eq: SMDP action-value optimality equation} and \eqref{eq: SMDP state-value optimality equation}]:
\begin{align} 
    q(s,a) &=   \tilde r_{sa} + \sum_{s' \in \S} p_{ss'}^a \max_{a' \in \A} q(s', a'), \qquad \forall \, s \in \S, \ a \in \A \label{eq-oe-q},\\
    v(s) &= \max_{a \in \A} \left\{ \tilde r_{sa} + \sum_{s' \in \S} p_{ss'}^a v(s') \right\}, \qquad \forall \, s \in \S, \label{eq-oe-v}
\end{align}   
where $\tilde r_{sa} \= r_{sa} - l_{sa} r_*$ and $r_{sa}, l_{sa}, p_{ss'}^a$ are one-step reward, transition time, and transition probability, respectively, defined in \eqref{eq: rsa} and \eqref{eq: pssa}. Recall that $\Qo$ (respectively, $\V$) is the set of all solutions to \eqref{eq-oe-q} (respectively, \eqref{eq-oe-v}). Then
\begin{align}
   q \in \Qo \ \ & \Rightarrow \ \ v_q(\cdot) \= \max_{a \in \A} q(\cdot, a) \in \V,   \label{eq-q2v} \\ 
   v \in \V   \, \ \ & \Rightarrow \ \ q_v \in \Qo \, \ \text{where} \ \, q_v(s,a) \= \tilde r_{sa} + \sum_{s' \in \S} p_{ss'}^a v(s'), \ \ (s,a) \in \S \times \A. \label{eq-v2q}
\end{align}
This sets up a one-to-one correspondence between $\Qo$ and $\V$, with the mappings $q \mapsto v_q$ and $v \mapsto q_v$ defining a homeomorphism---a one-to-one bicontinuous transformation---between the two spaces. 
 
\citet{ScF78} gave a comprehensive characterization of the solution set $\V$. (While we focus on the weakly communicating case, we mention that their work applies to general multichain SMDPs.) To describe their results, we need a few definitions.

Recall that $\Pi_*$ denotes the set of stationary optimal policies. Henceforth, we will omit the word ``stationary'' again for brevity, as we exclusively consider such policies. Let $\Pi_*^D$ denote the subset of deterministic optimal policies.

For a policy $\pi$, consider the Markov chain induced by $\pi$ on the state space $\S$. Let $n(\pi)$ denote the number of recurrent classes of this Markov chain, and $\Rs(\pi)$ the set of all states in these recurrent classes.
Define
\begin{align} 
     \Rst  & \= \left\{ s \in \S : s \in \Rs(\pi) \ \text{for some} \ \pi \in \Pi_*^D \right\} = \left\{ s \in \S : s \in \Rs(\pi) \ \text{for some} \ \pi \in \Pi_* \right\}, \label{eq-R*} \\  
     n^* & \= \min \left\{ n(\pi) : \Rs(\pi) = \Rst, \ \pi \in \Pi_* \right\}. \label{eq-n*}
\end{align}
Expressed in words, $\Rst$ consists of recurrent states under some optimal policy, and $n^*$ is the minimum number of recurrent classes under those optimal policies that make all states in $\Rst$ recurrent.

The set $\Rst$ can be partitioned into $n^*$ sets, $\Rs^{*1}, \Rs^{*2}, \ldots,  \Rs^{*n^*}$, which are the recurrent classes \emph{common} to all optimal policies $\pi_* \in \Pi_*$ with $\Rs(\pi_*) = \Rst$ and $n(\pi_*) = n^*$. 
For a weakly communicating SMDP, one such policy $\pi_*$ is given by the following: For $s \not\in \Rst$, let $\pi_*(a \mid s) > 0$ for all $a \in \A$; for $s \in \Rst$, let $\pi_*(a \mid  s) > 0$ if and only if $a \in \Kst(s)$, a set of optimal actions defined by
\begin{equation} \label{eq-K*}
 \Kst(s)   \= \left\{ a \in \A :  \pi(s) = a, s \in \Rs(\pi) \  \text{for some} \ \pi \in \Pi_*^D \right\}, \ \ \   s \in \Rst,
\end{equation} 
where $\pi(s)$ denotes the action taken at state $s$ for a deterministic policy $\pi$.

\citet{ScF78} showed that the solution set $\V$ can be parametrized by $n^*$ parameters $(y_1, \ldots, y_{n^*})$ that are associated with the sets $\Rs^{*1}, \Rs^{*2}, \ldots,  \Rs^{*n^*}$, with each $y_j$ corresponding to a shift in the state values by the constant $y_j$ for the states in $\Rs^{*j}$. More specifically, $\V$ has the following structure.
\begin{enumerate}[leftmargin=0.7cm]
\item[(i)] For $v \in \V$, its values $v(s), s \not\in \Rst$, are determined by its values $v(s), s \in \Rst$.
If we group the components of $v$ to write it as 
\begin{align}\label{eq: v two parts}
    v = (v^{(1)}, v^{(2)}) \ \ \  \text{with} \ \ v^{(1)} \= (v(s))_{s \in \Rst}, \  v^{(2)} \= (v(s))_{s \not\in \Rst},
\end{align}
then all solutions $v \in \V$ can be expressed as $v = (v^{(1)}, \phi(v^{(1)}))$ for some continuous function $\phi : \R^{|\Rst|} \to \R^{|\S| - |\Rst|}$ that satisfies $\phi(x + c\1) = \phi(x) + c \1$ for all $c \in \R$. (See \citepalias[Equation 4.5]{ScF78} for the exact expression of $\phi$.)
\item[(ii)] The set $\V^R \= \{ v^{(1)} \,\big|\,  v=(v^{(1)}, v^{(2)}) \in \V \}$, which determines $\V$ by (i), 
is an $n^*$-dimensional convex polyhedron. 
Specifically, fix some $\bar v^{(1)} \in \V^R$ and for $1 \leq j \leq n^*$, let $e_j \in \R^{|\Rst|}$ be the indicator of the set $\Rs^{*j}$: 
\begin{equation} \label{eq-def-e}
     e_j(s) = 1  \ \ \  \text{if} \ s \in \Rs^{*j};  \qquad e_j(s) = 0,  \ \ \  \text{if} \ s \in \Rst \setminus \Rs^{*j}.
\end{equation}
Then $\V^R$ can be parametrized as
\begin{equation} \label{eq-param-Vr}
\V^R = \left\{   \bar v^{(1)} +  y_1 e_1 + \cdots + y_{n^*} e_{n^*} \, \Big|\, (y_1, y_2, \ldots, y_{n^*}) \in D \right\}
\end{equation}
for an $n^*$-dimensional convex polyhedron $D \subset \R^{n^*}$ determined by the optimal policies in $\Pi^*$ and the sets $\Rs^{*1}, \Rs^{*2}, \ldots, \Rs^{*n^*}$. (See \citepalias[Theorem\ 5.1(d)]{ScF78} for the exact expression of $D$.) Constrained within the set $D$, these parameters $y_1, y_2, \ldots, y_{n^*}$ need not be globally independent; their values can depend on one another. In the particular case of a weakly communicating SMDP, unless $n^*=1$, no parameter can be chosen freely and independently of the other.
\item[(iii)] By (i) and (ii), the solutions $v \in \V$ can be parametrized as
\begin{equation} \label{eq-param-V}
v = (v^{(1)}, \phi(v^{(1)})) \ \ \text{with} \ v^{(1)} = \bar v^{(1)} +  y_1 e_1 + \cdots + y_{n^*} e_{n^*}, \ \  (y_1, \ldots, y_{n^*}) \in D.
\end{equation}
Thus, $\V$ is homeomorphic to the $n^*$-dimensional convex polyhedron $D$, and so is $\Qo$ since it is homeomorphic to $\V$, as discussed earlier. 
\end{enumerate}

\begin{myremark} \rm \label{rmk-app-A1}
We make two observations.\\*[0.1cm]
\noindent (a) For all $c \in \R$, $v + c \1 \in \V$ if $v \in \V$; or in other words, $\V + c\1 = \V$.  Therefore, $\V^R + c \1 = \V^R$ for all $c \in \R$ and likewise, given the definition of the $e_j$'s, the set $D$ has the property that $D + c \1 = D$ for all $c \in \R$.
For a weakly communicating SMDP, $\1$ and $-\1$ are the only directions along which the convex polyhedra $\V^R$ and $D$ are unbounded. This can be shown using the results from \citepalias{ScF78} or proved directly, similar to the boundedness proof for Theorem~\ref{thm: smdp properties}(i). This fact is closely linked to the property of $\V$ in the special case discussed next in (b).\\*[0.1cm]
\noindent (b) If the SMDP is weakly communicating and the rewards are all zero, then $\Rst$ is just the unique closed communicating system of the SMDP. Consequently, $n^* = 1$ and $\V$ is one-dimensional, consisting solely of vectors $c \1, c \in \R$. This gives an alternative proof of Lemma~\ref{lemma: 0 reward MDP has 1 d solution} based on the theory given in \citepalias{ScF78}.\qed
\end{myremark}

\subsection{Applying Degree of Freedom Analysis to RVI Algorithms} \label{sec-7.2}

We now use the preceding characterizations of $\V$ and $\Qo$ to derive a parametrization of the set $\Qs = \{ q \in \Qo \mid f(q) = r_* \}$, which corresponds to the solution sets $\mathcal{Q}_\infty$ and $\hat{\mathcal{Q}}_\infty$ of the three Q-learning algorithms studied previously in our Theorems~\ref{thm: Extended RVI Q-learning},~\ref{thm: inter-option Differential Q-learning},~\ref{thm: intra-option Differential Q-learning}. Recall that the function $f$ has the property that for some $u > 0$, $f(q + c \1) = f(q) + c u$ for all $c \in \R$ and $q \in \R^{\S \times \A}$ (Assumption~\ref{assu: f}(ii)).

\begin{mytheorem} \label{thm-dim-Qs}
In a weakly communicating SMDP, the set $\Qs$ is homeomorphic to an $(n^*-1)$-dimensional convex polyhedron, where $n^*$ is given by (\ref{eq-n*}).
\end{mytheorem}

\begin{proof}
Consider the space spanned by the vectors $\{e_1, e_2, \ldots, e_{n^*}\}$ defined in \eqref{eq-def-e}. Choose a different basis $\{\1, e'_1, \ldots, e'_{n^*-1}\}$ for this space, and express the $n^*$-dimensional convex polyhedron 
$E \= \{ y_1 e_1 + \cdots + y_{n^*} e_{n^*} \mid (y_1, \ldots, y_{n^*}) \in D \}$ in terms of the new basis vectors as
$E = \{ z_0 \1 + z_1 e'_1 + \cdots + z_{n^*-1} e'_{n^*-1} \mid (z_0, z_1, \ldots, z_{n^*-1}) \in D' \}$, for some $n^*$-dimensional convex polyhedron $D' \subset \R^{n^*}$. By \eqref{eq-param-V}, the solutions $v \in \V$ can be equivalently parametrized as 
\begin{equation} \label{eq-appA-prf1}
v = (v^{(1)}, \phi(v^{(1)})) \ \,  \text{with} \  v^{(1)} = \bar v^{(1)} +  z_0 \1 + z_1 e'_1 + \cdots + z_{n^*-1} e'_{n^*-1}, \ \ \text{for} \ (z_0, z_1, \ldots, z_{n^*-1}) \in D'.
\end{equation}
Then by the homeomorphism between $\V$ and $\Qo$ [cf.\ \eqref{eq-q2v} and \eqref{eq-v2q}], the solutions $q \in \Qo$ can also be parametrized by $(z_0, z_1, \ldots, z_{n^*-1})$ as 
$$ \Qo = \{ \psi(z_0, z_1, \ldots, z_{n^*-1}) \mid (z_0, z_1, \ldots, z_{n^*-1}) \in D' \},$$
where the function $\psi$ is the composition of the mapping $(z_0, z_1, \ldots, z_{n^*-1}) \mapsto v$ given by \eqref{eq-appA-prf1} with the mapping $v \mapsto q_v$ given by \eqref{eq-v2q} and is a homeomorphism between $D'$ and $\Qo$. 

Now the set $D'$ has the property that its $z_0$-sections are the same for all $z_0 \in \R$:
$$D'_0 \=  \{ (z_1, \ldots, z_{n^*-1}) \!\mid\! (0, z_1, \ldots, z_{n^*-1}) \in D' \} = \{ (z_1, \ldots, z_{n^*-1}) \!\mid\! (z_0, z_1, \ldots, z_{n^*-1}) \in D' \},$$
because for all $c \in \R$, $E + c \1 = E$ by the definition of $E$, the expression of $\V^R$ in (\ref{eq-param-Vr}), and the fact $\V^R + c \1 =  \V^R$ discussed earlier in Remark~\ref{rmk-app-A1}(a). Since $D'$ is an $n^*$-dimensional convex polyhedron, it follows that $D'_0$ is an $(n^*-1)$-dimensional convex polyhedron. 

By definition the function $\psi$ satisfies that for all $z = (z_1, \ldots, z_{n^*-1}) \in D'_0$, 
$$\psi(z_0, z) = \psi(0, z)  + z_0 \1, \qquad \forall \, z_0 \in \R.$$
Consequently, if $f(\psi(z_0,z)) = r_*$, then $r_* = f(\psi(0, z)  + z_0 \1) = f(\psi(0, z)) +  z_0 u$ by Assumption~\ref{assu: f}(ii), implying $z_0 = \big(r_* - f(\psi(0, z))\big)/u$.
Thus the set $\Qs = \{ q \in \Qo \mid f(q) = r_* \}$ can be parametrized as
\begin{equation} \label{eq-param-Qs}
 \Qs = \{ \psi(c_0(z), z) \mid z = (z_1, \ldots, z_{n^*-1}) \in D'_0 \}, \quad \text{where} \ c_0(z) \=  \big(r_* - f(\psi(0, z))\big)/u.
\end{equation} 
This shows that $\Qs$ is homeomorphic to the $(n^*-1)$-dimensional convex polyhedron $D'_0$.
\end{proof}

\vspace*{-0.4cm}

We close this section by discussing briefly an alternative way to analyze the degrees of freedom of solutions in the sets $\Qo$ and $\Qs$. 
This is to view the optimality equation (\ref{eq-oe-q}) for state-action value functions as the optimality equation (\ref{eq-oe-v}) for value functions in an SMDP with enlarged state and action spaces, which we call $\text{SMDP}_q$. Then $\Qo$ becomes the solution set $\V$ for $\text{SMDP}_q$ and can be characterized directly by applying the results of \citepalias{ScF78} to $\text{SMDP}_q$. 

The definition of $\text{SMDP}_q$ is as follows. Its state space is $\S \times \A$, and its action space is $\Pi^D$ (the finite set of deterministic policies of the original SMDP). From its state $(s, a)$ under action $\pi \in \Pi^D$, the probability of transitioning to state $(s', a')$ is given by $p_{ss'}^a \, \ind(\pi(s') = a')$, and the expected one-stage reward and holding time are given by $r_{sa}$ and $l_{sa}$, respectively, independently of the action $\pi$.
It is clear that if the original SMDP is weakly communicating, so is $\text{SMDP}_q$.

We use $\Rst_q$, $n^*_q$, and $\Rs^{*j}_q, 1 \leq j \leq n^*_q$, to refer to the objects given respectively by \eqref{eq-R*}, \eqref{eq-n*}, and the partition of $\Rst_q$ explained after (\ref{eq-n*}), for $\text{SMDP}_q$, while we reserve the notations $\Rst, n^*$, and $\Rs^{*j}, 1 \leq j \leq n^*,$ for these objects in the original SMDP. Recall the optimal action sets $K^*(s), s \in \Rst$, defined by (\ref{eq-K*}) for the original SMDP.
By applying \citepalias[Theorems 3.1 and 3.2]{ScF78}, we can show the following correspondences between $\text{SMDP}_q$ and the original SMDP, assuming the latter is weakly communicating:

\vspace*{-0.2cm}
\begin{mylemma} \label{lem-sol1}
We have $\Rst_q = \{ (s, a) :  a \in \Kst(s), \, s \in \Rst \}$ and $n^*_q = n^*$, and when ordered suitably, the sets $\Rs^{*j}_q = \{ (s, a) :  a \in \Kst(s), \, s \in \Rs^{*j} \}$ for all $1 \leq j \leq n^*$.
\end{mylemma}

\vspace*{-0.2cm}

Combining Lemma~\ref{lem-sol1} with the characterization of $\Qo$ given in \eqref{eq-param-V} for $\text{SMDP}_q$, we obtain an $n^*$-dimensional parametrization of the set $\Qo$. We can then use it to derive an $(n^*-1)$-dimensional parametrization of the set $\Qs$, similarly to the proof of Theorem~\ref{thm-dim-Qs}.

\vspace*{-0.2cm}

\section{Conclusions and Discussion}\label{sec: conclusions}
We introduced several new theoretical results for average-reward tabular RL algorithms. Our most significant result is the asymptotic convergence of a family of average-reward Q-learning algorithms in weakly communicating MDPs, a class of MDPs that is more general than previously considered. We also characterized the solution sets of these algorithms, demonstrating that they are nonempty, compact, connected, possibly nonconvex, and have one lower degree of freedom than the solution set of the average-reward optimality equation. Extending our results from algorithms operating with actions to those operating with options, we showed that two average-reward options learning algorithms converge when the underlying SMDP is weakly communicating. We believe that our findings contribute to a deeper understanding of average-reward RL algorithms, potentially facilitating their adoption in RL applications where achieving high performance over the long term is desired.

There are several ways in which our work can be extended.
First, in all the studied algorithms, step sizes are defined using the visitation count for each state-action pair. One potential way to extend our work is to develop convergence results for algorithms without these visitation counts, potentially using a recent stability result by \cites{liu2024ode}.
Second, RVI Q-learning in its current state can not handle general MDPs. This is because the algorithm solves only the average-reward optimality equation, while for more general MDPs, optimal policies are characterized by the optimality equation and another equation. One potential future direction is to adapt the RVI Q-learning algorithm to handle general MDPs and extend the analysis developed here to show convergence for the revised algorithm. Third, while RVI Q-learning is a family of tabular algorithms, they can be extended to the function approximation setting, following a way similar to the one outlined in Appendix\ E in \citet{wan2021learning}. A potential future work is to study the convergence of this function approximation extension. 

\vspace*{-0.3cm}

\acks{This research was conducted at the University of Alberta. YW thanks Meta AI for allowing him to finish writing this paper while employed by them. This research was supported in part by DeepMind and Amii. HY also acknowledges the support of the Natural Sciences and Engineering Research Council of Canada (NSERC), RGPIN-2024-04939. HY thanks Professor Eugene Feinberg for the helpful discussion on average-reward SMDPs. We appreciate Dr.\ Martha Steenstrup's helpful feedback on parts of the paper.}

\bibliography{jmlr}

\end{document}